%% file: main.tex
\documentclass[journal,draftclsnofoot,onecolumn,10pt,romanappendices]{IEEEtran}
\usepackage{tikz}
\usepackage{pgfplots}
\usepackage{amsmath,epsfig,amssymb,algorithm,algpseudocode,amsthm,cite,url,xcolor}
\usepackage{hyperref}
\usepackage{enumerate}
\usepackage[latin1]{inputenc}

\theoremstyle{plain}
\newtheorem{theorem}{Theorem}
\newtheorem{assumption}{Assumption}
\newtheorem{lemma}[theorem]{Lemma}

\theoremstyle{definition}

\newtheorem{remark}{Remark}

\newtheoremstyle{specialcasestyle}{1mm}{1mm}{\upshape}{}{\bfseries\upshape}{.}{0mm}{}
\theoremstyle{specialcasestyle}

\newcommand{\figref}[1]{Fig.~\protect\ref{#1}}

\newcommand{\bu}{{\bf u}}
\newcommand{\bg}{{\bf g}}

\newcommand{\bh}{{\bf h}}

\newcommand{\bx}{{\bf x}}
\newcommand{\bv}{{\bf v}}
\newcommand{\ba}{{\bf a}}

\newcommand{\bw}{{\bf w}}

\newcommand{\bz}{{\bf z}}

\newcommand{\bX}{{\bf X}}

\newcommand{\bP}{{\bf P}}
\newcommand{\bn}{{\bf n}}

\newcommand{\bbeta}{{\boldsymbol\beta}}

\newcommand{\blamb}{{\boldsymbol\lambda}}
 \newcommand{\balpha}{{\boldsymbol\alpha}}
\newcommand{\bxi}{{\boldsymbol\xi}}
\newcommand{\bone}{{\boldsymbol 1}}
 \newcommand{\asto}{\overset{a. s.}\longrightarrow }

     \makeatletter
\def\endthebibliography{%
  \def\@noitemerr{\@latex@warning{Empty `thebibliography' environment}}%
  \endlist
}
\makeatother
\usepackage{graphicx}
\usepackage{placeins}
\usepackage{float}
\usepackage{tabularx}
\usepackage{tkz-euclide,subfigure}
\usepackage{amsmath,epsfig,amssymb}
\usepackage{amsmath}
\usepackage{bbm}
\usetikzlibrary{intersections}
\usepgfplotslibrary{fillbetween}
\newcommand{\sign}{{\rm sign}}
\title{A Precise Performance Analysis of Support Vector Regression}
\author{Houssem Sifaou, Abla Kammoun, Mohamed-Slim Alouini}
\begin{document}
\maketitle
\begin{abstract}
In this paper, we study the hard and soft support vector regression techniques applied to a set of $n$ linear measurements of the form $y_i=\boldsymbol{\beta}_\star^{T}{\bf x}_i +n_i$ where $\boldsymbol{\beta}_\star$ is an unknown vector, $\left\{{\bf x}_i\right\}_{i=1}^n$ are the feature vectors and $\left\{{n}_i\right\}_{i=1}^n$ model the noise. Particularly, under some plausible assumptions on the statistical distribution of the data, we characterize the feasibility condition for the hard support vector regression in the regime of high dimensions and, when feasible, derive an asymptotic approximation for its risk. Similarly, we study the test risk for the soft support vector regression as a function of its parameters. Our results are then used to optimally tune the parameters intervening in the design of hard and soft support vector regression algorithms. Based on our analysis, we illustrate that adding more samples may be harmful to the test performance of support vector regression, while it is always beneficial when the parameters are optimally selected. Such a result reminds a similar phenomenon observed in modern learning architectures according to which optimally tuned architectures present a decreasing test performance curve with respect to the number of samples. 
\end{abstract}

\section{Introduction}

\noindent{\bf Motivation.} Recent works have demonstrated that the test performance of modern learning architectures exhibits both model-wise and sample-wise double descent phenomena that defy conventional statistical intuition. Model-wise descent, reported in recent works \cite{Belkin15849,pmlr-v89-belkin19a, Geiger2019ScalingDO}, suggests that, for very large architectures, performance improves with the number of parameters, thus contradicting the bias-variance trade-off. On the other hand, sample-wise descent, discussed recently in the works of Nakkiran {\it et al.} \cite{Nakkiran2020Deep,abs-2003-01897} indicates that more data may harm the performance. One potential solution to avoid such a behavior consists in optimally tuning the involved parameters. In doing so, the test performance in most scenarios decreases with the number of samples.

In this paper, we investigate the sample-wise double descent phenomenon for basic linear models. More precisely, we assume independent data samples $({\bf x}_i,y_i)$, $i=1,\cdots,n$
distributed as:
$$
y_i=\boldsymbol{\beta}_\star^{T}{\bf x}_i+{\sigma}n_i
$$
 where ${\bf x}_i\in\mathbb{R}^p$ is the feature vector assumed to have zero mean and covariance ${\bf I}_p$, $y_i$ represent the scalar response variables while  $n_i$ stands for zero-mean noise and $\sigma^2$ is the noise variance. To estimate $\boldsymbol{\beta}_\star$, we consider support vector regression techniques: hard-support vector regression (H-SVR), which estimates the regression vector $\boldsymbol{\beta}$ with minimum $\ell_2$ norm that satisfies the constraints ${y}_i=\boldsymbol{\beta}^{T}{\bf x}_i$ up to a maximum error $\epsilon$, and soft-support regression (S-SVR) which uses a regularization constant $C$ that aims to create a  trade-off between the minimization of the training error and the minimization of the model complexity.       

We study the performance of H-SVR and S-SVR when the number of features $p$ and the sample size $n$ grow simultaneously large such that $\frac{n}{p}\to\delta$ with $\delta>0$ and the norm of $\|\boldsymbol{\beta}_\star\|$ converges to $\beta$. One major outcome of the present work is to recover interesting behaviors observed in large-scale machine learning architectures. Particularly, we show that the double descent behavior appears only when the H-SVR or S-SVR parameters are not properly tuned. Such a behavior reminds the recent findings in \cite{abs-2003-01897} that suggest that unregularized models often suffer from the sample-wise double descent phenomenon, while optimally tuned models usually present a monotonic risk with respect to the number of samples.     

\noindent{\bf Contributions.} This paper investigates the test risk behavior as a function of the sample size for H-SVR and S-SVR techniques. Contrary to linear regression, which involves explicit form expressions for the solution, H-SVR and S-SVR require solving convex-optimization problems, which do not have closed-form solutions. To analyze the test risk, we rely on the Gaussian min-max theorem (CGMT) framework and more specifically on the extension of this framework recently developed in \cite{kam-chris}, which has been proven to be suitable to analyze functionals of solutions of convex optimizations problems.

Concretely, our results for the H-SVR and S-SVR can be summarized as follows:
\begin{enumerate}
\item  We derive for a fixed error tolerance $\epsilon$, a sharp phase transition-threshold $\delta_\star$ beyond which the H-SVR becomes infeasible. Interestingly,  we illustrate that the transition threshold depends only on $\epsilon$ and the noise variance and not on the SNR defined as ${\rm SNR}:=\frac{\beta^2}{\sigma^2}$.  Moreover, we show that $\delta_\star$ is always greater than $1$, which should be compared with the condition $\delta<1$ required for the least square estimator to exist. As a side note, we prove that contrary to hard-margin support vector classifiers, H-SVR can always be feasible through a proper tuning of the tolerance error $\epsilon$. This allows us to study the test risk of the H-SVR as a function of $\delta$ when $\delta\in(0,\infty)$ and $\epsilon$ carefully tuned to satisfy the feasibility condition.  
\item For fixed error tolerance $\epsilon$, we numerically illustrate that for moderate to large ${\rm SNR}$ the test risk of the H-SVR is a non-monotonic curve presenting a unique minimum that becomes the closest to $\delta_\star$ as the SNR increases.  For low SNR values, the test risk is an increasing function of $\delta_\star$ and is always worse than the null risk associated with the null estimator. It is worth mentioning that behavior of the same kind was reported for the min-norm least square estimator in \cite{hastie2019surprises}. Additionally, we illustrate that when the parameter $\epsilon$ is optimally tuned, the test risk becomes a decreasing function of $\delta$ and equivalently of the number of data samples.
\item Similarly, we derive the expression for the asymptotic test risk as a function of $\epsilon$, $\delta$, and the regularization constant $C$. Without optimal tuning of the regularization constant $C$ and factor $\epsilon$, the test curve as a function of the sample test size presents a double descent, which disappears when optimal settings of these constants is considered.   
\item We study the robustness of the S-SVR and H-SVR to impulsive noise. We illustrate that contrary to H-SVR, S-SVR, when optimally tuned, is resilient to impulsive noises. Particularly, we show that for mild impulsive noise conditions, S-SVR presents a slightly lower risk than optimally tuned ridge regression estimators but largely outperforms it under moderate to severe impulsive noise conditions. 
\end{enumerate}

\noindent{\bf Related works.}
The present work is part of the continued efforts to understand the double descent phenomena in large-scale machine learning architectures. An important body of research works focused on establishing the behavior of double descent of the test risk as a function of the model size in a variety of machine learning algorithms \cite{Belkin15849,Opper,spigler}. Very recently, the work in \cite{Nakkiran2020Deep} discovered that double descent occurs not just as a function of the model size but also as a function of the sample size \cite{Nakkiran2020Deep}. A major consequence of such a behavior is that performance may be degraded as we increase the number of samples.   

To further understand the generalization error, several works considered to analyze it as a function of the model size for mathematically tractable settings in regression \cite{hastie2019surprises,Belkin2019TwoMO,9051968,mitra19}  and more recently in classification \cite{kam-chris,9174344}, with the goal of investigating as to under which conditions, the double descent occurs. In this paper, similarly to \cite{abs-2003-01897}, we instead focus on the effect of the sample size on the test performance, but with the H-SVR and S-SVR as case examples. Moreover, on the technical level, our analysis provides sharp characterizations of the performance using the recently developed extension of the CGMT framework \cite{kam-chris}.

\section{Problem formulation}

Consider the problem of estimating the scalar response $y$ of a vector ${\bf x}$ in $\mathbb{R}^{p}$ from a set of $n$ data samples $\left\{({\bf x}_i,y_i)\right\}_{i=1}^n$ following the linear model:
\begin{equation}
{y_i}=\boldsymbol{\beta}_\star^{T}{\bf x}_i+\sigma {n}_i
\label{eq:linear}
\end{equation}
where $\boldsymbol{\beta}_\star$ is an unknown vector,   $\{n_i\}_{i=1}^n$ represent noise samples with  zero mean and variance $1$ and $\sigma^2$ represents the noise variance. We further assume that ${\bf x}$ have zero mean and covariance ${\bf I}_p$.   

To estimate $\boldsymbol{\beta}_\star$, we consider  support vector regression methods, namely the hard support vector regression denoted by H-SVR and the soft support vector regression referred to as S-SVR.  The H-SVR looks for a function $y=\boldsymbol{\beta}^{T}{\bf x}$ such that all data points $({\bf x}_i,\boldsymbol{\beta}^{T}{\bf x}_i)$ deviates at most $\epsilon$ from their targets $y_i$. Formally, this regression problem can be written as:

\begin{equation}
\label{eq:hard_margin}
\begin{aligned}
\hat{{\bf w}}_H:=\arg\min_{{\bf w}} \quad & \frac{1}{2}\|{\bf w}\|^2\\
\textrm{s.t.} \quad & y_i-{\bf w}^{T}{\bf x}_i \leq \epsilon \\
  &   {\bf w}^{T}{\bf x}_i-y_i \leq \epsilon  \\
\end{aligned}
\end{equation}
It is worth mentioning that when $\epsilon=0$ and  $n\leq p$, the H-SVR boilds down to the least square  estimator. In this case, it perfectly interpolates the training data, satisfying $y_i={\bf x}_i^{T}\hat{{\bf w}}_H$, $i=1,\cdots,n$.     

In general, depending on the value of $\epsilon$, there may not be a solution that satisfies the constraints.  
As in  support vector machines for classification, one solution to deal with such cases is to add slack variables that while relaxing the constraints, penalize in the objective function large deviations from them. Applying this approach gives the S-SVR method which involves solving the following optimization problem:
\begin{equation}
\begin{aligned}
\hat{{\bf w}}_S:=\arg\min_{{\bf w}} \quad & \frac{1}{2}\|{\bf w}\|^2 + \frac{C}{p}\sum_{i=1}^n(\xi_i+\tilde{\xi}_i)\\
\textrm{s.t.} \quad & y_i-{\bf w}^{T}{\bf x}_i \leq \epsilon+\xi_i, \ i=1,\cdots,n \\
  &   {\bf w}^{T}{\bf x}_i-y_i \leq \epsilon+\tilde{\xi}_i ,  \\
&\xi_i,\tilde{\xi}_i\geq 0,  
\end{aligned}
\label{eq:soft_SVR}
\end{equation}
The aim of the present work is to characterize analytically the performance of the H-SVR and the S-SVR. The assumption underlying the analysis is to consider that the number of samples and that of features grow with the same pace, and will be made more specific in the sequel. 

\noindent{\bf Prediction risk.} The metric of interest in this paper is the prediction risk. For a given estimator $\hat{\boldsymbol{\beta}}$, 
the prediction risk is defined as:
$$
\mathcal{R}(\hat{\boldsymbol{\beta}}):=\mathbb{E}_{{\bf x},y}|{\bf x}^{T}\hat{\boldsymbol{\beta}}-{\bf x}^{T}\boldsymbol{\beta}_\star|^2=\|\hat{\boldsymbol{\beta}}-\boldsymbol{\beta}_\star\|_2^2
$$
where ${\bf x}$ and ${\bf y}$ are test points following the model \eqref{eq:linear} but are independent of the training set. Expressing $\mathcal{R}(\hat{\boldsymbol{\beta}})$ as:
$$
\mathcal{R}(\hat{\boldsymbol{\beta}})= \|\boldsymbol{\beta}_\star\|_2^2+\|\hat{\boldsymbol{\beta}}\|_2^2-2 \|\boldsymbol{\beta}_\star\|_2\|\hat{\boldsymbol{\beta}}\|_2\frac{\boldsymbol{\beta}_\star^{T} \hat{\boldsymbol{\beta}}}{ \|\boldsymbol{\beta}_\star\|_2\|\hat{\boldsymbol{\beta}}\|_2}
$$
we can easily see that the risk depends on $\hat{\boldsymbol{\beta}}$ through its norm $\|\hat{\boldsymbol{\beta}}\|$ and the cosine similarity between $\boldsymbol{\beta}_\star$ and  $\hat{\boldsymbol{\beta}}$:
$$
{\rm cos}\left(\boldsymbol{\beta}_\star,\hat{\boldsymbol{\beta}}\right):=\frac{\boldsymbol{\beta}_\star^{T}\hat{\boldsymbol{\beta}}}{\|\boldsymbol{\beta}_\star\|.\|\hat{\boldsymbol{\beta}}\|}
$$

\section{Main results}
\label{main_results}
In this paper, we consider studying the performance of the hard and soft support vector regression problems under the asymptotic regime in which $n$ and $p$ grow large at the same pace. More specifically, the following assumption is considered. 
\begin{assumption}
Our study is based on the following set of assumptions:
\begin{itemize}
\item $n$ and $p$ grow to infinity with $\frac{n}{p}\to \delta$.
\item The noise variance $\sigma^2$ is a fixed positive constant.
\item We assume that $\|\boldsymbol{\beta}_\star\|\to \beta$ where $\beta$ is a certain positive scalar.
 \item {\bf Data model}: The training $\left\{{\bf x}_i\right\}_{i=1}^n$ are independent and identically distributed following standard normal distribution. Moreover, the noise samples $\{n_i\}$ are independent and are drawn from a symmetric distribution $p_N$ satisfying $\mathbb{E}|N|^2<\infty$ where $N\sim p_N$.  
\end{itemize}
\label{ass:regime}
\end{assumption}
\subsection{Hard SVR}
As mentioned earlier, the H-SVR problem is not always feasible. We provide in Theorem \ref{feasibility_region} a sharp characterization of the feasibility region of the H-SVR in the asymptotic regime defined in Assumption \ref{ass:regime}.
\begin{theorem}[Feasibility of the H-SVR]
Let $\delta_\star$ be defined as:
\begin{equation}
\delta_\star=\frac{1}{\displaystyle\inf_{t\in\mathbb{R}} \mathbb{E}\left(|{G}+t\sigma N|-t\epsilon\right)_{+}^2}=\frac{1}{\displaystyle\inf_{t\in\mathbb{R}} \mathbb{E}\left(|{G}+tN|-t\frac{\epsilon}{\sigma}\right)_{+}^2} \label{eq:delta_star}
\end{equation}
where the expectation is taken over the distribution of ${G}$ and $N$ where ${G}\sim\mathcal{N}(0,1)$ and $N\sim p_N$ \footnote{The second equality is obtained by operating the change of variable $\tilde{t}\leftrightarrow {t\sigma}$}. 
Consider the asymptotic regime and data model in Assumption \ref{ass:regime}. 
Then the following statements hold true:
\begin{align}
\delta > \delta_\star &\Rightarrow \mathbb{P}\left[\text{The H-SVR is feasible for sufficiently large} \ \ n\right] = 0 \label{eq:firstp}\\ 
\delta<\delta_\star &\Rightarrow  \mathbb{P}\left[\text{The H-SVR is feasible for sufficiently large} \ \ n\right] = 1. \label{eq:secondp}
\end{align}
\label{feasibility_region}
\end{theorem}
\begin{remark}{(\bf Feasibility of the  H-SVR depends on the noise variance but not on the SNR.)}
The above result establishes that the existence of the H-SVR undergoes a sharp transition phenomenon. Particularly, in the limit of large sample size $n$ and number of features $p$ such that 
$\frac{n}{p}\to \delta$, the H-SVR is almost surely unfeasible when $\delta>\delta_\star$ and always feasible when $\delta<\delta_\star$. The obtained expression is reminiscent of other previously established result for the existence of the hard-margin SVM for classification established in a series of recent works \cite{svm_kammoun} and \cite{kammoun_chris}. However, contrary to the expressions obtained in these works, the separability boundary curve captured by $\delta_\star$ does not depend on the  Euclidean norm of $\boldsymbol{\beta}_\star$, or equivalently on the SNR defined as $\frac{\|\boldsymbol{\beta}_\star\|^2}{\sigma^2}$ but only on the noise variance. The reason behind this is that feasibility is essentially related to how much the data samples deviate from the hyperplane defined as $\boldsymbol{\beta}_\star^T{\bf x}=y$.       
We note that as $\sigma$ approaches $0$ and $\epsilon\neq 0$, $\delta_\star\to\infty$, which implies that the H-SVR is always feasible in this case. This is because in the noiseless case, all data samples $({\bf x}_i,y_i)$ belong to the hyperplane $y_i=\boldsymbol{\beta}_\star^{T}{\bf x}_i$ and thus $\boldsymbol{\beta}_\star$ is in the  feasibility set  of the H-SVR regardless of the value of $\epsilon$ and also on $\beta$. On the other hand, as $\sigma$ increases, $\delta_\star$ decreases, which suggests that the H-SVR becomes less feasible since it is less easy to find a hyperplane that contains all data samples with a reasonable error tolerance $\epsilon$. 
\end{remark}
\begin{remark}{(\bf The H-SVR can be feasible when the least square estimator is not) }
 One can easily check that if $\epsilon=0$, then $\delta_\star=1$. This result makes sense since as long as $\delta<1$, the linear system $\boldsymbol{\beta}^{T}{\bf x}_i=y_i$ is under determined and as such  a solution $\boldsymbol{\beta}$ exists. However, when $\delta_\star>1$, the linear system becomes over-determined, and as such it is impossible to find a solution to this linear system. Moreover, since $\delta_\star$ is an increasing function of $\epsilon$, we conclude that $\delta_\star>1$ for all $\epsilon>0$. This particularly shows that the H-SVR provides a larger feasibility region than the  least square estimator, for which $\delta$ should be less than $1$ to exist.  
We can even push this result further and claim that {\it every $\delta$ is feasible once $\epsilon$ is appropriately tuned. }
To see this, it suffices to note that $\epsilon\mapsto \delta_\star$ is an increasing function establishing a one-to-one map from $(0,\infty)$ to $(1,\infty)$. Hence, for any $\delta\in(1,\infty)$ there exists $\epsilon^\star(\delta)$ such that for all $\epsilon> \epsilon^\star(\delta)$,  the H-SVR  is almost surely feasible. 
\end{remark}

For the sake of illustration, Figure \ref{fig:deltastar} displays $\delta_\star$ as a function of $\epsilon$ for several values of $\sigma$. As can be seen, $\delta_\star$ is an increasing function growing to infinity with $\epsilon$. The value of $\epsilon$ plays a fundamental role to remediate the effect of the noise and ensure the feasibility of the H-SVR.  
One can note as expected that $\delta_\star$ for $\epsilon=0.1$ and $\sigma=0.1$ is the same as the one obtained when  $\epsilon=0.2$ and $\sigma=0.2$. This finding can be easily concluded from the second equality in \eqref{eq:delta_star}. Moreover, in agreement with our previous discussion, we can easily see that $\delta_\star$ decreases significantly as the noise variance increases. Such a scenario can be fixed by adapting the value of $\epsilon$.  
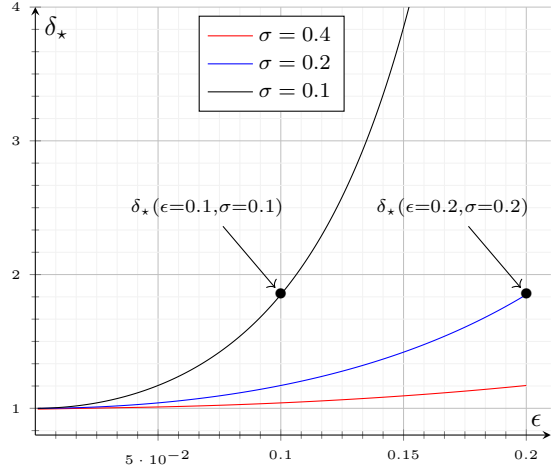
\begin{figure}[h]
\begin{center}
\begin{tikzpicture}[scale=1,font=\fontsize{10}{10}\selectfont]
   \tikzstyle{every axis y label}+=[yshift=5pt]
   \tikzstyle{every axis x label}+=[yshift=5pt]
   \tikzstyle{every axis legend}+=[cells={anchor=west},fill=white,
        at={(0.60,0.98)}, anchor=north east, font=\fontsize{8}{8}\selectfont]
   \begin{axis}[
grid=both,
grid style={line width=.1pt, draw=gray!10},major grid style={line width=.2pt,draw=gray!50},
xtick={},ytick={},
    minor tick num=5,
    enlargelimits={abs=0},
    ticklabel style={font=\tiny,fill=white},
    axis lines=middle,
      xmin=0.0,
      ymin=0.8,
      xmax=0.21,
      ymax=4,
xlabel={$\epsilon$ },
ylabel={$\delta_\star$}	]		
\addplot[name path=f,color=red] coordinates{
(1.000000e-03,9.939569e-01)(2.000000e-03,9.942732e-01)(3.000000e-03,9.945896e-01)(4.000000e-03,9.949061e-01)(5.000000e-03,9.952228e-01)(6.000000e-03,9.955396e-01)(7.000000e-03,9.958565e-01)(8.000000e-03,9.961735e-01)(9.000000e-03,9.964907e-01)(1.000000e-02,9.968080e-01)(1.100000e-02,9.971254e-01)(1.200000e-02,9.974429e-01)(1.300000e-02,9.977605e-01)(1.400000e-02,9.980783e-01)(1.500000e-02,9.983962e-01)(1.600000e-02,9.987142e-01)(1.700000e-02,9.990323e-01)(1.800000e-02,9.993506e-01)(1.900000e-02,9.996690e-01)(2.000000e-02,9.999875e-01)(2.100000e-02,1.000306e+00)(2.200000e-02,1.000625e+00)(2.300000e-02,1.000944e+00)(2.400000e-02,1.001263e+00)(2.500000e-02,1.001582e+00)(2.600000e-02,1.001901e+00)(2.700000e-02,1.002220e+00)(2.800000e-02,1.002540e+00)(2.900000e-02,1.002860e+00)(3.000000e-02,1.003179e+00)(3.100000e-02,1.003499e+00)(3.200000e-02,1.003819e+00)(3.300000e-02,1.004139e+00)(3.400000e-02,1.004460e+00)(3.500000e-02,1.004780e+00)(3.600000e-02,1.005100e+00)(3.700000e-02,1.005421e+00)(3.800000e-02,1.005742e+00)(3.900000e-02,1.006063e+00)(4.000000e-02,1.006384e+00)(4.100000e-02,1.006708e+00)(4.200000e-02,1.007040e+00)(4.300000e-02,1.007380e+00)(4.400000e-02,1.007729e+00)(4.500000e-02,1.008085e+00)(4.600000e-02,1.008450e+00)(4.700000e-02,1.008823e+00)(4.800000e-02,1.009204e+00)(4.900000e-02,1.009593e+00)(5.000000e-02,1.009990e+00)(5.100000e-02,1.010395e+00)(5.200000e-02,1.010809e+00)(5.300000e-02,1.011230e+00)(5.400000e-02,1.011660e+00)(5.500000e-02,1.012099e+00)(5.600000e-02,1.012545e+00)(5.700000e-02,1.012999e+00)(5.800000e-02,1.013462e+00)(5.900000e-02,1.013933e+00)(6.000000e-02,1.014412e+00)(6.100000e-02,1.014900e+00)(6.200000e-02,1.015396e+00)(6.300000e-02,1.015900e+00)(6.400000e-02,1.016412e+00)(6.500000e-02,1.016933e+00)(6.600000e-02,1.017462e+00)(6.700000e-02,1.017999e+00)(6.800000e-02,1.018544e+00)(6.900000e-02,1.019098e+00)(7.000000e-02,1.019661e+00)(7.100000e-02,1.020231e+00)(7.200000e-02,1.020810e+00)(7.300000e-02,1.021398e+00)(7.400000e-02,1.021993e+00)(7.500000e-02,1.022598e+00)(7.600000e-02,1.023210e+00)(7.700000e-02,1.023831e+00)(7.800000e-02,1.024461e+00)(7.900000e-02,1.025099e+00)(8.000000e-02,1.025745e+00)(8.100000e-02,1.026400e+00)(8.200000e-02,1.027064e+00)(8.300000e-02,1.027736e+00)(8.400000e-02,1.028416e+00)(8.500000e-02,1.029105e+00)(8.600000e-02,1.029803e+00)(8.700000e-02,1.030509e+00)(8.800000e-02,1.031224e+00)(8.900000e-02,1.031947e+00)(9.000000e-02,1.032679e+00)(9.100000e-02,1.033420e+00)(9.200000e-02,1.034169e+00)(9.300000e-02,1.034927e+00)(9.400000e-02,1.035694e+00)(9.500000e-02,1.036469e+00)(9.600000e-02,1.037253e+00)(9.700000e-02,1.038046e+00)(9.800000e-02,1.038848e+00)(9.900000e-02,1.039658e+00)(1.000000e-01,1.040477e+00)(1.010000e-01,1.041305e+00)(1.020000e-01,1.042141e+00)(1.030000e-01,1.042987e+00)(1.040000e-01,1.043841e+00)(1.050000e-01,1.044705e+00)(1.060000e-01,1.045577e+00)(1.070000e-01,1.046458e+00)(1.080000e-01,1.047348e+00)(1.090000e-01,1.048247e+00)(1.100000e-01,1.049155e+00)(1.110000e-01,1.050071e+00)(1.120000e-01,1.050997e+00)(1.130000e-01,1.051932e+00)(1.140000e-01,1.052876e+00)(1.150000e-01,1.053829e+00)(1.160000e-01,1.054791e+00)(1.170000e-01,1.055762e+00)(1.180000e-01,1.056743e+00)(1.190000e-01,1.057732e+00)(1.200000e-01,1.058731e+00)(1.210000e-01,1.059739e+00)(1.220000e-01,1.060756e+00)(1.230000e-01,1.061782e+00)(1.240000e-01,1.062817e+00)(1.250000e-01,1.063862e+00)(1.260000e-01,1.064916e+00)(1.270000e-01,1.065979e+00)(1.280000e-01,1.067052e+00)(1.290000e-01,1.068134e+00)(1.300000e-01,1.069226e+00)(1.310000e-01,1.070327e+00)(1.320000e-01,1.071437e+00)(1.330000e-01,1.072557e+00)(1.340000e-01,1.073686e+00)(1.350000e-01,1.074825e+00)(1.360000e-01,1.075973e+00)(1.370000e-01,1.077131e+00)(1.380000e-01,1.078298e+00)(1.390000e-01,1.079475e+00)(1.400000e-01,1.080662e+00)(1.410000e-01,1.081858e+00)(1.420000e-01,1.083064e+00)(1.430000e-01,1.084280e+00)(1.440000e-01,1.085505e+00)(1.450000e-01,1.086741e+00)(1.460000e-01,1.087986e+00)(1.470000e-01,1.089241e+00)(1.480000e-01,1.090505e+00)(1.490000e-01,1.091780e+00)(1.500000e-01,1.093064e+00)(1.510000e-01,1.094359e+00)(1.520000e-01,1.095663e+00)(1.530000e-01,1.096977e+00)(1.540000e-01,1.098301e+00)(1.550000e-01,1.099636e+00)(1.560000e-01,1.100980e+00)(1.570000e-01,1.102335e+00)(1.580000e-01,1.103699e+00)(1.590000e-01,1.105074e+00)(1.600000e-01,1.106459e+00)(1.610000e-01,1.107854e+00)(1.620000e-01,1.109260e+00)(1.630000e-01,1.110676e+00)(1.640000e-01,1.112102e+00)(1.650000e-01,1.113538e+00)(1.660000e-01,1.114985e+00)(1.670000e-01,1.116442e+00)(1.680000e-01,1.117910e+00)(1.690000e-01,1.119388e+00)(1.700000e-01,1.120876e+00)(1.710000e-01,1.122376e+00)(1.720000e-01,1.123885e+00)(1.730000e-01,1.125406e+00)(1.740000e-01,1.126937e+00)(1.750000e-01,1.128478e+00)(1.760000e-01,1.130030e+00)(1.770000e-01,1.131594e+00)(1.780000e-01,1.133167e+00)(1.790000e-01,1.134752e+00)(1.800000e-01,1.136348e+00)(1.810000e-01,1.137954e+00)(1.820000e-01,1.139571e+00)(1.830000e-01,1.141200e+00)(1.840000e-01,1.142839e+00)(1.850000e-01,1.144489e+00)(1.860000e-01,1.146150e+00)(1.870000e-01,1.147823e+00)(1.880000e-01,1.149506e+00)(1.890000e-01,1.151201e+00)(1.900000e-01,1.152907e+00)(1.910000e-01,1.154624e+00)(1.920000e-01,1.156353e+00)(1.930000e-01,1.158093e+00)(1.940000e-01,1.159844e+00)(1.950000e-01,1.161607e+00)(1.960000e-01,1.163381e+00)(1.970000e-01,1.165166e+00)(1.980000e-01,1.166963e+00)(1.990000e-01,1.168772e+00)(2.000000e-01,1.170592e+00)

};\addlegendentry{$\sigma=0.4$}           
\addplot[name path=f,color=blue] coordinates{
(1.000000e-03,9.987210e-01)(2.000000e-03,9.990396e-01)(3.000000e-03,9.993583e-01)(4.000000e-03,9.996771e-01)(5.000000e-03,9.999960e-01)(6.000000e-03,1.000315e+00)(7.000000e-03,1.000634e+00)(8.000000e-03,1.000954e+00)(9.000000e-03,1.001273e+00)(1.000000e-02,1.001593e+00)(1.100000e-02,1.001927e+00)(1.200000e-02,1.002294e+00)(1.300000e-02,1.002693e+00)(1.400000e-02,1.003124e+00)(1.500000e-02,1.003586e+00)(1.600000e-02,1.004082e+00)(1.700000e-02,1.004609e+00)(1.800000e-02,1.005168e+00)(1.900000e-02,1.005760e+00)(2.000000e-02,1.006384e+00)(2.100000e-02,1.007040e+00)(2.200000e-02,1.007729e+00)(2.300000e-02,1.008450e+00)(2.400000e-02,1.009204e+00)(2.500000e-02,1.009990e+00)(2.600000e-02,1.010809e+00)(2.700000e-02,1.011660e+00)(2.800000e-02,1.012545e+00)(2.900000e-02,1.013462e+00)(3.000000e-02,1.014412e+00)(3.100000e-02,1.015396e+00)(3.200000e-02,1.016412e+00)(3.300000e-02,1.017462e+00)(3.400000e-02,1.018544e+00)(3.500000e-02,1.019661e+00)(3.600000e-02,1.020810e+00)(3.700000e-02,1.021993e+00)(3.800000e-02,1.023210e+00)(3.900000e-02,1.024461e+00)(4.000000e-02,1.025745e+00)(4.100000e-02,1.027064e+00)(4.200000e-02,1.028416e+00)(4.300000e-02,1.029803e+00)(4.400000e-02,1.031224e+00)(4.500000e-02,1.032679e+00)(4.600000e-02,1.034169e+00)(4.700000e-02,1.035694e+00)(4.800000e-02,1.037253e+00)(4.900000e-02,1.038848e+00)(5.000000e-02,1.040477e+00)(5.100000e-02,1.042141e+00)(5.200000e-02,1.043841e+00)(5.300000e-02,1.045577e+00)(5.400000e-02,1.047348e+00)(5.500000e-02,1.049155e+00)(5.600000e-02,1.050997e+00)(5.700000e-02,1.052876e+00)(5.800000e-02,1.054791e+00)(5.900000e-02,1.056743e+00)(6.000000e-02,1.058731e+00)(6.100000e-02,1.060756e+00)(6.200000e-02,1.062817e+00)(6.300000e-02,1.064916e+00)(6.400000e-02,1.067052e+00)(6.500000e-02,1.069226e+00)(6.600000e-02,1.071437e+00)(6.700000e-02,1.073686e+00)(6.800000e-02,1.075973e+00)(6.900000e-02,1.078298e+00)(7.000000e-02,1.080662e+00)(7.100000e-02,1.083064e+00)(7.200000e-02,1.085505e+00)(7.300000e-02,1.087986e+00)(7.400000e-02,1.090505e+00)(7.500000e-02,1.093064e+00)(7.600000e-02,1.095663e+00)(7.700000e-02,1.098301e+00)(7.800000e-02,1.100980e+00)(7.900000e-02,1.103699e+00)(8.000000e-02,1.106459e+00)(8.100000e-02,1.109260e+00)(8.200000e-02,1.112102e+00)(8.300000e-02,1.114985e+00)(8.400000e-02,1.117910e+00)(8.500000e-02,1.120876e+00)(8.600000e-02,1.123885e+00)(8.700000e-02,1.126937e+00)(8.800000e-02,1.130030e+00)(8.900000e-02,1.133167e+00)(9.000000e-02,1.136348e+00)(9.100000e-02,1.139571e+00)(9.200000e-02,1.142839e+00)(9.300000e-02,1.146150e+00)(9.400000e-02,1.149506e+00)(9.500000e-02,1.152907e+00)(9.600000e-02,1.156353e+00)(9.700000e-02,1.159844e+00)(9.800000e-02,1.163381e+00)(9.900000e-02,1.166963e+00)(1.000000e-01,1.170592e+00)(1.010000e-01,1.174268e+00)(1.020000e-01,1.177990e+00)(1.030000e-01,1.181760e+00)(1.040000e-01,1.185577e+00)(1.050000e-01,1.189443e+00)(1.060000e-01,1.193356e+00)(1.070000e-01,1.197319e+00)(1.080000e-01,1.201330e+00)(1.090000e-01,1.205391e+00)(1.100000e-01,1.209502e+00)(1.110000e-01,1.213663e+00)(1.120000e-01,1.217874e+00)(1.130000e-01,1.222137e+00)(1.140000e-01,1.226450e+00)(1.150000e-01,1.230816e+00)(1.160000e-01,1.235234e+00)(1.170000e-01,1.239704e+00)(1.180000e-01,1.244227e+00)(1.190000e-01,1.248804e+00)(1.200000e-01,1.253434e+00)(1.210000e-01,1.258119e+00)(1.220000e-01,1.262859e+00)(1.230000e-01,1.267654e+00)(1.240000e-01,1.272504e+00)(1.250000e-01,1.277411e+00)(1.260000e-01,1.282374e+00)(1.270000e-01,1.287394e+00)(1.280000e-01,1.292472e+00)(1.290000e-01,1.297607e+00)(1.300000e-01,1.302802e+00)(1.310000e-01,1.308055e+00)(1.320000e-01,1.313368e+00)(1.330000e-01,1.318740e+00)(1.340000e-01,1.324174e+00)(1.350000e-01,1.329668e+00)(1.360000e-01,1.335224e+00)(1.370000e-01,1.340842e+00)(1.380000e-01,1.346523e+00)(1.390000e-01,1.352268e+00)(1.400000e-01,1.358076e+00)(1.410000e-01,1.363948e+00)(1.420000e-01,1.369886e+00)(1.430000e-01,1.375889e+00)(1.440000e-01,1.381958e+00)(1.450000e-01,1.388094e+00)(1.460000e-01,1.394297e+00)(1.470000e-01,1.400569e+00)(1.480000e-01,1.406909e+00)(1.490000e-01,1.413318e+00)(1.500000e-01,1.419797e+00)(1.510000e-01,1.426347e+00)(1.520000e-01,1.432968e+00)(1.530000e-01,1.439661e+00)(1.540000e-01,1.446426e+00)(1.550000e-01,1.453265e+00)(1.560000e-01,1.460178e+00)(1.570000e-01,1.467166e+00)(1.580000e-01,1.474228e+00)(1.590000e-01,1.481367e+00)(1.600000e-01,1.488583e+00)(1.610000e-01,1.495877e+00)(1.620000e-01,1.503249e+00)(1.630000e-01,1.510700e+00)(1.640000e-01,1.518231e+00)(1.650000e-01,1.525842e+00)(1.660000e-01,1.533535e+00)(1.670000e-01,1.541310e+00)(1.680000e-01,1.549169e+00)(1.690000e-01,1.557111e+00)(1.700000e-01,1.565138e+00)(1.710000e-01,1.573251e+00)(1.720000e-01,1.581450e+00)(1.730000e-01,1.589737e+00)(1.740000e-01,1.598112e+00)(1.750000e-01,1.606576e+00)(1.760000e-01,1.615130e+00)(1.770000e-01,1.623775e+00)(1.780000e-01,1.632511e+00)(1.790000e-01,1.641341e+00)(1.800000e-01,1.650265e+00)(1.810000e-01,1.659283e+00)(1.820000e-01,1.668397e+00)(1.830000e-01,1.677607e+00)(1.840000e-01,1.686916e+00)(1.850000e-01,1.696323e+00)(1.860000e-01,1.705830e+00)(1.870000e-01,1.715438e+00)(1.880000e-01,1.725147e+00)(1.890000e-01,1.734960e+00)(1.900000e-01,1.744877e+00)(1.910000e-01,1.754898e+00)(1.920000e-01,1.765027e+00)(1.930000e-01,1.775262e+00)(1.940000e-01,1.785606e+00)(1.950000e-01,1.796060e+00)(1.960000e-01,1.806625e+00)(1.970000e-01,1.817302e+00)(1.980000e-01,1.828092e+00)(1.990000e-01,1.838996e+00)(2.000000e-01,1.850017e+00)

};\addlegendentry{$\sigma=0.2$}           
\addplot[name path=f,color=black] coordinates{
(1.000000e-03,9.999192e-01)(2.000000e-03,1.000238e+00)(3.000000e-03,1.000573e+00)(4.000000e-03,1.001019e+00)(5.000000e-03,1.001593e+00)(6.000000e-03,1.002294e+00)(7.000000e-03,1.003124e+00)(8.000000e-03,1.004082e+00)(9.000000e-03,1.005168e+00)(1.000000e-02,1.006384e+00)(1.100000e-02,1.007729e+00)(1.200000e-02,1.009204e+00)(1.300000e-02,1.010809e+00)(1.400000e-02,1.012545e+00)(1.500000e-02,1.014412e+00)(1.600000e-02,1.016412e+00)(1.700000e-02,1.018544e+00)(1.800000e-02,1.020810e+00)(1.900000e-02,1.023210e+00)(2.000000e-02,1.025745e+00)(2.100000e-02,1.028416e+00)(2.200000e-02,1.031224e+00)(2.300000e-02,1.034169e+00)(2.400000e-02,1.037253e+00)(2.500000e-02,1.040477e+00)(2.600000e-02,1.043841e+00)(2.700000e-02,1.047348e+00)(2.800000e-02,1.050997e+00)(2.900000e-02,1.054791e+00)(3.000000e-02,1.058731e+00)(3.100000e-02,1.062817e+00)(3.200000e-02,1.067052e+00)(3.300000e-02,1.071437e+00)(3.400000e-02,1.075973e+00)(3.500000e-02,1.080662e+00)(3.600000e-02,1.085505e+00)(3.700000e-02,1.090505e+00)(3.800000e-02,1.095663e+00)(3.900000e-02,1.100980e+00)(4.000000e-02,1.106459e+00)(4.100000e-02,1.112102e+00)(4.200000e-02,1.117910e+00)(4.300000e-02,1.123885e+00)(4.400000e-02,1.130030e+00)(4.500000e-02,1.136348e+00)(4.600000e-02,1.142839e+00)(4.700000e-02,1.149506e+00)(4.800000e-02,1.156353e+00)(4.900000e-02,1.163381e+00)(5.000000e-02,1.170592e+00)(5.100000e-02,1.177990e+00)(5.200000e-02,1.185577e+00)(5.300000e-02,1.193356e+00)(5.400000e-02,1.201330e+00)(5.500000e-02,1.209502e+00)(5.600000e-02,1.217874e+00)(5.700000e-02,1.226450e+00)(5.800000e-02,1.235234e+00)(5.900000e-02,1.244227e+00)(6.000000e-02,1.253434e+00)(6.100000e-02,1.262859e+00)(6.200000e-02,1.272504e+00)(6.300000e-02,1.282374e+00)(6.400000e-02,1.292472e+00)(6.500000e-02,1.302802e+00)(6.600000e-02,1.313368e+00)(6.700000e-02,1.324174e+00)(6.800000e-02,1.335224e+00)(6.900000e-02,1.346523e+00)(7.000000e-02,1.358076e+00)(7.100000e-02,1.369886e+00)(7.200000e-02,1.381958e+00)(7.300000e-02,1.394297e+00)(7.400000e-02,1.406909e+00)(7.500000e-02,1.419797e+00)(7.600000e-02,1.432968e+00)(7.700000e-02,1.446426e+00)(7.800000e-02,1.460178e+00)(7.900000e-02,1.474228e+00)(8.000000e-02,1.488583e+00)(8.100000e-02,1.503249e+00)(8.200000e-02,1.518231e+00)(8.300000e-02,1.533535e+00)(8.400000e-02,1.549169e+00)(8.500000e-02,1.565138e+00)(8.600000e-02,1.581450e+00)(8.700000e-02,1.598112e+00)(8.800000e-02,1.615130e+00)(8.900000e-02,1.632511e+00)(9.000000e-02,1.650265e+00)(9.100000e-02,1.668397e+00)(9.200000e-02,1.686916e+00)(9.300000e-02,1.705830e+00)(9.400000e-02,1.725147e+00)(9.500000e-02,1.744877e+00)(9.600000e-02,1.765027e+00)(9.700000e-02,1.785606e+00)(9.800000e-02,1.806625e+00)(9.900000e-02,1.828092e+00)(1.000000e-01,1.850017e+00)(1.010000e-01,1.872410e+00)(1.020000e-01,1.895281e+00)(1.030000e-01,1.918642e+00)(1.040000e-01,1.942502e+00)(1.050000e-01,1.966873e+00)(1.060000e-01,1.991766e+00)(1.070000e-01,2.017194e+00)(1.080000e-01,2.043167e+00)(1.090000e-01,2.069699e+00)(1.100000e-01,2.096802e+00)(1.110000e-01,2.124488e+00)(1.120000e-01,2.152773e+00)(1.130000e-01,2.181668e+00)(1.140000e-01,2.211189e+00)(1.150000e-01,2.241349e+00)(1.160000e-01,2.272165e+00)(1.170000e-01,2.303650e+00)(1.180000e-01,2.335821e+00)(1.190000e-01,2.368694e+00)(1.200000e-01,2.402285e+00)(1.210000e-01,2.436612e+00)(1.220000e-01,2.471692e+00)(1.230000e-01,2.507543e+00)(1.240000e-01,2.544184e+00)(1.250000e-01,2.581632e+00)(1.260000e-01,2.619910e+00)(1.270000e-01,2.659035e+00)(1.280000e-01,2.699029e+00)(1.290000e-01,2.739912e+00)(1.300000e-01,2.781708e+00)(1.310000e-01,2.824438e+00)(1.320000e-01,2.868125e+00)(1.330000e-01,2.912793e+00)(1.340000e-01,2.958466e+00)(1.350000e-01,3.005170e+00)(1.360000e-01,3.052929e+00)(1.370000e-01,3.101771e+00)(1.380000e-01,3.151723e+00)(1.390000e-01,3.202812e+00)(1.400000e-01,3.255068e+00)(1.410000e-01,3.308521e+00)(1.420000e-01,3.363199e+00)(1.430000e-01,3.419136e+00)(1.440000e-01,3.476363e+00)(1.450000e-01,3.534914e+00)(1.460000e-01,3.594822e+00)(1.470000e-01,3.656122e+00)(1.480000e-01,3.718852e+00)(1.490000e-01,3.783047e+00)(1.500000e-01,3.848747e+00)(1.510000e-01,3.915991e+00)(1.520000e-01,3.984818e+00)(1.530000e-01,4.055272e+00)(1.540000e-01,4.127394e+00)(1.550000e-01,4.201230e+00)(1.560000e-01,4.276824e+00)(1.570000e-01,4.354224e+00)(1.580000e-01,4.433478e+00)(1.590000e-01,4.514635e+00)(1.600000e-01,4.597747e+00)(1.610000e-01,4.682867e+00)(1.620000e-01,4.770048e+00)(1.630000e-01,4.859347e+00)(1.640000e-01,4.950821e+00)(1.650000e-01,5.044529e+00)(1.660000e-01,5.140534e+00)(1.670000e-01,5.238897e+00)(1.680000e-01,5.339684e+00)(1.690000e-01,5.442961e+00)(1.700000e-01,5.548798e+00)(1.710000e-01,5.657265e+00)(1.720000e-01,5.768435e+00)(1.730000e-01,5.882385e+00)(1.740000e-01,5.999192e+00)(1.750000e-01,6.118935e+00)(1.760000e-01,6.241698e+00)(1.770000e-01,6.367565e+00)(1.780000e-01,6.496625e+00)(1.790000e-01,6.628969e+00)(1.800000e-01,6.764688e+00)(1.810000e-01,6.903881e+00)(1.820000e-01,7.046645e+00)(1.830000e-01,7.193084e+00)(1.840000e-01,7.343303e+00)(1.850000e-01,7.497412e+00)(1.860000e-01,7.655522e+00)(1.870000e-01,7.817749e+00)(1.880000e-01,7.984214e+00)(1.890000e-01,8.155039e+00)(1.900000e-01,8.330353e+00)(1.910000e-01,8.510286e+00)(1.920000e-01,8.694975e+00)(1.930000e-01,8.884559e+00)(1.940000e-01,9.079183e+00)(1.950000e-01,9.278997e+00)(1.960000e-01,9.484153e+00)(1.970000e-01,9.694812e+00)(1.980000e-01,9.911138e+00)(1.990000e-01,1.013330e+01)(2.000000e-01,1.036147e+01)
};\addlegendentry{$\sigma=0.1$}    
\node at (axis cs:0.07,2.5)(source1){$\scriptstyle{\delta_\star(\epsilon=0.1,\sigma=0.1)}$}; 
\node at (axis cs:0.17,2.5)(source2){$\scriptstyle{\delta_\star(\epsilon=0.2,\sigma=0.2)}$}; 
\node[anchor=north,inner sep=0] (destination1) at (axis cs:0.1,1.9) {$\bullet$};
\node[anchor=north,inner sep=0] (destination2) at (axis cs:0.2,1.9) {$\bullet$};
\draw[->](source1)--(destination1);
\draw[->](source2)--(destination2);
\end{axis}
\end{tikzpicture}
\end{center}
\caption{Theoretical predictions of $\delta_\star$ as a function of $\epsilon$ for different noise variance values. The figure shows that every $\delta$ can be forced to be in the feasibility region by appropriately choosing $\epsilon$. }
\label{fig:deltastar}
\end{figure}

Having characterized the feasibility region of the H-SVR, we are now ready to provide sharp asymptotics of its performance in terms of the prediction risk and the cosine similarity. 

\begin{theorem}[Convergence of the H-SVR]
Define function $\mathcal{D}:\mathbb{R}^2\to \mathbb{R}$ as:
$$
\mathcal{D}(\tilde{\gamma_1},\tilde{\gamma_2})= \sqrt{\delta}\sqrt{\mathbb{E}\left(|\sqrt{\tilde{\gamma}_1^2+\tilde{\gamma}_2^2}{G}+ N| -\frac{\epsilon}{\sigma}\right)_{+}^2} -\tilde{\gamma}_1
$$
where the expectation is taken over the distribution of the independent random variables ${G}$ and $N$ drawn respectively from the standard normal distribution $\mathcal{N}(0,1)$ and $p_N$.  
Let $\tilde{\gamma}_1^\star$ and $\tilde{\gamma}_2^\star$ be the unique solutions to the following optimization problem:
\begin{align}
(\tilde{\gamma}_1^\star,\tilde{\gamma}_2^\star)= \arg\min_{\substack{\tilde{\gamma}_1,\tilde{\gamma}_2\\ \mathcal{D}(\tilde{\gamma}_1,\tilde{\gamma}_{2})\leq 0}} \frac{1}{2}(\tilde{\gamma}_2-\frac{\beta}{\sigma})^2+\frac{1}{2}\tilde{\gamma}_1^2
\label{eq:unique}
\end{align}
Let $\hat{\bf w}_H$ be the solution to the H-SVR. Then, under   Assumption  \ref{ass:regime} and assuming $\delta<\delta_\star$, the following convergences hold true:
$$
\|\hat{\bf w}_H-\boldsymbol{\beta}_\star\|_2\asto \sigma\sqrt{(\tilde{\gamma}_1^\star)^2+(\tilde{\gamma}_2^\star)^2}
$$
and
$$
\frac{\hat{\bf w}_H^{T}\boldsymbol{\beta}_\star}{\|\hat{\bf w}_H\|_2\|\boldsymbol{\beta}_\star\|_2}\asto\frac{\frac{\beta}{\sigma}-\tilde{\gamma}_2^\star}{\sqrt{(\tilde{\gamma}_1^\star)^2+(\tilde{\gamma}_2^\star-\frac{\beta}{\sigma})^2}} 
$$
from which we deduce that the risk of the H-SVR converges to:
$$
\mathcal{R}(\hat{\bf w}_H)-\overline{R}_H\asto 0
$$
where $\overline{R}_H=\sigma^2((\tilde{\gamma}_1^\star)^2+(\tilde{\gamma}_2^\star)^2)$. 
\label{pred_risk_conv_HM}
\end{theorem}

\begin{remark}{({\bf Behavior of the H-SVR as $\delta\to 0$}).} As $\delta$ tending to zero, one can check after careful investigation of the asymptotic expressions that $\tilde{\gamma}_2^\star\to\frac{\beta}{\sigma}$ and $\tilde{\gamma}_1^\star\to 0$. In this case, the asymptotic risk is thus given by $\beta^2$. 
 To see this, it suffices to note that $\mathcal{D}(\delta^{\frac{1}{4}},\frac{\beta}{\sigma})$ converges from below to zero as $\delta\downarrow 0$. By continuity of $\mathcal{D}$, we may find $\eta$ sufficiently small such that for all $(\tilde{\gamma}_1,\tilde{\gamma}_2)$ satisfying: 
$$
(\tilde{\gamma}_1,\tilde{\gamma}_2)\in \mathcal{C}(\eta):=\left\{(\tilde{\gamma}_1,\tilde{\gamma}_2), \  | \ \tilde{\gamma}_1\in(0,\eta) \ \textnormal{and} \ \tilde{\gamma}_2\in(\frac{\beta}{\sigma}-\eta,\frac{\beta}{\sigma}+\eta)\right\}
$$
we have $\mathcal{D}(\tilde{\gamma}_1,\tilde{\gamma}_2)\leq 0$. Moreover, it is easy to see that the objective in \eqref{eq:unique} can be bounded by $\eta^2$ when $(\tilde{\gamma}_1,\tilde{\gamma}_2)\in\mathcal{C}(\eta)$. Evaluation of this objective when $|{\tilde{\gamma}_2}-\frac{\beta}{\sigma}|\geq \sqrt{2}\eta$  or when $|\tilde{\gamma}_1| \geq \sqrt{2}\eta$ yields values greater than $\eta^2$. Hence, necessarily, $(\tilde{\gamma}_1,\tilde{\gamma}_2)\in \mathcal{C}(\sqrt{2}\eta)$, which proves the desired. Finally, observing that the risk of the null estimator $\hat{\boldsymbol{\beta}}= 0$ is also $\beta^2$, we conclude that the H-SVR is no better than the null estimator when a small number of samples is employed.  
\end{remark}    
\begin{remark}{({\bf Behavior of the H-SVR as $\delta\to\delta_\star$}).} As $\delta\to\delta_\star$, the set $\left\{(\tilde{\gamma}_1,\tilde{\gamma_2}) \ | \ \mathcal{D}(\tilde{\gamma}_1,\tilde{\gamma_2})\leq 0\right\}$ becomes the unit set $\left\{\left(\tilde{\gamma}_1^\circ,0\right)\right\}$ where $\tilde{\gamma}_1^\circ$ is the smallest solution to the equation $\mathcal{D}(\tilde{\gamma}_1,0)=0$. The asymptotic risk thus becomes equal to $\sigma^2(\tilde{\gamma}_1^\circ)^2$ and is as such independent of the SNR $\frac{\beta^2}{\sigma^2}$. As a result, when $\delta$ approaches $\delta_\star$, it is the noise variance and the value of $\epsilon$ that determines the performance of the H-SVR and not the SNR. This is in opposition to the behavior in the operation region $\delta\to 0$, for which the risk tends to $\beta^2$.    
\end{remark}

\subsection{ Soft SVR}
In this section, the soft SVR problem is considered and the convergence of the corresponding prediction risk is established.
\begin{theorem}Under Assumption 1, we have
$$
\mathcal{R}_{\rm Pred}\asto \sigma^2(\tilde \gamma_1^{*2}+\tilde \gamma_2^{*2})
$$
where $\tilde \gamma_1^*$ and $\tilde \gamma_2^*$ are the solutions of the following scalar optimization problem
\begin{align*}
\overline\phi&=\min_{\tilde\gamma_1,\tilde\gamma_2}\sup_{\chi>0}\overline{D}(\tilde \gamma_1,\tilde\gamma_2,\chi)
\end{align*}
with
\begin{align*}
\overline{D}(\tilde\gamma_1,\tilde\gamma_2,\chi)=&\frac{\delta}{\sigma}  \mathbb{E}\left\{C\left[\left(\left|\sqrt{\tilde\gamma_1^2+\tilde\gamma_2^2}G+ N\right|-{\epsilon}/{\sigma}\right)_+ - \frac{C \tilde\gamma_1}{2\chi}\right]\mathbbm{1}_{\big\{\left(\left|\sqrt{\tilde\gamma_1^2+\tilde\gamma_2^2}G+ N\right|-{\epsilon}/{\sigma}\right)_+ \chi> \tilde \gamma_1C\big\}}\right. \\ & \left. +\frac{\chi}{2\tilde\gamma_1}\left(\left|\sqrt{\tilde\gamma_1^2+\tilde\gamma_2^2}G+\ N\right|-\epsilon/\sigma\right)_+^2\mathbbm{1}_{\big\{\left(\left|\sqrt{\tilde\gamma_1^2+\tilde\gamma_2^2}G+ N\right|-\epsilon/\sigma\right)_+ \chi \leq \tilde\gamma_1  C\big\}}\right\}-\frac{\tilde\gamma_1\chi}{2\sigma}+ \frac{1}{2}\tilde\gamma_1^2+ \frac{1}{2}\left(\tilde\gamma_2-\frac{\beta}{\sigma}\right)^2,
\end{align*}
where the expectation is with respect to the distributions of $G$ and $N$ with $G\sim \mathcal{N}(0,1)$ and $N\sim p_N$.
\label{soft_thm}
\end{theorem}

\begin{remark}
Theorem \ref{soft_thm} can be used to optimally tune the parameters $(\epsilon,C)$ so that they minimize the asymptotic test risk. For that, one is required to estimate the noise variance $\sigma^2$ and the signal power $\beta^2$. 
These can be easily estimated through the following approach. Let ${\bf y}=[y_1,\cdots,y_n]^{T}$ be the vector of the responses and ${\bf X}=[{\bf x}_1,\cdots, {\bf x}_n]$ the matrix stacking all training samples. Then,
$$
{\bf y}={\bf X}^{T}\boldsymbol{\beta}_\star+\sigma {\bf n}
$$  
where ${\bf n}=[{n}_1,\cdots,n_n]^{T}$. From the strong law of large numbers,
\begin{equation}
\frac{1}{n}{\bf y}^{T}{\bf y}\asto \sigma^2+\beta^2
\label{eq:sig}
\end{equation}
On the other hand, assuming  $\delta>1$, 
\begin{equation}
\frac{1}{n}{\bf y}^{T}({\bf I}_n-{\bf X}^{T}({\bf X}{\bf X}^{T})^{-1}{\bf X}){\bf y} \asto\sigma^2(1-\frac{1}{\delta}) \label{eq:noise}
\end{equation}
Combining \eqref{eq:sig} and \eqref{eq:noise},  consistent estimators for the noise variance $\sigma^2$ and for $\beta^2$ are given by:
\begin{align}
&\hat{\sigma}^2=\frac{\frac{1}{n}{\bf y}^{T}({\bf I}_n-{\bf X}^{T}({\bf X}{\bf X}^{T})^{-1}{\bf X}){\bf y}}{1-\frac{1}{\delta}}\\
&\hat{\beta}^2=\frac{1}{n}{\bf y}^T{\bf y}-\frac{\frac{1}{n}{\bf y}^{T}({\bf I}_n-{\bf X}^{T}({\bf X}{\bf X}^{T})^{-1}{\bf X}){\bf y}}{1-\frac{1}{\delta}}
\end{align}
In case $\delta<1$, the problem of estimating the noise variance becomes more challenging, and there are, to the best of our knowledge, no general unbiased estimators with the same statistical guarantees as in the case $\delta >1$. To address this issue, some other techniques may be used \cite{cherkassky2004practical} but they are not guaranteed to lead to consistent estimators. 
\end{remark}
\begin{remark}{\bf Behavior of the S-SVR when $\delta\to 0$.}
The test risk of the S-SVR is much more involved than that of the H-SVR. Nevertheless, it can be easily seen that when $\delta$ goes to zero,
\begin{align}
\lim_{\delta\to 0}\sup_{\chi\geq 0}\overline{D}(\tilde{\gamma}_1,\tilde{\gamma}_2,\chi)&\stackrel{(a)}{=}\inf_{\delta\geq 0}\sup_{\chi\geq 0}\overline{D}(\tilde{\gamma}_1,\tilde{\gamma}_2,\chi)\\
&\stackrel{(b)}{=}\sup_{\chi\geq 0}\inf_{\delta\geq 0}\overline{D}(\tilde{\gamma}_1,\tilde{\gamma}_2,\chi)\\
&=\sup_{\chi\geq 0}\lim_{\delta\to 0}\overline{D}(\tilde{\gamma}_1,\tilde{\gamma}_2,\chi)\\
&=\frac{1}{2}\tilde{\gamma}_1^2+\frac{1}{2}(\tilde{\gamma}_2-\frac{\beta}{\sigma})^2 \label{eq:g1g2}
\end{align}
where $(a)$ follows from the fact that the objective function is an increasing function in $\delta$ and $(b)$ from the fact that the objective function is convex in $\delta$ and concave in $\chi$. 
The asymptotic limit of $\overline{D}$ in \eqref{eq:g1g2} has a unique minimum given by $\tilde{\gamma}_2=\frac{\beta}{\sigma}$ and $\tilde{\gamma}_1=0$. Plugging these values into that of the test risk, we conclude that when $\delta$ goes to zero, the test risk of the S-SVR converges to that of the null estimator. 
\end{remark}

\section{Numerical illustration}
\label{Numerical_illustration}
\subsection{H-SVR}
\begin{figure}[]
\begin{center}
\subfigure[Risk]{
\begin{tikzpicture}[scale=0.9,font=\fontsize{10}{10}\selectfont]
   \tikzstyle{every axis y label}+=[yshift=0pt]
   \tikzstyle{every axis x label}+=[yshift=5pt]
   \tikzstyle{every axis legend}+=[cells={anchor=west},fill=white,
        at={(0.98,0.98)}, anchor=north east, font=\fontsize{10}{10}\selectfont]
   \begin{axis}[
      xmin=0,
      ymin=0,
      xmax=2,
      ymax=5,
xlabel={$\delta$},
ylabel={Risk }	]		
\addplot[color=black] coordinates{
(1.000000e-02,2.505003e-01)(2.000000e-02,2.510010e-01)(3.000000e-02,2.516026e-01)(4.000000e-02,2.521044e-01)(5.000000e-02,2.527073e-01)(6.000000e-02,2.533109e-01)(7.000000e-02,2.538144e-01)(8.000000e-02,2.544194e-01)(9.000000e-02,2.550250e-01)(1.000000e-01,2.556314e-01)(1.100000e-01,2.562384e-01)(1.200000e-01,2.569476e-01)(1.300000e-01,2.575563e-01)(1.400000e-01,2.581656e-01)(1.500000e-01,2.588774e-01)(1.600000e-01,2.595903e-01)(1.700000e-01,2.602020e-01)(1.800000e-01,2.609166e-01)(1.900000e-01,2.616323e-01)(2.000000e-01,2.623488e-01)(2.100000e-01,2.630664e-01)(2.200000e-01,2.638877e-01)(2.300000e-01,2.646074e-01)(2.400000e-01,2.654310e-01)(2.500000e-01,2.661528e-01)(2.600000e-01,2.669789e-01)(2.700000e-01,2.678062e-01)(2.800000e-01,2.686349e-01)(2.900000e-01,2.694648e-01)(3.000000e-01,2.702960e-01)(3.100000e-01,2.711285e-01)(3.200000e-01,2.720666e-01)(3.300000e-01,2.729018e-01)(3.400000e-01,2.738429e-01)(3.500000e-01,2.747856e-01)(3.600000e-01,2.757300e-01)(3.700000e-01,2.766760e-01)(3.800000e-01,2.776236e-01)(3.900000e-01,2.785728e-01)(4.000000e-01,2.795237e-01)(4.100000e-01,2.805821e-01)(4.200000e-01,2.815364e-01)(4.300000e-01,2.825986e-01)(4.400000e-01,2.836628e-01)(4.500000e-01,2.847290e-01)(4.600000e-01,2.857972e-01)(4.700000e-01,2.869745e-01)(4.800000e-01,2.880469e-01)(4.900000e-01,2.892288e-01)(5.000000e-01,2.903054e-01)(5.100000e-01,2.914920e-01)(5.200000e-01,2.926810e-01)(5.300000e-01,2.938724e-01)(5.400000e-01,2.951749e-01)(5.500000e-01,2.963714e-01)(5.600000e-01,2.976794e-01)(5.700000e-01,2.989902e-01)(5.800000e-01,3.001944e-01)(5.900000e-01,3.015108e-01)(6.000000e-01,3.029402e-01)(6.100000e-01,3.042626e-01)(6.200000e-01,3.056984e-01)(6.300000e-01,3.070268e-01)(6.400000e-01,3.084692e-01)(6.500000e-01,3.099149e-01)(6.600000e-01,3.113640e-01)(6.700000e-01,3.129284e-01)(6.800000e-01,3.143845e-01)(6.900000e-01,3.159564e-01)(7.000000e-01,3.175323e-01)(7.100000e-01,3.191120e-01)(7.200000e-01,3.206957e-01)(7.300000e-01,3.222833e-01)(7.400000e-01,3.239886e-01)(7.500000e-01,3.255844e-01)(7.600000e-01,3.272984e-01)(7.700000e-01,3.291317e-01)(7.800000e-01,3.308550e-01)(7.900000e-01,3.325829e-01)(8.000000e-01,3.344309e-01)(8.100000e-01,3.362840e-01)(8.200000e-01,3.381423e-01)(8.300000e-01,3.401222e-01)(8.400000e-01,3.419910e-01)(8.500000e-01,3.439823e-01)(8.600000e-01,3.459792e-01)(8.700000e-01,3.479820e-01)(8.800000e-01,3.501089e-01)(8.900000e-01,3.521236e-01)(9.000000e-01,3.542630e-01)(9.100000e-01,3.564090e-01)(9.200000e-01,3.586812e-01)(9.300000e-01,3.608405e-01)(9.400000e-01,3.631268e-01)(9.500000e-01,3.654203e-01)(9.600000e-01,3.678423e-01)(9.700000e-01,3.701506e-01)(9.800000e-01,3.725882e-01)(9.900000e-01,3.750338e-01)(1,3.776103e-01)(1.010000e+00,3.801956e-01)(1.020000e+00,3.827897e-01)(1.030000e+00,3.853926e-01)(1.040000e+00,3.880044e-01)(1.050000e+00,3.907500e-01)(1.060000e+00,3.936308e-01)(1.070000e+00,3.963962e-01)(1.080000e+00,3.992976e-01)(1.090000e+00,4.022096e-01)(1.100000e+00,4.052596e-01)(1.110000e+00,4.083210e-01)(1.120000e+00,4.113940e-01)(1.130000e+00,4.144784e-01)(1.140000e+00,4.177037e-01)(1.150000e+00,4.210712e-01)(1.160000e+00,4.243220e-01)(1.170000e+00,4.277160e-01)(1.180000e+00,4.312549e-01)(1.190000e+00,4.348084e-01)(1.200000e+00,4.383764e-01)(1.210000e+00,4.420920e-01)(1.220000e+00,4.458233e-01)(1.230000e+00,4.497044e-01)(1.240000e+00,4.536022e-01)(1.250000e+00,4.575170e-01)(1.260000e+00,4.615844e-01)(1.270000e+00,4.658062e-01)(1.280000e+00,4.700474e-01)(1.290000e+00,4.744454e-01)(1.300000e+00,4.788640e-01)(1.310000e+00,4.834421e-01)(1.320000e+00,4.880420e-01)(1.330000e+00,4.928040e-01)(1.340000e+00,4.977303e-01)(1.350000e+00,5.026810e-01)(1.360000e+00,5.077988e-01)(1.370000e+00,5.129424e-01)(1.380000e+00,5.184000e-01)(1.390000e+00,5.238864e-01)(1.400000e+00,5.294018e-01)(1.410000e+00,5.352386e-01)(1.420000e+00,5.411074e-01)(1.430000e+00,5.473040e-01)(1.440000e+00,5.535360e-01)(1.450000e+00,5.599529e-01)(1.460000e+00,5.665573e-01)(1.470000e+00,5.733518e-01)(1.480000e+00,5.803392e-01)(1.490000e+00,5.875223e-01)(1.500000e+00,5.949037e-01)(1.510000e+00,6.026417e-01)(1.520000e+00,6.105860e-01)(1.530000e+00,6.187396e-01)(1.540000e+00,6.271056e-01)(1.550000e+00,6.358468e-01)(1.560000e+00,6.449696e-01)(1.570000e+00,6.543192e-01)(1.580000e+00,6.640620e-01)(1.590000e+00,6.740410e-01)(1.600000e+00,6.845908e-01)(1.610000e+00,6.955560e-01)(1.620000e+00,7.069446e-01)(1.630000e+00,7.187648e-01)(1.640000e+00,7.311960e-01)(1.650000e+00,7.442513e-01)(1.660000e+00,7.579444e-01)(1.670000e+00,7.722894e-01)(1.680000e+00,7.874788e-01)(1.690000e+00,8.033537e-01)(1.700000e+00,8.202925e-01)(1.710000e+00,8.383234e-01)(1.720000e+00,8.574760e-01)(1.730000e+00,8.781564e-01)(1.740000e+00,9.002214e-01)(1.750000e+00,9.240977e-01)(1.760000e+00,9.502350e-01)(1.770000e+00,9.787145e-01)(1.780000e+00,1.010427e+00)(1.790000e+00,1.045915e+00)(1.800000e+00,1.086389e+00)(1.810000e+00,1.133586e+00)(1.820000e+00,1.190499e+00)(1.830000e+00,1.263151e+00)(1.840000e+00,1.366795e+00)
};\addlegendentry{$\beta=0.5$}   

\addplot[color=blue] coordinates{

(1.000000e-02,9.960040e-01)(2.000000e-02,9.920160e-01)(3.000000e-02,9.880360e-01)(4.000000e-02,9.842624e-01)(5.000000e-02,9.802980e-01)(6.000000e-02,9.765392e-01)(7.000000e-02,9.725904e-01)(8.000000e-02,9.688465e-01)(9.000000e-02,9.651098e-01)(1.000000e-01,9.613803e-01)(1.100000e-01,9.576580e-01)(1.200000e-01,9.539429e-01)(1.300000e-01,9.502350e-01)(1.400000e-01,9.465344e-01)(1.500000e-01,9.430352e-01)(1.600000e-01,9.393486e-01)(1.700000e-01,9.358628e-01)(1.800000e-01,9.321902e-01)(1.900000e-01,9.287177e-01)(2.000000e-01,9.252516e-01)(2.100000e-01,9.217920e-01)(2.200000e-01,9.183389e-01)(2.300000e-01,9.148923e-01)(2.400000e-01,9.116430e-01)(2.500000e-01,9.082090e-01)(2.600000e-01,9.049717e-01)(2.700000e-01,9.015502e-01)(2.800000e-01,8.983248e-01)(2.900000e-01,8.951052e-01)(3.000000e-01,8.918914e-01)(3.100000e-01,8.886833e-01)(3.200000e-01,8.854810e-01)(3.300000e-01,8.824724e-01)(3.400000e-01,8.792813e-01)(3.500000e-01,8.762832e-01)(3.600000e-01,8.731034e-01)(3.700000e-01,8.701158e-01)(3.800000e-01,8.671334e-01)(3.900000e-01,8.641562e-01)(4.000000e-01,8.611840e-01)(4.100000e-01,8.584023e-01)(4.200000e-01,8.554400e-01)(4.300000e-01,8.526676e-01)(4.400000e-01,8.497152e-01)(4.500000e-01,8.469521e-01)(4.600000e-01,8.441934e-01)(4.700000e-01,8.414393e-01)(4.800000e-01,8.388728e-01)(4.900000e-01,8.361274e-01)(5.000000e-01,8.333864e-01)(5.100000e-01,8.308323e-01)(5.200000e-01,8.282820e-01)(5.300000e-01,8.257357e-01)(5.400000e-01,8.231933e-01)(5.500000e-01,8.206548e-01)(5.600000e-01,8.183012e-01)(5.700000e-01,8.157702e-01)(5.800000e-01,8.134236e-01)(5.900000e-01,8.110804e-01)(6.000000e-01,8.087405e-01)(6.100000e-01,8.064040e-01)(6.200000e-01,8.040709e-01)(6.300000e-01,8.019203e-01)(6.400000e-01,7.997725e-01)(6.500000e-01,7.974490e-01)(6.600000e-01,7.953072e-01)(6.700000e-01,7.933465e-01)(6.800000e-01,7.912103e-01)(6.900000e-01,7.890769e-01)(7.000000e-01,7.871238e-01)(7.100000e-01,7.851732e-01)(7.200000e-01,7.832250e-01)(7.300000e-01,7.812792e-01)(7.400000e-01,7.795124e-01)(7.500000e-01,7.775712e-01)(7.600000e-01,7.758086e-01)(7.700000e-01,7.740480e-01)(7.800000e-01,7.724652e-01)(7.900000e-01,7.707084e-01)(8.000000e-01,7.691290e-01)(8.100000e-01,7.673760e-01)(8.200000e-01,7.658000e-01)(8.300000e-01,7.644005e-01)(8.400000e-01,7.628276e-01)(8.500000e-01,7.614308e-01)(8.600000e-01,7.600352e-01)(8.700000e-01,7.586410e-01)(8.800000e-01,7.572480e-01)(8.900000e-01,7.560303e-01)(9.000000e-01,7.548134e-01)(9.100000e-01,7.535976e-01)(9.200000e-01,7.523828e-01)(9.300000e-01,7.511689e-01)(9.400000e-01,7.501292e-01)(9.500000e-01,7.490903e-01)(9.600000e-01,7.482250e-01)(9.700000e-01,7.471874e-01)(9.800000e-01,7.463232e-01)(9.900000e-01,7.454596e-01)(1,7.445964e-01)(1.010000e+00,7.439063e-01)(1.020000e+00,7.432164e-01)(1.030000e+00,7.425269e-01)(1.040000e+00,7.420100e-01)(1.050000e+00,7.414932e-01)(1.060000e+00,7.409766e-01)(1.070000e+00,7.404603e-01)(1.080000e+00,7.401161e-01)(1.090000e+00,7.397720e-01)(1.100000e+00,7.396000e-01)(1.110000e+00,7.392560e-01)(1.120000e+00,7.390841e-01)(1.130000e+00,7.390841e-01)(1.140000e+00,7.390841e-01)(1.150000e+00,7.390841e-01)(1.160000e+00,7.392560e-01)(1.170000e+00,7.392560e-01)(1.180000e+00,7.396000e-01)(1.190000e+00,7.399440e-01)(1.200000e+00,7.402882e-01)(1.210000e+00,7.406324e-01)(1.220000e+00,7.411488e-01)(1.230000e+00,7.418377e-01)(1.240000e+00,7.425269e-01)(1.250000e+00,7.432164e-01)(1.260000e+00,7.440788e-01)(1.270000e+00,7.449416e-01)(1.280000e+00,7.459777e-01)(1.290000e+00,7.471874e-01)(1.300000e+00,7.483980e-01)(1.310000e+00,7.496096e-01)(1.320000e+00,7.509956e-01)(1.330000e+00,7.525563e-01)(1.340000e+00,7.541186e-01)(1.350000e+00,7.558564e-01)(1.360000e+00,7.575962e-01)(1.370000e+00,7.595123e-01)(1.380000e+00,7.616053e-01)(1.390000e+00,7.637012e-01)(1.400000e+00,7.661501e-01)(1.410000e+00,7.686029e-01)(1.420000e+00,7.710596e-01)(1.430000e+00,7.738721e-01)(1.440000e+00,7.766897e-01)(1.450000e+00,7.796890e-01)(1.460000e+00,7.830480e-01)(1.470000e+00,7.864142e-01)(1.480000e+00,7.899654e-01)(1.490000e+00,7.937028e-01)(1.500000e+00,7.976276e-01)(1.510000e+00,8.017412e-01)(1.520000e+00,8.062244e-01)(1.530000e+00,8.107202e-01)(1.540000e+00,8.155896e-01)(1.550000e+00,8.208360e-01)(1.560000e+00,8.262810e-01)(1.570000e+00,8.319264e-01)(1.580000e+00,8.379572e-01)(1.590000e+00,8.441934e-01)(1.600000e+00,8.510062e-01)(1.610000e+00,8.580317e-01)(1.620000e+00,8.656442e-01)(1.630000e+00,8.734772e-01)(1.640000e+00,8.819088e-01)(1.650000e+00,8.909472e-01)(1.660000e+00,9.004112e-01)(1.670000e+00,9.104976e-01)(1.680000e+00,9.214080e-01)(1.690000e+00,9.329628e-01)(1.700000e+00,9.453673e-01)(1.710000e+00,9.588326e-01)(1.720000e+00,9.731823e-01)(1.730000e+00,9.888314e-01)(1.740000e+00,1.006009e+00)(1.750000e+00,1.024549e+00)(1.760000e+00,1.045097e+00)(1.770000e+00,1.067916e+00)(1.780000e+00,1.093488e+00)(1.790000e+00,1.122540e+00)(1.800000e+00,1.156055e+00)(1.810000e+00,1.195742e+00)(1.820000e+00,1.244117e+00)(1.830000e+00,1.306449e+00)(1.840000e+00,1.397360e+00)

};\addlegendentry{$\beta=1$}           
\addplot[color=red] coordinates{

(1.000000e-02,3.970853e+00)(2.000000e-02,3.941813e+00)(3.000000e-02,3.912880e+00)(4.000000e-02,3.884053e+00)(5.000000e-02,3.855332e+00)(6.000000e-02,3.826718e+00)(7.000000e-02,3.798211e+00)(8.000000e-02,3.769811e+00)(9.000000e-02,3.741516e+00)(1.000000e-01,3.713329e+00)(1.100000e-01,3.685248e+00)(1.200000e-01,3.657274e+00)(1.300000e-01,3.629787e+00)(1.400000e-01,3.602024e+00)(1.500000e-01,3.574368e+00)(1.600000e-01,3.547196e+00)(1.700000e-01,3.519751e+00)(1.800000e-01,3.492787e+00)(1.900000e-01,3.465555e+00)(2.000000e-01,3.438799e+00)(2.100000e-01,3.412148e+00)(2.200000e-01,3.385600e+00)(2.300000e-01,3.359156e+00)(2.400000e-01,3.332815e+00)(2.500000e-01,3.306579e+00)(2.600000e-01,3.280445e+00)(2.700000e-01,3.254416e+00)(2.800000e-01,3.228490e+00)(2.900000e-01,3.203026e+00)(3.000000e-01,3.177663e+00)(3.100000e-01,3.152045e+00)(3.200000e-01,3.126885e+00)(3.300000e-01,3.101825e+00)(3.400000e-01,3.076867e+00)(3.500000e-01,3.052009e+00)(3.600000e-01,3.027252e+00)(3.700000e-01,3.002942e+00)(3.800000e-01,2.978386e+00)(3.900000e-01,2.954273e+00)(4.000000e-01,2.930259e+00)(4.100000e-01,2.906343e+00)(4.200000e-01,2.882525e+00)(4.300000e-01,2.858805e+00)(4.400000e-01,2.835519e+00)(4.500000e-01,2.811994e+00)(4.600000e-01,2.788900e+00)(4.700000e-01,2.765902e+00)(4.800000e-01,2.742998e+00)(4.900000e-01,2.720520e+00)(5.000000e-01,2.697806e+00)(5.100000e-01,2.675514e+00)(5.200000e-01,2.653315e+00)(5.300000e-01,2.631208e+00)(5.400000e-01,2.609194e+00)(5.500000e-01,2.587594e+00)(5.600000e-01,2.565763e+00)(5.700000e-01,2.544344e+00)(5.800000e-01,2.523015e+00)(5.900000e-01,2.502091e+00)(6.000000e-01,2.480940e+00)(6.100000e-01,2.460192e+00)(6.200000e-01,2.439532e+00)(6.300000e-01,2.419269e+00)(6.400000e-01,2.398781e+00)(6.500000e-01,2.378689e+00)(6.600000e-01,2.358682e+00)(6.700000e-01,2.338758e+00)(6.800000e-01,2.319224e+00)(6.900000e-01,2.299772e+00)(7.000000e-01,2.280402e+00)(7.100000e-01,2.261114e+00)(7.200000e-01,2.242207e+00)(7.300000e-01,2.223379e+00)(7.400000e-01,2.204631e+00)(7.500000e-01,2.185962e+00)(7.600000e-01,2.167667e+00)(7.700000e-01,2.149449e+00)(7.800000e-01,2.131600e+00)(7.900000e-01,2.113534e+00)(8.000000e-01,2.095835e+00)(8.100000e-01,2.078499e+00)(8.200000e-01,2.060947e+00)(8.300000e-01,2.043756e+00)(8.400000e-01,2.026637e+00)(8.500000e-01,2.009873e+00)(8.600000e-01,1.993179e+00)(8.700000e-01,1.976555e+00)(8.800000e-01,1.960280e+00)(8.900000e-01,1.944072e+00)(9.000000e-01,1.927932e+00)(9.100000e-01,1.912136e+00)(9.200000e-01,1.896404e+00)(9.300000e-01,1.880738e+00)(9.400000e-01,1.865410e+00)(9.500000e-01,1.850144e+00)(9.600000e-01,1.834941e+00)(9.700000e-01,1.820071e+00)(9.800000e-01,1.805261e+00)(9.900000e-01,1.790779e+00)(1,1.776356e+00)(1.010000e+00,1.761991e+00)(1.020000e+00,1.747948e+00)(1.030000e+00,1.733962e+00)(1.040000e+00,1.720295e+00)(1.050000e+00,1.706681e+00)(1.060000e+00,1.693382e+00)(1.070000e+00,1.679875e+00)(1.080000e+00,1.666939e+00)(1.090000e+00,1.653796e+00)(1.100000e+00,1.641217e+00)(1.110000e+00,1.628431e+00)(1.120000e+00,1.615949e+00)(1.130000e+00,1.603769e+00)(1.140000e+00,1.591635e+00)(1.150000e+00,1.579546e+00)(1.160000e+00,1.567754e+00)(1.170000e+00,1.556007e+00)(1.180000e+00,1.544552e+00)(1.190000e+00,1.533387e+00)(1.200000e+00,1.522262e+00)(1.210000e+00,1.511178e+00)(1.220000e+00,1.500380e+00)(1.230000e+00,1.489620e+00)(1.240000e+00,1.479142e+00)(1.250000e+00,1.468944e+00)(1.260000e+00,1.458539e+00)(1.270000e+00,1.448653e+00)(1.280000e+00,1.438800e+00)(1.290000e+00,1.429220e+00)(1.300000e+00,1.419672e+00)(1.310000e+00,1.410394e+00)(1.320000e+00,1.401146e+00)(1.330000e+00,1.392164e+00)(1.340000e+00,1.383446e+00)(1.350000e+00,1.374756e+00)(1.360000e+00,1.366327e+00)(1.370000e+00,1.358157e+00)(1.380000e+00,1.350012e+00)(1.390000e+00,1.342122e+00)(1.400000e+00,1.334487e+00)(1.410000e+00,1.326874e+00)(1.420000e+00,1.319741e+00)(1.430000e+00,1.312628e+00)(1.440000e+00,1.305535e+00)(1.450000e+00,1.298916e+00)(1.460000e+00,1.292314e+00)(1.470000e+00,1.286183e+00)(1.480000e+00,1.280066e+00)(1.490000e+00,1.274189e+00)(1.500000e+00,1.268552e+00)(1.510000e+00,1.263151e+00)(1.520000e+00,1.257987e+00)(1.530000e+00,1.253280e+00)(1.540000e+00,1.248583e+00)(1.550000e+00,1.244117e+00)(1.560000e+00,1.240105e+00)(1.570000e+00,1.236322e+00)(1.580000e+00,1.232988e+00)(1.590000e+00,1.229659e+00)(1.600000e+00,1.226999e+00)(1.610000e+00,1.224342e+00)(1.620000e+00,1.222351e+00)(1.630000e+00,1.220583e+00)(1.640000e+00,1.219478e+00)(1.650000e+00,1.218595e+00)(1.660000e+00,1.218154e+00)(1.670000e+00,1.218374e+00)(1.680000e+00,1.219258e+00)(1.690000e+00,1.220583e+00)(1.700000e+00,1.222572e+00)(1.710000e+00,1.225449e+00)(1.720000e+00,1.229216e+00)(1.730000e+00,1.233877e+00)(1.740000e+00,1.239660e+00)(1.750000e+00,1.246572e+00)(1.760000e+00,1.255072e+00)(1.770000e+00,1.265400e+00)(1.780000e+00,1.277578e+00)(1.790000e+00,1.292769e+00)(1.800000e+00,1.311025e+00)(1.810000e+00,1.334256e+00)(1.820000e+00,1.363990e+00)(1.830000e+00,1.404699e+00)(1.840000e+00,1.466763e+00)

};\addlegendentry{$\beta=2$}     

\addplot[color=red,mark=diamond,mark options={solid},mark size=2pt,only marks]  plot coordinates{
(1.000000e-01,3.731345e+00)(2.000000e-01,3.454522e+00)(3.000000e-01,3.193624e+00)(4.000000e-01,2.947952e+00)(5.000000e-01,2.715224e+00)(6.000000e-01,2.496989e+00)(7.000000e-01,2.297213e+00)(8.000000e-01,2.104236e+00)(9.000000e-01,1.948113e+00)(1,1.794368e+00)(1.100000e+00,1.657849e+00)(1.200000e+00,1.533925e+00)(1.300000e+00,1.441280e+00)(1.400000e+00,1.348708e+00)(1.500000e+00,1.296007e+00)(1.600000e+00,1.258943e+00)(1.700000e+00,1.246101e+00)
};
\addplot[color=blue,mark=diamond,mark options={solid},mark size=2pt,only marks]  plot coordinates{
(1.000000e-01,9.623300e-01)(2.000000e-01,9.267071e-01)(3.000000e-01,8.948144e-01)(4.000000e-01,8.628190e-01)(5.000000e-01,8.339286e-01)(6.000000e-01,8.102829e-01)(7.000000e-01,7.929244e-01)(8.000000e-01,7.726942e-01)(9.000000e-01,7.609174e-01)(1,7.473155e-01)(1.100000e+00,7.401565e-01)(1.200000e+00,7.470459e-01)(1.300000e+00,7.569682e-01)(1.400000e+00,7.759593e-01)(1.500000e+00,8.170698e-01)(1.600000e+00,8.718171e-01)(1.700000e+00,9.747565e-01)
};
\addplot[color=black,mark=diamond,mark options={solid},mark size=2pt,only marks]  plot coordinates{
(1.000000e-01,2.537008e-01)(2.000000e-01,2.606083e-01)(3.000000e-01,2.688749e-01)(4.000000e-01,2.776354e-01)(5.000000e-01,2.879298e-01)(6.000000e-01,3.027557e-01)(7.000000e-01,3.177375e-01)(8.000000e-01,3.347237e-01)(9.000000e-01,3.532709e-01)(1,3.774833e-01)(1.100000e+00,4.086149e-01)(1.200000e+00,4.415180e-01)(1.300000e+00,4.810070e-01)(1.400000e+00,5.407476e-01)(1.500000e+00,6.107302e-01)(1.600000e+00,7.018336e-01)(1.700000e+00,8.282322e-01)};

\addplot[color=red,mark =none,dotted] coordinates{(0,4)(1,4)(2,4)};
\addplot[color=brown,mark =none,dotted] coordinates{(1.8500,0)(1.8500,5)(1.8500,10)};
\addplot[color=blue,mark =none,dotted] coordinates{(0,1)(1,1)(2,1)};
\addplot[color=black,mark =none,dotted] coordinates{(0,0.25)(1,0.25)(2,0.25)};
\node at (axis cs:1.85,0.15)(source1){${\delta_\star}$}; 
\end{axis}
\end{tikzpicture}}
\subfigure[Cos similarity]{
\begin{tikzpicture}[scale=0.9,font=\fontsize{10}{10}\selectfont]
   \tikzstyle{every axis y label}+=[yshift=0pt]
   \tikzstyle{every axis x label}+=[yshift=5pt]
   \tikzstyle{every axis legend}+=[cells={anchor=west},fill=white,
        at={(0.02,0.98)}, anchor=north west, font=\fontsize{10}{10}\selectfont]
   \begin{axis}[
      xmin=0,
      ymin=0,
      xmax=2,
      ymax=1,
xlabel={$\delta$},
ylabel={Cos similarity }	]		
\addplot[color=black] coordinates{(1.000000e-02,3.847282e-02)(2.000000e-02,5.449654e-02)(3.000000e-02,6.566695e-02)(4.000000e-02,7.629386e-02)(5.000000e-02,8.475217e-02)(6.000000e-02,9.249480e-02)(7.000000e-02,1.004640e-01)(8.000000e-02,1.071689e-01)(9.000000e-02,1.135233e-01)(1.000000e-01,1.195832e-01)(1.100000e-01,1.253917e-01)(1.200000e-01,1.303900e-01)(1.300000e-01,1.358139e-01)(1.400000e-01,1.410671e-01)(1.500000e-01,1.456383e-01)(1.600000e-01,1.501071e-01)(1.700000e-01,1.549808e-01)(1.800000e-01,1.592568e-01)(1.900000e-01,1.634554e-01)(2.000000e-01,1.675826e-01)(2.100000e-01,1.716439e-01)(2.200000e-01,1.752061e-01)(2.300000e-01,1.791594e-01)(2.400000e-01,1.826399e-01)(2.500000e-01,1.864991e-01)(2.600000e-01,1.899085e-01)(2.700000e-01,1.932894e-01)(2.800000e-01,1.966436e-01)(2.900000e-01,1.999727e-01)(3.000000e-01,2.032783e-01)(3.100000e-01,2.065617e-01)(3.200000e-01,2.094602e-01)(3.300000e-01,2.127090e-01)(3.400000e-01,2.155859e-01)(3.500000e-01,2.184555e-01)(3.600000e-01,2.213185e-01)(3.700000e-01,2.241752e-01)(3.800000e-01,2.270264e-01)(3.900000e-01,2.298724e-01)(4.000000e-01,2.327139e-01)(4.100000e-01,2.352287e-01)(4.200000e-01,2.380663e-01)(4.300000e-01,2.405856e-01)(4.400000e-01,2.431092e-01)(4.500000e-01,2.456374e-01)(4.600000e-01,2.481703e-01)(4.700000e-01,2.504061e-01)(4.800000e-01,2.529522e-01)(4.900000e-01,2.552074e-01)(5.000000e-01,2.577672e-01)(5.100000e-01,2.600421e-01)(5.200000e-01,2.623285e-01)(5.300000e-01,2.646263e-01)(5.400000e-01,2.666528e-01)(5.500000e-01,2.689761e-01)(5.600000e-01,2.710332e-01)(5.700000e-01,2.731066e-01)(5.800000e-01,2.754699e-01)(5.900000e-01,2.775739e-01)(6.000000e-01,2.794253e-01)(6.100000e-01,2.815644e-01)(6.200000e-01,2.834550e-01)(6.300000e-01,2.856293e-01)(6.400000e-01,2.875590e-01)(6.500000e-01,2.895093e-01)(6.600000e-01,2.914803e-01)(6.700000e-01,2.932159e-01)(6.800000e-01,2.952300e-01)(6.900000e-01,2.970123e-01)(7.000000e-01,2.988187e-01)(7.100000e-01,3.006494e-01)(7.200000e-01,3.025043e-01)(7.300000e-01,3.043836e-01)(7.400000e-01,3.060426e-01)(7.500000e-01,3.079724e-01)(7.600000e-01,3.096850e-01)(7.700000e-01,3.111848e-01)(7.800000e-01,3.129542e-01)(7.900000e-01,3.147515e-01)(8.000000e-01,3.163403e-01)(8.100000e-01,3.179599e-01)(8.200000e-01,3.196104e-01)(8.300000e-01,3.210592e-01)(8.400000e-01,3.227732e-01)(8.500000e-01,3.242882e-01)(8.600000e-01,3.258371e-01)(8.700000e-01,3.274203e-01)(8.800000e-01,3.288107e-01)(8.900000e-01,3.304640e-01)(9.000000e-01,3.319273e-01)(9.100000e-01,3.334278e-01)(9.200000e-01,3.347430e-01)(9.300000e-01,3.363201e-01)(9.400000e-01,3.377144e-01)(9.500000e-01,3.391493e-01)(9.600000e-01,3.404059e-01)(9.700000e-01,3.419239e-01)(9.800000e-01,3.432662e-01)(9.900000e-01,3.446525e-01)(1,3.458676e-01)(1.010000e+00,3.471292e-01)(1.020000e+00,3.484379e-01)(1.030000e+00,3.497942e-01)(1.040000e+00,3.511987e-01)(1.050000e+00,3.524404e-01)(1.060000e+00,3.535221e-01)(1.070000e+00,3.548672e-01)(1.080000e+00,3.560551e-01)(1.090000e+00,3.572976e-01)(1.100000e+00,3.583874e-01)(1.110000e+00,3.595347e-01)(1.120000e+00,3.607404e-01)(1.130000e+00,3.620053e-01)(1.140000e+00,3.631251e-01)(1.150000e+00,3.641028e-01)(1.160000e+00,3.653493e-01)(1.170000e+00,3.664570e-01)(1.180000e+00,3.674290e-01)(1.190000e+00,3.684704e-01)(1.200000e+00,3.695826e-01)(1.210000e+00,3.705658e-01)(1.220000e+00,3.716237e-01)(1.230000e+00,3.725577e-01)(1.240000e+00,3.735705e-01)(1.250000e+00,3.746639e-01)(1.260000e+00,3.756410e-01)(1.270000e+00,3.765054e-01)(1.280000e+00,3.774580e-01)(1.290000e+00,3.783037e-01)(1.300000e+00,3.792429e-01)(1.310000e+00,3.800816e-01)(1.320000e+00,3.810194e-01)(1.330000e+00,3.818638e-01)(1.340000e+00,3.826187e-01)(1.350000e+00,3.834827e-01)(1.360000e+00,3.842650e-01)(1.370000e+00,3.851637e-01)(1.380000e+00,3.857960e-01)(1.390000e+00,3.865534e-01)(1.400000e+00,3.874399e-01)(1.410000e+00,3.880759e-01)(1.420000e+00,3.888512e-01)(1.430000e+00,3.893882e-01)(1.440000e+00,3.900759e-01)(1.450000e+00,3.907293e-01)(1.460000e+00,3.913554e-01)(1.470000e+00,3.919614e-01)(1.480000e+00,3.925552e-01)(1.490000e+00,3.931451e-01)(1.500000e+00,3.937399e-01)(1.510000e+00,3.941604e-01)(1.520000e+00,3.946062e-01)(1.530000e+00,3.950880e-01)(1.540000e+00,3.956178e-01)(1.550000e+00,3.960207e-01)(1.560000e+00,3.963113e-01)(1.570000e+00,3.966926e-01)(1.580000e+00,3.969942e-01)(1.590000e+00,3.974217e-01)(1.600000e+00,3.976218e-01)(1.610000e+00,3.978043e-01)(1.620000e+00,3.979941e-01)(1.630000e+00,3.982191e-01)(1.640000e+00,3.983246e-01)(1.650000e+00,3.983460e-01)(1.660000e+00,3.983234e-01)(1.670000e+00,3.983021e-01)(1.680000e+00,3.981487e-01)(1.690000e+00,3.981085e-01)(1.700000e+00,3.978813e-01)(1.710000e+00,3.975495e-01)(1.720000e+00,3.972110e-01)(1.730000e+00,3.966138e-01)(1.740000e+00,3.960862e-01)(1.750000e+00,3.954369e-01)(1.760000e+00,3.945223e-01)(1.770000e+00,3.936339e-01)(1.780000e+00,3.924230e-01)(1.790000e+00,3.910586e-01)(1.800000e+00,3.893985e-01)(1.810000e+00,3.873574e-01)(1.820000e+00,3.847678e-01)(1.830000e+00,3.812229e-01)(1.840000e+00,3.761236e-01)

};\addlegendentry{$\beta=0.5$}   

\addplot[color=blue] coordinates{

(1.000000e-02,6.411698e-02)(2.000000e-02,9.064264e-02)(3.000000e-02,1.109745e-01)(4.000000e-02,1.274346e-01)(5.000000e-02,1.425749e-01)(6.000000e-02,1.556987e-01)(7.000000e-02,1.682825e-01)(8.000000e-02,1.795093e-01)(9.000000e-02,1.900621e-01)(1.000000e-01,2.000476e-01)(1.100000e-01,2.095468e-01)(1.200000e-01,2.186233e-01)(1.300000e-01,2.273277e-01)(1.400000e-01,2.357013e-01)(1.500000e-01,2.434451e-01)(1.600000e-01,2.512654e-01)(1.700000e-01,2.585295e-01)(1.800000e-01,2.658910e-01)(1.900000e-01,2.727528e-01)(2.000000e-01,2.794403e-01)(2.100000e-01,2.859659e-01)(2.200000e-01,2.923404e-01)(2.300000e-01,2.985735e-01)(2.400000e-01,3.044159e-01)(2.500000e-01,3.103974e-01)(2.600000e-01,3.160142e-01)(2.700000e-01,3.217714e-01)(2.800000e-01,3.271861e-01)(2.900000e-01,3.325098e-01)(3.000000e-01,3.377469e-01)(3.100000e-01,3.429012e-01)(3.200000e-01,3.479766e-01)(3.300000e-01,3.527609e-01)(3.400000e-01,3.576922e-01)(3.500000e-01,3.623457e-01)(3.600000e-01,3.671443e-01)(3.700000e-01,3.716772e-01)(3.800000e-01,3.761547e-01)(3.900000e-01,3.805788e-01)(4.000000e-01,3.849514e-01)(4.100000e-01,3.890846e-01)(4.200000e-01,3.933622e-01)(4.300000e-01,3.974089e-01)(4.400000e-01,4.015977e-01)(4.500000e-01,4.055634e-01)(4.600000e-01,4.094916e-01)(4.700000e-01,4.133835e-01)(4.800000e-01,4.170675e-01)(4.900000e-01,4.208921e-01)(5.000000e-01,4.246834e-01)(5.100000e-01,4.282760e-01)(5.200000e-01,4.318411e-01)(5.300000e-01,4.353795e-01)(5.400000e-01,4.388918e-01)(5.500000e-01,4.423788e-01)(5.600000e-01,4.456840e-01)(5.700000e-01,4.491240e-01)(5.800000e-01,4.523869e-01)(5.900000e-01,4.556302e-01)(6.000000e-01,4.588544e-01)(6.100000e-01,4.620600e-01)(6.200000e-01,4.652474e-01)(6.300000e-01,4.682713e-01)(6.400000e-01,4.712808e-01)(6.500000e-01,4.744193e-01)(6.600000e-01,4.773997e-01)(6.700000e-01,4.802266e-01)(6.800000e-01,4.831820e-01)(6.900000e-01,4.861248e-01)(7.000000e-01,4.889189e-01)(7.100000e-01,4.917035e-01)(7.200000e-01,4.944788e-01)(7.300000e-01,4.972450e-01)(7.400000e-01,4.998710e-01)(7.500000e-01,5.026211e-01)(7.600000e-01,5.052336e-01)(7.700000e-01,5.078403e-01)(7.800000e-01,5.103144e-01)(7.900000e-01,5.129111e-01)(8.000000e-01,5.153776e-01)(8.100000e-01,5.179651e-01)(8.200000e-01,5.204248e-01)(8.300000e-01,5.227598e-01)(8.400000e-01,5.252154e-01)(8.500000e-01,5.275485e-01)(8.600000e-01,5.298814e-01)(8.700000e-01,5.322141e-01)(8.800000e-01,5.345468e-01)(8.900000e-01,5.367633e-01)(9.000000e-01,5.389819e-01)(9.100000e-01,5.412026e-01)(9.200000e-01,5.434257e-01)(9.300000e-01,5.456513e-01)(9.400000e-01,5.477672e-01)(9.500000e-01,5.498876e-01)(9.600000e-01,5.519018e-01)(9.700000e-01,5.540321e-01)(9.800000e-01,5.560581e-01)(9.900000e-01,5.580906e-01)(1,5.601297e-01)(1.010000e+00,5.620687e-01)(1.020000e+00,5.640160e-01)(1.030000e+00,5.659718e-01)(1.040000e+00,5.678315e-01)(1.050000e+00,5.697014e-01)(1.060000e+00,5.715817e-01)(1.070000e+00,5.734724e-01)(1.080000e+00,5.752718e-01)(1.090000e+00,5.770833e-01)(1.100000e+00,5.788063e-01)(1.110000e+00,5.806434e-01)(1.120000e+00,5.823937e-01)(1.130000e+00,5.840591e-01)(1.140000e+00,5.857402e-01)(1.150000e+00,5.874372e-01)(1.160000e+00,5.890530e-01)(1.170000e+00,5.907833e-01)(1.180000e+00,5.923377e-01)(1.190000e+00,5.939113e-01)(1.200000e+00,5.955045e-01)(1.210000e+00,5.971176e-01)(1.220000e+00,5.986567e-01)(1.230000e+00,6.001239e-01)(1.240000e+00,6.016141e-01)(1.250000e+00,6.031277e-01)(1.260000e+00,6.045731e-01)(1.270000e+00,6.060438e-01)(1.280000e+00,6.074491e-01)(1.290000e+00,6.087909e-01)(1.300000e+00,6.101615e-01)(1.310000e+00,6.115614e-01)(1.320000e+00,6.129019e-01)(1.330000e+00,6.141849e-01)(1.340000e+00,6.155010e-01)(1.350000e+00,6.167628e-01)(1.360000e+00,6.180601e-01)(1.370000e+00,6.193065e-01)(1.380000e+00,6.205040e-01)(1.390000e+00,6.217415e-01)(1.400000e+00,6.228475e-01)(1.410000e+00,6.239972e-01)(1.420000e+00,6.251918e-01)(1.430000e+00,6.262623e-01)(1.440000e+00,6.273820e-01)(1.450000e+00,6.284678e-01)(1.460000e+00,6.294385e-01)(1.470000e+00,6.304653e-01)(1.480000e+00,6.314667e-01)(1.490000e+00,6.324457e-01)(1.500000e+00,6.334056e-01)(1.510000e+00,6.343499e-01)(1.520000e+00,6.351999e-01)(1.530000e+00,6.361240e-01)(1.540000e+00,6.369623e-01)(1.550000e+00,6.377202e-01)(1.560000e+00,6.384839e-01)(1.570000e+00,6.392587e-01)(1.580000e+00,6.399694e-01)(1.590000e+00,6.407031e-01)(1.600000e+00,6.413063e-01)(1.610000e+00,6.419471e-01)(1.620000e+00,6.424743e-01)(1.630000e+00,6.430566e-01)(1.640000e+00,6.435458e-01)(1.650000e+00,6.439539e-01)(1.660000e+00,6.443734e-01)(1.670000e+00,6.447405e-01)(1.680000e+00,6.449942e-01)(1.690000e+00,6.452329e-01)(1.700000e+00,6.454016e-01)(1.710000e+00,6.454496e-01)(1.720000e+00,6.454871e-01)(1.730000e+00,6.453972e-01)(1.740000e+00,6.451499e-01)(1.750000e+00,6.448810e-01)(1.760000e+00,6.444320e-01)(1.770000e+00,6.438212e-01)(1.780000e+00,6.430214e-01)(1.790000e+00,6.419769e-01)(1.800000e+00,6.406337e-01)(1.810000e+00,6.388455e-01)(1.820000e+00,6.365092e-01)(1.830000e+00,6.333185e-01)(1.840000e+00,6.282190e-01)

};\addlegendentry{$\beta=1$}           
\addplot[color=red] coordinates{
(1.000000e-02,8.593526e-02)(2.000000e-02,1.214129e-01)(3.000000e-02,1.485553e-01)(4.000000e-02,1.713699e-01)(5.000000e-02,1.914106e-01)(6.000000e-02,2.094750e-01)(7.000000e-02,2.260374e-01)(8.000000e-02,2.414073e-01)(9.000000e-02,2.557997e-01)(1.000000e-01,2.693714e-01)(1.100000e-01,2.822408e-01)(1.200000e-01,2.944999e-01)(1.300000e-01,3.060475e-01)(1.400000e-01,3.172993e-01)(1.500000e-01,3.281221e-01)(1.600000e-01,3.384012e-01)(1.700000e-01,3.484883e-01)(1.800000e-01,3.581058e-01)(1.900000e-01,3.675758e-01)(2.000000e-01,3.766327e-01)(2.100000e-01,3.854411e-01)(2.200000e-01,3.940178e-01)(2.300000e-01,4.023775e-01)(2.400000e-01,4.105336e-01)(2.500000e-01,4.184979e-01)(2.600000e-01,4.262813e-01)(2.700000e-01,4.338934e-01)(2.800000e-01,4.413430e-01)(2.900000e-01,4.485279e-01)(3.000000e-01,4.555703e-01)(3.100000e-01,4.625831e-01)(3.200000e-01,4.693577e-01)(3.300000e-01,4.760081e-01)(3.400000e-01,4.825395e-01)(3.500000e-01,4.889565e-01)(3.600000e-01,4.952636e-01)(3.700000e-01,5.013699e-01)(3.800000e-01,5.074710e-01)(3.900000e-01,5.133820e-01)(4.000000e-01,5.192010e-01)(4.100000e-01,5.249313e-01)(4.200000e-01,5.305755e-01)(4.300000e-01,5.361366e-01)(4.400000e-01,5.415320e-01)(4.500000e-01,5.469355e-01)(4.600000e-01,5.521804e-01)(4.700000e-01,5.573540e-01)(4.800000e-01,5.624581e-01)(4.900000e-01,5.674156e-01)(5.000000e-01,5.723874e-01)(5.100000e-01,5.772183e-01)(5.200000e-01,5.819887e-01)(5.300000e-01,5.867001e-01)(5.400000e-01,5.913540e-01)(5.500000e-01,5.958788e-01)(5.600000e-01,6.004225e-01)(5.700000e-01,6.048415e-01)(5.800000e-01,6.092097e-01)(5.900000e-01,6.134589e-01)(6.000000e-01,6.177297e-01)(6.100000e-01,6.218852e-01)(6.200000e-01,6.259957e-01)(6.300000e-01,6.299959e-01)(6.400000e-01,6.340197e-01)(6.500000e-01,6.379364e-01)(6.600000e-01,6.418128e-01)(6.700000e-01,6.456498e-01)(6.800000e-01,6.493855e-01)(6.900000e-01,6.530845e-01)(7.000000e-01,6.567474e-01)(7.100000e-01,6.603747e-01)(7.200000e-01,6.639074e-01)(7.300000e-01,6.674070e-01)(7.400000e-01,6.708741e-01)(7.500000e-01,6.743090e-01)(7.600000e-01,6.776551e-01)(7.700000e-01,6.809714e-01)(7.800000e-01,6.842021e-01)(7.900000e-01,6.874603e-01)(8.000000e-01,6.906350e-01)(8.100000e-01,6.937283e-01)(8.200000e-01,6.968496e-01)(8.300000e-01,6.998913e-01)(8.400000e-01,7.029079e-01)(8.500000e-01,7.058477e-01)(8.600000e-01,7.087642e-01)(8.700000e-01,7.116575e-01)(8.800000e-01,7.144774e-01)(8.900000e-01,7.172758e-01)(9.000000e-01,7.200528e-01)(9.100000e-01,7.227596e-01)(9.200000e-01,7.254465e-01)(9.300000e-01,7.281138e-01)(9.400000e-01,7.307138e-01)(9.500000e-01,7.332955e-01)(9.600000e-01,7.358590e-01)(9.700000e-01,7.383581e-01)(9.800000e-01,7.408404e-01)(9.900000e-01,7.432602e-01)(1,7.456643e-01)(1.010000e+00,7.480528e-01)(1.020000e+00,7.503815e-01)(1.030000e+00,7.526958e-01)(1.040000e+00,7.549520e-01)(1.050000e+00,7.571949e-01)(1.060000e+00,7.593816e-01)(1.070000e+00,7.615986e-01)(1.080000e+00,7.637181e-01)(1.090000e+00,7.658681e-01)(1.100000e+00,7.679231e-01)(1.110000e+00,7.700089e-01)(1.120000e+00,7.720428e-01)(1.130000e+00,7.740260e-01)(1.140000e+00,7.759998e-01)(1.150000e+00,7.779643e-01)(1.160000e+00,7.798800e-01)(1.170000e+00,7.817874e-01)(1.180000e+00,7.836475e-01)(1.190000e+00,7.854615e-01)(1.200000e+00,7.872686e-01)(1.210000e+00,7.890690e-01)(1.220000e+00,7.908251e-01)(1.230000e+00,7.925752e-01)(1.240000e+00,7.942824e-01)(1.250000e+00,7.959477e-01)(1.260000e+00,7.976450e-01)(1.270000e+00,7.992650e-01)(1.280000e+00,8.008813e-01)(1.290000e+00,8.024583e-01)(1.300000e+00,8.040325e-01)(1.310000e+00,8.055688e-01)(1.320000e+00,8.071029e-01)(1.330000e+00,8.086005e-01)(1.340000e+00,8.100624e-01)(1.350000e+00,8.115237e-01)(1.360000e+00,8.129506e-01)(1.370000e+00,8.143440e-01)(1.380000e+00,8.157383e-01)(1.390000e+00,8.171005e-01)(1.400000e+00,8.184314e-01)(1.410000e+00,8.197647e-01)(1.420000e+00,8.210356e-01)(1.430000e+00,8.223103e-01)(1.440000e+00,8.235890e-01)(1.450000e+00,8.248084e-01)(1.460000e+00,8.260332e-01)(1.470000e+00,8.272010e-01)(1.480000e+00,8.283758e-01)(1.490000e+00,8.295269e-01)(1.500000e+00,8.306554e-01)(1.510000e+00,8.317623e-01)(1.520000e+00,8.328487e-01)(1.530000e+00,8.338856e-01)(1.540000e+00,8.349346e-01)(1.550000e+00,8.359665e-01)(1.560000e+00,8.369531e-01)(1.570000e+00,8.379256e-01)(1.580000e+00,8.388561e-01)(1.590000e+00,8.398048e-01)(1.600000e+00,8.406861e-01)(1.610000e+00,8.415888e-01)(1.620000e+00,8.424287e-01)(1.630000e+00,8.432654e-01)(1.640000e+00,8.440447e-01)(1.650000e+00,8.448260e-01)(1.660000e+00,8.455843e-01)(1.670000e+00,8.462950e-01)(1.680000e+00,8.469624e-01)(1.690000e+00,8.476189e-01)(1.700000e+00,8.482419e-01)(1.710000e+00,8.488103e-01)(1.720000e+00,8.493313e-01)(1.730000e+00,8.498137e-01)(1.740000e+00,8.502406e-01)(1.750000e+00,8.506248e-01)(1.760000e+00,8.509285e-01)(1.770000e+00,8.511455e-01)(1.780000e+00,8.513030e-01)(1.790000e+00,8.513049e-01)(1.800000e+00,8.512051e-01)(1.810000e+00,8.508740e-01)(1.820000e+00,8.502921e-01)(1.830000e+00,8.492559e-01)(1.840000e+00,8.473896e-01)

};\addlegendentry{$\beta=2$}

\addplot[color=red,mark=diamond,mark size=2pt,only marks]  plot coordinates{

(1.000000e-01,2.587725e-01)(2.000000e-01,3.702301e-01)(3.000000e-01,4.499816e-01)(4.000000e-01,5.142291e-01)(5.000000e-01,5.684234e-01)(6.000000e-01,6.143449e-01)(7.000000e-01,6.538009e-01)(8.000000e-01,6.894306e-01)(9.000000e-01,7.170046e-01)(1,7.432619e-01)(1.100000e+00,7.659275e-01)(1.200000e+00,7.862136e-01)(1.300000e+00,8.016989e-01)(1.400000e+00,8.175147e-01)(1.500000e+00,8.287700e-01)(1.600000e+00,8.393746e-01)(1.700000e+00,8.472409e-01)

};

\addplot[color=blue,mark=diamond,mark size=2pt,only marks]  plot coordinates{

(1.000000e-01,1.949313e-01)(2.000000e-01,2.753020e-01)(3.000000e-01,3.325928e-01)(4.000000e-01,3.820272e-01)(5.000000e-01,4.232242e-01)(6.000000e-01,4.576782e-01)(7.000000e-01,4.836859e-01)(8.000000e-01,5.136612e-01)(9.000000e-01,5.363772e-01)(1,5.607377e-01)(1.100000e+00,5.799267e-01)(1.200000e+00,5.941176e-01)(1.300000e+00,6.075841e-01)(1.400000e+00,6.208596e-01)(1.500000e+00,6.303826e-01)(1.600000e+00,6.386306e-01)(1.700000e+00,6.441000e-01)

};
\addplot[color=black,mark=diamond,mark size=2pt,only marks]  plot coordinates{
(1.000000e-01,1.228929e-01)(2.000000e-01,1.636027e-01)(3.000000e-01,2.004447e-01)(4.000000e-01,2.312473e-01)(5.000000e-01,2.562578e-01)(6.000000e-01,2.767030e-01)(7.000000e-01,2.996672e-01)(8.000000e-01,3.144730e-01)(9.000000e-01,3.363687e-01)(1,3.483624e-01)(1.100000e+00,3.563265e-01)(1.200000e+00,3.656091e-01)(1.300000e+00,3.807690e-01)(1.400000e+00,3.833644e-01)(1.500000e+00,3.935468e-01)(1.600000e+00,3.972855e-01)(1.700000e+00,3.969118e-01)
};
\addplot[color=brown,mark =none,dotted] coordinates{(1.8500,0)(1.8500,5)(1.8500,10)};
\node at (axis cs:1.85,0.04)(source1){${\delta_\star}$}; 
\end{axis}
\end{tikzpicture}}

\caption{Performance of H-SVR as a function of $\delta$ when $\sigma=1$, $\epsilon=1$ for different values of $\beta$. The continuous line curves correspond to the asymptotic performance while the points denote finite-sample performance when $p=200$ and $n= \lfloor{\delta p}\rfloor$. The null risks (corresponding to $\hat\bw_H=\boldsymbol{0}$) are also reported by the dotted lines for the different values of $\beta$. }
\label{HMvsdelta}
\end{center}
\end{figure}
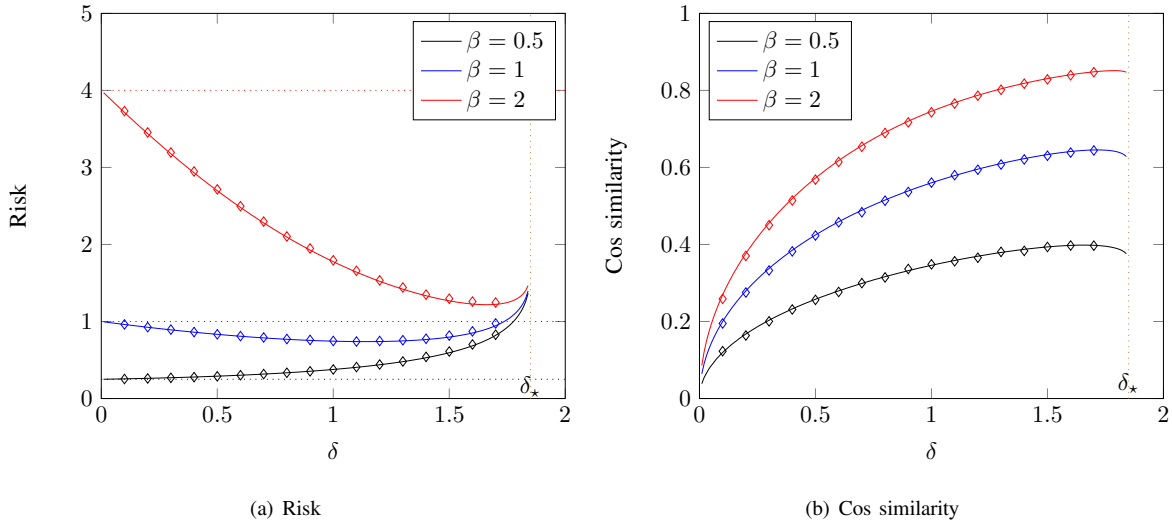
\subsubsection{Test risk as a function of the number of samples}
In a first experiment, we investigate the behavior of the test risk of H-SVR as a function of the number of samples for different values of the signal power $\beta^2=\|\boldsymbol{\beta}_\star\|^2$. Particularly, for each $\beta\in\{0.5,1,2\}$, we fix the noise variance $\sigma^2$ and $\epsilon$ and plot the test risk and cosine similarity over the range $[0,\delta_\star]$ over which the H-SVR is feasible.  \figref{HMvsdelta} represents the theoretical results along with their empirical averages obtained for $p=200$ and $n=\lfloor \delta p\rfloor$. As can be seen, this figure's results validate the accuracy of the theoretical predictions: a perfect match is noted over the whole range of $\delta$ and for all signal power values. We also corroborate our predictions in Remark 3 and Remark 4: the risk tends to $\beta^2$ which is the null estimator's risk when $\delta\to 0$, while it tends to the same limit irrespective of $\beta^2$ as $\delta\to\delta_\star$. 
Away from these limiting cases, we note that for moderate to high signal powers, the test risk presents a non-monotonic behavior with respect to $\delta$ and as such with respect to the number of samples. The minimal risk corresponds to a $\delta$ that becomes the nearest to $\delta_\star$ as the signal power $\beta^2$ increases. However, for low signal powers, the test risk is an increasing function of the number of samples and is always larger than $\beta^2$, which is the null estimator's risk. Such behavior is similar to that of the min-norm least square estimator, which becomes worse than the null estimator when the SNR is less than $1$ \cite{hastie2019surprises}. 

\subsubsection{Impact of choice of $\epsilon$ on the test performance}
\figref{hardvs.epsilon} displays the test risk with respect to $\epsilon$ for fixed signal power and noise variance and oversampling ratios $\delta=1$ and $\delta=1.4$, respectively. As can be noted, the test performance is sensitive to the choice of $\epsilon$. An arbitrary choice of $\epsilon$ may lead to a significant loss in test performance. Indeed, a small $\epsilon$ tolerates less deviation from the plane ${y}=\hat{\bf w}_H^{T}{\bf x}$, which becomes inappropriate when the noise variance increases. On the other hand, a larger $\epsilon$ tolerates more deviation, and as such, tends to give less credit on the information from the training samples. We can also note that the optimal $\epsilon$ increases when more training samples are used. This can be explained by the fact that when using more training samples, it becomes harder to fit them into the insensitivity tube. 
Moreover, as can be seen from this figure, a right choice for the value $\epsilon$ is essential in practice, as arbitrary choices may lead to severe risk performance degradation. Several previous works addressed this question  \cite{cherkassky2004practical,Cherkassky2002}. However, they do not rely on theoretical analysis but instead on cross-validation approaches.
Finally, with reference to {\bf Remark 2}, we plot in \figref{hardvs.delta_opt_eps} the risk of H-SVR with respect to $\delta$ when at each $\delta$, the optimal $\epsilon$ that minimizes the optimal risk is used. The obtained results show that the test risk becomes, in this case, a decreasing function of $\delta$. This is in agreement with the fact that in optimally regularized learning architectures, there is always a gain from using more training samples.

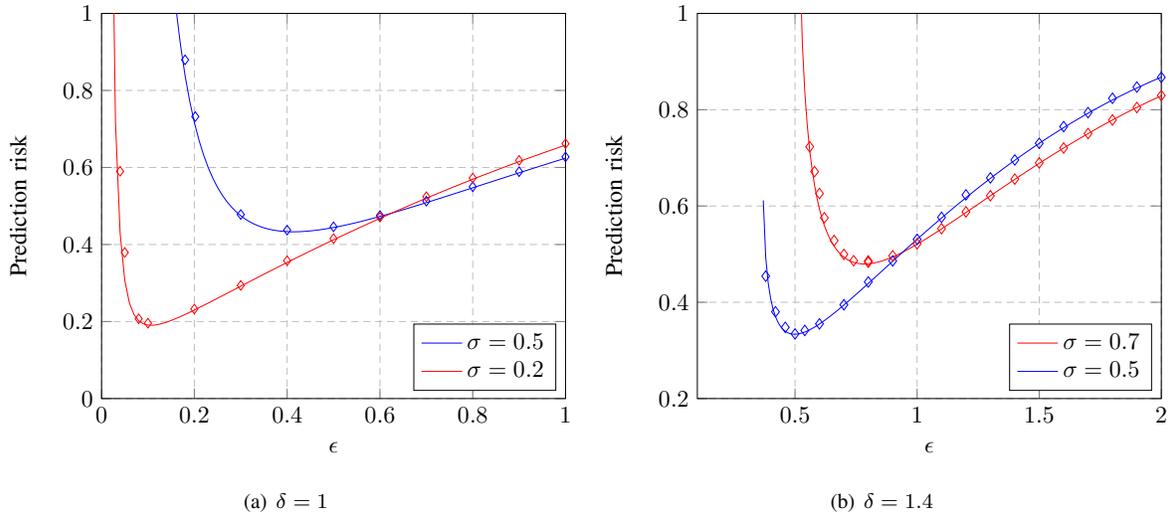
\begin{figure}[]
\begin{center}
\subfigure[$\delta=1$]{
\begin{tikzpicture}[scale=0.9,font=\fontsize{10}{10}\selectfont]
    \renewcommand{\axisdefaulttryminticks}{4}
    \tikzstyle{every major grid}+=[style=densely dashed]
    \tikzstyle{every axis y label}+=[yshift=-20pt]
    \tikzstyle{every axis x label}+=[yshift=5pt]
    \tikzstyle{every axis legend}+=[cells={anchor=west},fill=white,
        at={(0.98,0.02)},anchor=south east,font=\fontsize{10}{10}\selectfont]
    \begin{axis}[
      xmin=0.0,
      ymin=0.0,
      xmax=1,
      ymax=1,
      grid=major,
      scaled ticks=true,
   			xlabel={$\epsilon$},
   			ylabel={Prediction risk}			
      ]
   
   \addplot[color=blue,mark size=1.4pt,mark =none]  plot coordinates{    
(1.300000e-01,1.477440e+00)(1.400000e-01,1.288225e+00)(1.500000e-01,1.138062e+00)(1.600000e-01,1.017072e+00)(1.700000e-01,9.189140e-01)(1.800000e-01,8.385065e-01)(1.900000e-01,7.721137e-01)(2.000000e-01,7.169009e-01)(2.100000e-01,6.707610e-01)(2.200000e-01,6.320250e-01)(2.300000e-01,5.992308e-01)(2.400000e-01,5.716872e-01)(2.500000e-01,5.481922e-01)(2.600000e-01,5.283836e-01)(2.700000e-01,5.115110e-01)(2.800000e-01,4.971660e-01)(2.900000e-01,4.849730e-01)(3.000000e-01,4.745832e-01)(3.100000e-01,4.658062e-01)(3.200000e-01,4.585998e-01)(3.300000e-01,4.523908e-01)(3.400000e-01,4.474272e-01)(3.500000e-01,4.432896e-01)(3.600000e-01,4.399669e-01)(3.700000e-01,4.374500e-01)(3.800000e-01,4.356000e-01)(3.900000e-01,4.342810e-01)(4.000000e-01,4.333589e-01)(4.100000e-01,4.329640e-01)(4.200000e-01,4.330956e-01)(4.300000e-01,4.334906e-01)(4.400000e-01,4.342810e-01)(4.500000e-01,4.353360e-01)(4.600000e-01,4.366566e-01)(4.700000e-01,4.381116e-01)(4.800000e-01,4.399669e-01)(4.900000e-01,4.419590e-01)(5.000000e-01,4.440890e-01)(5.100000e-01,4.464912e-01)(5.200000e-01,4.489000e-01)(5.300000e-01,4.515840e-01)(5.400000e-01,4.542760e-01)(5.500000e-01,4.571112e-01)(5.600000e-01,4.600909e-01)(5.700000e-01,4.632164e-01)(5.800000e-01,4.663524e-01)(5.900000e-01,4.696361e-01)(6.000000e-01,4.729313e-01)(6.100000e-01,4.762380e-01)(6.200000e-01,4.796948e-01)(6.300000e-01,4.831640e-01)(6.400000e-01,4.867853e-01)(6.500000e-01,4.902800e-01)(6.600000e-01,4.939278e-01)(6.700000e-01,4.975892e-01)(6.800000e-01,5.012640e-01)(6.900000e-01,5.049524e-01)(7.000000e-01,5.087969e-01)(7.100000e-01,5.125128e-01)(7.200000e-01,5.163860e-01)(7.300000e-01,5.201294e-01)(7.400000e-01,5.240312e-01)(7.500000e-01,5.278023e-01)(7.600000e-01,5.317326e-01)(7.700000e-01,5.355312e-01)(7.800000e-01,5.394903e-01)(7.900000e-01,5.433164e-01)(8.000000e-01,5.471561e-01)(8.100000e-01,5.511578e-01)(8.200000e-01,5.550250e-01)(8.300000e-01,5.589058e-01)(8.400000e-01,5.628000e-01)(8.500000e-01,5.667078e-01)(8.600000e-01,5.706292e-01)(8.700000e-01,5.744124e-01)(8.800000e-01,5.783603e-01)(8.900000e-01,5.821690e-01)(9.000000e-01,5.861434e-01)(9.100000e-01,5.899776e-01)(9.200000e-01,5.938244e-01)(9.300000e-01,5.976836e-01)(9.400000e-01,6.014003e-01)(9.500000e-01,6.052840e-01)(9.600000e-01,6.091802e-01)(9.700000e-01,6.129324e-01)(9.800000e-01,6.166961e-01)(9.900000e-01,6.204713e-01)(1,6.242580e-01)
};\addlegendentry{$\sigma=0.5$}     
   \addplot[color=red,mark size=1.4pt,mark =none]  plot coordinates{  
 (1.000000e-02,6.278031e+00)(2.000000e-02,1.574774e+00)(3.000000e-02,7.169009e-01)(4.000000e-02,4.295492e-01)(5.000000e-02,3.072485e-01)(6.000000e-02,2.483029e-01)(7.000000e-02,2.179956e-01)(8.000000e-02,2.020503e-01)(9.000000e-02,1.940402e-01)(1.000000e-01,1.907069e-01)(1.100000e-01,1.902704e-01)(1.200000e-01,1.917564e-01)(1.300000e-01,1.945692e-01)(1.400000e-01,1.982921e-01)(1.500000e-01,2.026800e-01)(1.600000e-01,2.075714e-01)(1.700000e-01,2.127977e-01)(1.800000e-01,2.182758e-01)(1.900000e-01,2.240129e-01)(2.000000e-01,2.298244e-01)(2.100000e-01,2.358074e-01)(2.200000e-01,2.418672e-01)(2.300000e-01,2.480040e-01)(2.400000e-01,2.542176e-01)(2.500000e-01,2.604061e-01)(2.600000e-01,2.666690e-01)(2.700000e-01,2.729018e-01)(2.800000e-01,2.792066e-01)(2.900000e-01,2.854765e-01)(3.000000e-01,2.917080e-01)(3.100000e-01,2.980068e-01)(3.200000e-01,3.042626e-01)(3.300000e-01,3.104718e-01)(3.400000e-01,3.166313e-01)(3.500000e-01,3.228512e-01)(3.600000e-01,3.290170e-01)(3.700000e-01,3.351252e-01)(3.800000e-01,3.411728e-01)(3.900000e-01,3.472745e-01)(4.000000e-01,3.533114e-01)(4.100000e-01,3.594003e-01)(4.200000e-01,3.654203e-01)(4.300000e-01,3.713684e-01)(4.400000e-01,3.772416e-01)(4.500000e-01,3.831610e-01)(4.600000e-01,3.890017e-01)(4.700000e-01,3.948866e-01)(4.800000e-01,4.006890e-01)(4.900000e-01,4.064063e-01)(5.000000e-01,4.121640e-01)(5.100000e-01,4.179623e-01)(5.200000e-01,4.235406e-01)(5.300000e-01,4.292870e-01)(5.400000e-01,4.348084e-01)(5.500000e-01,4.404977e-01)(5.600000e-01,4.459568e-01)(5.700000e-01,4.514496e-01)(5.800000e-01,4.569760e-01)(5.900000e-01,4.624000e-01)(6.000000e-01,4.678560e-01)(6.100000e-01,4.732064e-01)(6.200000e-01,4.785872e-01)(6.300000e-01,4.838594e-01)(6.400000e-01,4.891604e-01)(6.500000e-01,4.944902e-01)(6.600000e-01,4.997076e-01)(6.700000e-01,5.048103e-01)(6.800000e-01,5.099388e-01)(6.900000e-01,5.150933e-01)(7.000000e-01,5.202737e-01)(7.100000e-01,5.253350e-01)(7.200000e-01,5.302752e-01)(7.300000e-01,5.352386e-01)(7.400000e-01,5.402250e-01)(7.500000e-01,5.452346e-01)(7.600000e-01,5.501189e-01)(7.700000e-01,5.548760e-01)(7.800000e-01,5.598032e-01)(7.900000e-01,5.646020e-01)(8.000000e-01,5.692703e-01)(8.100000e-01,5.741093e-01)(8.200000e-01,5.788166e-01)(8.300000e-01,5.833904e-01)(8.400000e-01,5.879822e-01)(8.500000e-01,5.925920e-01)(8.600000e-01,5.972198e-01)(8.700000e-01,6.017105e-01)(8.800000e-01,6.062180e-01)(8.900000e-01,6.107422e-01)(9.000000e-01,6.152834e-01)(9.100000e-01,6.196838e-01)(9.200000e-01,6.241000e-01)(9.300000e-01,6.283733e-01)(9.400000e-01,6.326612e-01)(9.500000e-01,6.369636e-01)(9.600000e-01,6.412806e-01)(9.700000e-01,6.454516e-01)(9.800000e-01,6.496360e-01)(9.900000e-01,6.538340e-01)(1,6.580454e-01)
};\addlegendentry{$\sigma=0.2$}    
      
\addplot[color=blue,mark=diamond,mark size=2pt,only marks]  plot coordinates{(1.000000e-01,3.600676e+00)(1.800000e-01,8.794296e-01)(2.020000e-01,7.320852e-01)(3.000000e-01,4.779023e-01)(4.000000e-01,4.373641e-01)(5.000000e-01,4.459377e-01)(6.000000e-01,4.751709e-01)(7.000000e-01,5.124813e-01)(8.000000e-01,5.499734e-01)(9.000000e-01,5.890378e-01)(1,6.274738e-01)

};
\addplot[color=red,mark=diamond,mark size=2pt,only marks]  plot coordinates{(1.000000e-02,1.237045e+01)(2.000000e-02,3.649479e+00)(3.000000e-02,2.238951e+00)(4.000000e-02,5.900826e-01)(5.000000e-02,3.793459e-01)(8.000000e-02,2.072080e-01)(1.000000e-01,1.956808e-01)(2.000000e-01,2.324901e-01)(3.000000e-01,2.938023e-01)(4.000000e-01,3.578950e-01)(5.000000e-01,4.152673e-01)(6.000000e-01,4.710088e-01)(7.000000e-01,5.241065e-01)(8.000000e-01,5.726396e-01)(9.000000e-01,6.183714e-01)(1,6.618415e-01)};

           \end{axis}
  \end{tikzpicture}}
\subfigure[$\delta=1.4$]{
\begin{tikzpicture}[scale=0.9,font=\fontsize{10}{10}\selectfont]
    \renewcommand{\axisdefaulttryminticks}{4}
    \tikzstyle{every major grid}+=[style=densely dashed]
    \tikzstyle{every axis y label}+=[yshift=-20pt]
    \tikzstyle{every axis x label}+=[yshift=5pt]
    \tikzstyle{every axis legend}+=[cells={anchor=west},fill=white,
        at={(0.98,0.02)},anchor=south east,font=\fontsize{10}{10}\selectfont]
    \begin{axis}[
      xmin=0.1,
      ymin=0.2,
      xmax=2,
      ymax=1,
      grid=major,
      scaled ticks=true,
   			xlabel={$\epsilon$},
   			ylabel={Prediction risk}			
      ]
   
   \addplot[color=red,mark size=1.4pt,mark =none]  plot coordinates{    
(5.200000e-01,1.113658e+00)(5.300000e-01,9.372176e-01)(5.400000e-01,8.370420e-01)(5.500000e-01,7.679017e-01)(5.600000e-01,7.160544e-01)(5.700000e-01,6.755196e-01)(5.800000e-01,6.430436e-01)(5.900000e-01,6.162250e-01)(6.000000e-01,5.939785e-01)(6.100000e-01,5.753223e-01)(6.200000e-01,5.595040e-01)(6.300000e-01,5.461210e-01)(6.400000e-01,5.346534e-01)(6.500000e-01,5.249003e-01)(6.600000e-01,5.165297e-01)(6.700000e-01,5.093677e-01)(6.800000e-01,5.032484e-01)(6.900000e-01,4.981536e-01)(7.000000e-01,4.937873e-01)(7.100000e-01,4.901400e-01)(7.200000e-01,4.872040e-01)(7.300000e-01,4.848337e-01)(7.400000e-01,4.828860e-01)(7.500000e-01,4.814972e-01)(7.600000e-01,4.805262e-01)(7.700000e-01,4.798333e-01)(7.800000e-01,4.795563e-01)(7.900000e-01,4.795563e-01)(8.000000e-01,4.799718e-01)(8.100000e-01,4.805262e-01)(8.200000e-01,4.812197e-01)(8.300000e-01,4.823303e-01)(8.400000e-01,4.835812e-01)(8.500000e-01,4.849730e-01)(8.600000e-01,4.865063e-01)(8.700000e-01,4.883214e-01)(8.800000e-01,4.901400e-01)(8.900000e-01,4.921023e-01)(9.000000e-01,4.942090e-01)(9.100000e-01,4.964612e-01)(9.200000e-01,4.988597e-01)(9.300000e-01,5.012640e-01)(9.400000e-01,5.038160e-01)(9.500000e-01,5.065169e-01)(9.600000e-01,5.092250e-01)(9.700000e-01,5.119402e-01)(9.800000e-01,5.148063e-01)(9.900000e-01,5.176803e-01)(1,5.205623e-01)(1.010000e+00,5.235970e-01)(1.020000e+00,5.266405e-01)(1.030000e+00,5.296928e-01)(1.040000e+00,5.329000e-01)(1.050000e+00,5.361168e-01)(1.060000e+00,5.391965e-01)(1.070000e+00,5.424323e-01)(1.080000e+00,5.458254e-01)(1.090000e+00,5.490810e-01)(1.100000e+00,5.523462e-01)(1.110000e+00,5.557703e-01)(1.120000e+00,5.590553e-01)(1.130000e+00,5.625000e-01)(1.140000e+00,5.658048e-01)(1.150000e+00,5.692703e-01)(1.160000e+00,5.725949e-01)(1.170000e+00,5.760810e-01)(1.180000e+00,5.795777e-01)(1.190000e+00,5.829323e-01)(1.200000e+00,5.864496e-01)(1.210000e+00,5.898240e-01)(1.220000e+00,5.933621e-01)(1.230000e+00,5.967563e-01)(1.240000e+00,6.003150e-01)(1.250000e+00,6.037290e-01)(1.260000e+00,6.073085e-01)(1.270000e+00,6.107422e-01)(1.280000e+00,6.141857e-01)(1.290000e+00,6.176388e-01)(1.300000e+00,6.211016e-01)(1.310000e+00,6.245741e-01)(1.320000e+00,6.280562e-01)(1.330000e+00,6.313892e-01)(1.340000e+00,6.348902e-01)(1.350000e+00,6.382412e-01)(1.360000e+00,6.417612e-01)(1.370000e+00,6.451302e-01)(1.380000e+00,6.485081e-01)(1.390000e+00,6.518948e-01)(1.400000e+00,6.552902e-01)(1.410000e+00,6.586946e-01)(1.420000e+00,6.621077e-01)(1.430000e+00,6.653665e-01)(1.440000e+00,6.687968e-01)(1.450000e+00,6.720720e-01)(1.460000e+00,6.753552e-01)(1.470000e+00,6.786464e-01)(1.480000e+00,6.819456e-01)(1.490000e+00,6.852528e-01)(1.500000e+00,6.885680e-01)(1.510000e+00,6.917249e-01)(1.520000e+00,6.950557e-01)(1.530000e+00,6.982274e-01)(1.540000e+00,7.014063e-01)(1.550000e+00,7.045924e-01)(1.560000e+00,7.077857e-01)(1.570000e+00,7.109862e-01)(1.580000e+00,7.140250e-01)(1.590000e+00,7.172396e-01)(1.600000e+00,7.202917e-01)(1.610000e+00,7.233503e-01)(1.620000e+00,7.264153e-01)(1.630000e+00,7.294868e-01)(1.640000e+00,7.325648e-01)(1.650000e+00,7.354778e-01)(1.660000e+00,7.385684e-01)(1.670000e+00,7.414932e-01)(1.680000e+00,7.444238e-01)(1.690000e+00,7.473603e-01)(1.700000e+00,7.503024e-01)(1.710000e+00,7.532504e-01)(1.720000e+00,7.562042e-01)(1.730000e+00,7.589894e-01)(1.740000e+00,7.619544e-01)(1.750000e+00,7.647503e-01)(1.760000e+00,7.675512e-01)(1.770000e+00,7.703573e-01)(1.780000e+00,7.729926e-01)(1.790000e+00,7.758086e-01)(1.800000e+00,7.786298e-01)(1.810000e+00,7.812792e-01)(1.820000e+00,7.839332e-01)(1.830000e+00,7.865916e-01)(1.840000e+00,7.892546e-01)(1.850000e+00,7.919220e-01)(1.860000e+00,7.945940e-01)(1.870000e+00,7.970918e-01)(1.880000e+00,7.997725e-01)(1.890000e+00,8.022785e-01)(1.900000e+00,8.047884e-01)(1.910000e+00,8.073023e-01)(1.920000e+00,8.098200e-01)(1.930000e+00,8.123417e-01)(1.940000e+00,8.146868e-01)(1.950000e+00,8.172160e-01)(1.960000e+00,8.195681e-01)(1.970000e+00,8.219236e-01)(1.980000e+00,8.242824e-01)(1.990000e+00,8.266446e-01)(2,8.290103e-01)

};\addlegendentry{$\sigma=0.7$}     
   
   \addplot[color=blue,mark size=1.4pt,mark =none]  plot coordinates{    
(3.700000e-01,6.116804e-01)(3.800000e-01,4.932253e-01)(3.900000e-01,4.404977e-01)(4.000000e-01,4.080654e-01)(4.100000e-01,3.858894e-01)(4.200000e-01,3.701506e-01)(4.300000e-01,3.586812e-01)(4.400000e-01,3.503456e-01)(4.500000e-01,3.440996e-01)(4.600000e-01,3.397724e-01)(4.700000e-01,3.367481e-01)(4.800000e-01,3.347780e-01)(4.900000e-01,3.338528e-01)(5.000000e-01,3.336218e-01)(5.100000e-01,3.340840e-01)(5.200000e-01,3.350094e-01)(5.300000e-01,3.364000e-01)(5.400000e-01,3.381423e-01)(5.500000e-01,3.403556e-01)(5.600000e-01,3.428103e-01)(5.700000e-01,3.455088e-01)(5.800000e-01,3.484541e-01)(5.900000e-01,3.515304e-01)(6.000000e-01,3.548585e-01)(6.100000e-01,3.583220e-01)(6.200000e-01,3.619226e-01)(6.300000e-01,3.656621e-01)(6.400000e-01,3.695424e-01)(6.500000e-01,3.734432e-01)(6.600000e-01,3.774874e-01)(6.700000e-01,3.815533e-01)(6.800000e-01,3.857652e-01)(6.900000e-01,3.900003e-01)(7.000000e-01,3.942584e-01)(7.100000e-01,3.986660e-01)(7.200000e-01,4.029710e-01)(7.300000e-01,4.074269e-01)(7.400000e-01,4.119072e-01)(7.500000e-01,4.164121e-01)(7.600000e-01,4.208117e-01)(7.700000e-01,4.253648e-01)(7.800000e-01,4.299425e-01)(7.900000e-01,4.345446e-01)(8.000000e-01,4.391713e-01)(8.100000e-01,4.436892e-01)(8.200000e-01,4.483642e-01)(8.300000e-01,4.529290e-01)(8.400000e-01,4.575170e-01)(8.500000e-01,4.621280e-01)(8.600000e-01,4.667622e-01)(8.700000e-01,4.714196e-01)(8.800000e-01,4.759620e-01)(8.900000e-01,4.806649e-01)(9.000000e-01,4.852516e-01)(9.100000e-01,4.898600e-01)(9.200000e-01,4.944902e-01)(9.300000e-01,4.990010e-01)(9.400000e-01,5.036741e-01)(9.500000e-01,5.082264e-01)(9.600000e-01,5.127992e-01)(9.700000e-01,5.172486e-01)(9.800000e-01,5.218618e-01)(9.900000e-01,5.263503e-01)(1,5.308580e-01)(1.010000e+00,5.353849e-01)(1.020000e+00,5.397841e-01)(1.030000e+00,5.442013e-01)(1.040000e+00,5.487846e-01)(1.050000e+00,5.530897e-01)(1.060000e+00,5.575609e-01)(1.070000e+00,5.619002e-01)(1.080000e+00,5.662563e-01)(1.090000e+00,5.706292e-01)(1.100000e+00,5.750189e-01)(1.110000e+00,5.792732e-01)(1.120000e+00,5.836960e-01)(1.130000e+00,5.878289e-01)(1.140000e+00,5.921303e-01)(1.150000e+00,5.964473e-01)(1.160000e+00,6.006250e-01)(1.170000e+00,6.048173e-01)(1.180000e+00,6.090242e-01)(1.190000e+00,6.130890e-01)(1.200000e+00,6.173245e-01)(1.210000e+00,6.214169e-01)(1.220000e+00,6.253646e-01)(1.230000e+00,6.294836e-01)(1.240000e+00,6.336160e-01)(1.250000e+00,6.376022e-01)(1.260000e+00,6.416010e-01)(1.270000e+00,6.454516e-01)(1.280000e+00,6.494748e-01)(1.290000e+00,6.533489e-01)(1.300000e+00,6.572345e-01)(1.310000e+00,6.611316e-01)(1.320000e+00,6.650402e-01)(1.330000e+00,6.687968e-01)(1.340000e+00,6.725640e-01)(1.350000e+00,6.763418e-01)(1.360000e+00,6.801301e-01)(1.370000e+00,6.839290e-01)(1.380000e+00,6.875726e-01)(1.390000e+00,6.912260e-01)(1.400000e+00,6.948890e-01)(1.410000e+00,6.985616e-01)(1.420000e+00,7.020764e-01)(1.430000e+00,7.057680e-01)(1.440000e+00,7.093008e-01)(1.450000e+00,7.128425e-01)(1.460000e+00,7.162237e-01)(1.470000e+00,7.197826e-01)(1.480000e+00,7.231802e-01)(1.490000e+00,7.265858e-01)(1.500000e+00,7.299994e-01)(1.510000e+00,7.334210e-01)(1.520000e+00,7.368506e-01)(1.530000e+00,7.401161e-01)(1.540000e+00,7.433888e-01)(1.550000e+00,7.466688e-01)(1.560000e+00,7.499560e-01)(1.570000e+00,7.530768e-01)(1.580000e+00,7.563781e-01)(1.590000e+00,7.595123e-01)(1.600000e+00,7.626529e-01)(1.610000e+00,7.658000e-01)(1.620000e+00,7.687782e-01)(1.630000e+00,7.719380e-01)(1.640000e+00,7.749281e-01)(1.650000e+00,7.779240e-01)(1.660000e+00,7.809257e-01)(1.670000e+00,7.839332e-01)(1.680000e+00,7.869464e-01)(1.690000e+00,7.897877e-01)(1.700000e+00,7.926341e-01)(1.710000e+00,7.954856e-01)(1.720000e+00,7.983423e-01)(1.730000e+00,8.012040e-01)(1.740000e+00,8.038916e-01)(1.750000e+00,8.067632e-01)(1.760000e+00,8.094601e-01)(1.770000e+00,8.121614e-01)(1.780000e+00,8.148673e-01)(1.790000e+00,8.175776e-01)(1.800000e+00,8.201114e-01)(1.810000e+00,8.228304e-01)(1.820000e+00,8.253723e-01)(1.830000e+00,8.279180e-01)(1.840000e+00,8.304677e-01)(1.850000e+00,8.328388e-01)(1.860000e+00,8.353960e-01)(1.870000e+00,8.379572e-01)(1.880000e+00,8.403389e-01)(1.890000e+00,8.427240e-01)(1.900000e+00,8.451125e-01)(1.910000e+00,8.475044e-01)(1.920000e+00,8.497152e-01)(1.930000e+00,8.521136e-01)(1.940000e+00,8.543305e-01)(1.950000e+00,8.567354e-01)(1.960000e+00,8.589582e-01)(1.970000e+00,8.611840e-01)(1.980000e+00,8.632268e-01)(1.990000e+00,8.654581e-01)(2,8.675060e-01)
};\addlegendentry{$\sigma=0.5$}  

   \addplot[color=blue,mark=diamond,mark size=2.2pt,only marks]  plot coordinates{
(3.800000e-01,4.542481e-01)(4.200000e-01,3.802445e-01)(4.600000e-01,3.480000e-01)(5.000000e-01,3.342733e-01)(5.400000e-01,3.415498e-01)(6.000000e-01,3.552401e-01)(7.000000e-01,3.948042e-01)(8.000000e-01,4.421997e-01)(9.000000e-01,4.859342e-01)(1,5.306043e-01)(1.100000e+00,5.764165e-01)(1.200000e+00,6.230798e-01)(1.300000e+00,6.583114e-01)(1.400000e+00,6.955804e-01)(1.500000e+00,7.302373e-01)(1.600000e+00,7.649460e-01)(1.700000e+00,7.939546e-01)(1.800000e+00,8.237030e-01)(1.900000e+00,8.470199e-01)(2,8.672734e-01)
};

   \addplot[color=red,mark=diamond,mark size=2.2pt,only marks]  plot coordinates{
(5.600000e-01,7.229762e-01)(5.800000e-01,6.714311e-01)(6.000000e-01,6.259221e-01)(6.200000e-01,5.753749e-01)(6.600000e-01,5.282595e-01)(7.400000e-01,4.863410e-01)(8.000000e-01,4.831680e-01)(7.000000e-01,4.994528e-01)(8.000000e-01,4.860077e-01)(9.000000e-01,4.965915e-01)(1,5.214661e-01)(1.100000e+00,5.528550e-01)(1.200000e+00,5.877339e-01)(1.300000e+00,6.213047e-01)(1.400000e+00,6.558343e-01)(1.500000e+00,6.895681e-01)(1.600000e+00,7.204578e-01)(1.700000e+00,7.505584e-01)(1.800000e+00,7.786863e-01)(1.900000e+00,8.045437e-01)(2,8.294955e-01)};

         \end{axis}
  \end{tikzpicture}}
\caption{Performance of H-SVR vs $\epsilon$ when $\beta=1$ for different values of the noise variance $\sigma^2$ and for different values of $\delta$. The continuous line curves correspond to the theoretical predictions while the points denote finite-sample performance when $p=200$ and and $n= \lfloor{\delta p}\rfloor$.}
\label{hardvs.epsilon}
\end{center}
\end{figure}

\begin{figure}[]
\begin{center}
\begin{tikzpicture}[scale=1,font=\fontsize{10}{10}\selectfont]
   \tikzstyle{every axis y label}+=[yshift=0pt]
   \tikzstyle{every axis x label}+=[yshift=5pt]
   \tikzstyle{every axis legend}+=[cells={anchor=west},fill=white,
        at={(0.98,0.98)}, anchor=north east, font=\fontsize{10}{10}\selectfont]
   \begin{axis}[
      xmin=0,
      ymin=0.3,
      xmax=5,
      ymax=1.5,
xlabel={$\delta$},
ylabel={Risk }	]

\addplot[color=green] coordinates{
(1.000000e-01,9.692403e-01)(2.000000e-01,9.397364e-01)(3.000000e-01,9.116430e-01)(4.000000e-01,8.847284e-01)(5.000000e-01,8.589582e-01)(6.000000e-01,8.344822e-01)(7.000000e-01,8.109003e-01)(8.000000e-01,7.885440e-01)(9.000000e-01,7.673760e-01)(1,7.470145e-01)(1.100000e+00,7.277796e-01)(1.200000e+00,7.093008e-01)(1.300000e+00,6.918912e-01)(1.400000e+00,6.755196e-01)(1.500000e+00,6.598313e-01)(1.600000e+00,6.451302e-01)(1.700000e+00,6.312302e-01)(1.800000e+00,6.181104e-01)(1.900000e+00,6.059066e-01)(2,5.944410e-01)(2.100000e+00,5.838488e-01)(2.200000e+00,5.741093e-01)(2.300000e+00,5.653536e-01)(2.400000e+00,5.572623e-01)(2.500000e+00,5.502672e-01)(2.600000e+00,5.440538e-01)(2.700000e+00,5.390496e-01)(2.800000e+00,5.350923e-01)(2.900000e+00,5.323162e-01)(3,5.308580e-01)(3.100000e+00,5.311494e-01)(3.200000e+00,5.331920e-01)(3.300000e+00,5.375822e-01)(3.400000e+00,5.449392e-01)(3.500000e+00,5.563668e-01)(3.600000e+00,5.741093e-01)(3.700000e+00,6.029523e-01)(3.748747e+00,6.247322e-01)(3.758747e+00,6.302772e-01)(3.768747e+00,6.363253e-01)(3.778747e+00,6.430436e-01)(3.788747e+00,6.506036e-01)(3.798747e+00,6.590192e-01)(3.808747e+00,6.687968e-01)(3.818747e+00,6.806250e-01)(3.828747e+00,6.952224e-01)(3.838747e+00,7.157160e-01)
};\addlegendentry{$\epsilon=1.5$}

\addplot[color=black] coordinates{
(1.000000e-01,9.635386e-01)(2.000000e-01,9.291032e-01)(3.000000e-01,8.964302e-01)(4.000000e-01,8.656442e-01)(5.000000e-01,8.368590e-01)(6.000000e-01,8.100000e-01)(7.000000e-01,7.849960e-01)(8.000000e-01,7.619544e-01)(9.000000e-01,7.408045e-01)(1,7.216503e-01)(1.100000e+00,7.045924e-01)(1.200000e+00,6.895642e-01)(1.300000e+00,6.768353e-01)(1.400000e+00,6.665090e-01)(1.500000e+00,6.588569e-01)(1.600000e+00,6.543192e-01)(1.700000e+00,6.531872e-01)(1.800000e+00,6.562620e-01)(1.900000e+00,6.645510e-01)(2,6.801301e-01)(2.100000e+00,7.061041e-01)(2.200000e+00,7.490903e-01)(2.300000e+00,8.279180e-01)

};\addlegendentry{$\epsilon=1.2$}

\addplot[color=blue] coordinates{

(1.000000e-02,9.960040e-01)(2.000000e-02,9.920160e-01)(3.000000e-02,9.880360e-01)(4.000000e-02,9.842624e-01)(5.000000e-02,9.802980e-01)(6.000000e-02,9.765392e-01)(7.000000e-02,9.725904e-01)(8.000000e-02,9.688465e-01)(9.000000e-02,9.651098e-01)(1.000000e-01,9.613803e-01)(1.100000e-01,9.576580e-01)(1.200000e-01,9.539429e-01)(1.300000e-01,9.502350e-01)(1.400000e-01,9.465344e-01)(1.500000e-01,9.430352e-01)(1.600000e-01,9.393486e-01)(1.700000e-01,9.358628e-01)(1.800000e-01,9.321902e-01)(1.900000e-01,9.287177e-01)(2.000000e-01,9.252516e-01)(2.100000e-01,9.217920e-01)(2.200000e-01,9.183389e-01)(2.300000e-01,9.148923e-01)(2.400000e-01,9.116430e-01)(2.500000e-01,9.082090e-01)(2.600000e-01,9.049717e-01)(2.700000e-01,9.015502e-01)(2.800000e-01,8.983248e-01)(2.900000e-01,8.951052e-01)(3.000000e-01,8.918914e-01)(3.100000e-01,8.886833e-01)(3.200000e-01,8.854810e-01)(3.300000e-01,8.824724e-01)(3.400000e-01,8.792813e-01)(3.500000e-01,8.762832e-01)(3.600000e-01,8.731034e-01)(3.700000e-01,8.701158e-01)(3.800000e-01,8.671334e-01)(3.900000e-01,8.641562e-01)(4.000000e-01,8.611840e-01)(4.100000e-01,8.584023e-01)(4.200000e-01,8.554400e-01)(4.300000e-01,8.526676e-01)(4.400000e-01,8.497152e-01)(4.500000e-01,8.469521e-01)(4.600000e-01,8.441934e-01)(4.700000e-01,8.414393e-01)(4.800000e-01,8.388728e-01)(4.900000e-01,8.361274e-01)(5.000000e-01,8.333864e-01)(5.100000e-01,8.308323e-01)(5.200000e-01,8.282820e-01)(5.300000e-01,8.257357e-01)(5.400000e-01,8.231933e-01)(5.500000e-01,8.206548e-01)(5.600000e-01,8.183012e-01)(5.700000e-01,8.157702e-01)(5.800000e-01,8.134236e-01)(5.900000e-01,8.110804e-01)(6.000000e-01,8.087405e-01)(6.100000e-01,8.064040e-01)(6.200000e-01,8.040709e-01)(6.300000e-01,8.019203e-01)(6.400000e-01,7.997725e-01)(6.500000e-01,7.974490e-01)(6.600000e-01,7.953072e-01)(6.700000e-01,7.933465e-01)(6.800000e-01,7.912103e-01)(6.900000e-01,7.890769e-01)(7.000000e-01,7.871238e-01)(7.100000e-01,7.851732e-01)(7.200000e-01,7.832250e-01)(7.300000e-01,7.812792e-01)(7.400000e-01,7.795124e-01)(7.500000e-01,7.775712e-01)(7.600000e-01,7.758086e-01)(7.700000e-01,7.740480e-01)(7.800000e-01,7.724652e-01)(7.900000e-01,7.707084e-01)(8.000000e-01,7.691290e-01)(8.100000e-01,7.673760e-01)(8.200000e-01,7.658000e-01)(8.300000e-01,7.644005e-01)(8.400000e-01,7.628276e-01)(8.500000e-01,7.614308e-01)(8.600000e-01,7.600352e-01)(8.700000e-01,7.586410e-01)(8.800000e-01,7.572480e-01)(8.900000e-01,7.560303e-01)(9.000000e-01,7.548134e-01)(9.100000e-01,7.535976e-01)(9.200000e-01,7.523828e-01)(9.300000e-01,7.511689e-01)(9.400000e-01,7.501292e-01)(9.500000e-01,7.490903e-01)(9.600000e-01,7.482250e-01)(9.700000e-01,7.471874e-01)(9.800000e-01,7.463232e-01)(9.900000e-01,7.454596e-01)(1,7.445964e-01)(1.010000e+00,7.439063e-01)(1.020000e+00,7.432164e-01)(1.030000e+00,7.425269e-01)(1.040000e+00,7.420100e-01)(1.050000e+00,7.414932e-01)(1.060000e+00,7.409766e-01)(1.070000e+00,7.404603e-01)(1.080000e+00,7.401161e-01)(1.090000e+00,7.397720e-01)(1.100000e+00,7.396000e-01)(1.110000e+00,7.392560e-01)(1.120000e+00,7.390841e-01)(1.130000e+00,7.390841e-01)(1.140000e+00,7.390841e-01)(1.150000e+00,7.390841e-01)(1.160000e+00,7.392560e-01)(1.170000e+00,7.392560e-01)(1.180000e+00,7.396000e-01)(1.190000e+00,7.399440e-01)(1.200000e+00,7.402882e-01)(1.210000e+00,7.406324e-01)(1.220000e+00,7.411488e-01)(1.230000e+00,7.418377e-01)(1.240000e+00,7.425269e-01)(1.250000e+00,7.432164e-01)(1.260000e+00,7.440788e-01)(1.270000e+00,7.449416e-01)(1.280000e+00,7.459777e-01)(1.290000e+00,7.471874e-01)(1.300000e+00,7.483980e-01)(1.310000e+00,7.496096e-01)(1.320000e+00,7.509956e-01)(1.330000e+00,7.525563e-01)(1.340000e+00,7.541186e-01)(1.350000e+00,7.558564e-01)(1.360000e+00,7.575962e-01)(1.370000e+00,7.595123e-01)(1.380000e+00,7.616053e-01)(1.390000e+00,7.637012e-01)(1.400000e+00,7.661501e-01)(1.410000e+00,7.686029e-01)(1.420000e+00,7.710596e-01)(1.430000e+00,7.738721e-01)(1.440000e+00,7.766897e-01)(1.450000e+00,7.796890e-01)(1.460000e+00,7.830480e-01)(1.470000e+00,7.864142e-01)(1.480000e+00,7.899654e-01)(1.490000e+00,7.937028e-01)(1.500000e+00,7.976276e-01)(1.510000e+00,8.017412e-01)(1.520000e+00,8.062244e-01)(1.530000e+00,8.107202e-01)(1.540000e+00,8.155896e-01)(1.550000e+00,8.208360e-01)(1.560000e+00,8.262810e-01)(1.570000e+00,8.319264e-01)(1.580000e+00,8.379572e-01)(1.590000e+00,8.441934e-01)(1.600000e+00,8.510062e-01)(1.610000e+00,8.580317e-01)(1.620000e+00,8.656442e-01)(1.630000e+00,8.734772e-01)(1.640000e+00,8.819088e-01)(1.650000e+00,8.909472e-01)(1.660000e+00,9.004112e-01)(1.670000e+00,9.104976e-01)(1.680000e+00,9.214080e-01)(1.690000e+00,9.329628e-01)(1.700000e+00,9.453673e-01)(1.710000e+00,9.588326e-01)(1.720000e+00,9.731823e-01)(1.730000e+00,9.888314e-01)(1.740000e+00,1.006009e+00)(1.750000e+00,1.024549e+00)(1.760000e+00,1.045097e+00)(1.770000e+00,1.067916e+00)(1.780000e+00,1.093488e+00)(1.790000e+00,1.122540e+00)(1.800000e+00,1.156055e+00)(1.810000e+00,1.195742e+00)(1.820000e+00,1.244117e+00)

};\addlegendentry{$\epsilon=1$}           
 \addplot[color=red] coordinates{
(1.000000e-01,9.609881e-01)(2.000000e-01,9.250592e-01)(3.000000e-01,8.918914e-01)(4.000000e-01,8.611840e-01)(5.000000e-01,8.330213e-01)(6.000000e-01,8.071226e-01)(7.000000e-01,7.830480e-01)(8.000000e-01,7.609073e-01)(9.000000e-01,7.406324e-01)(1,7.216503e-01)(1.100000e+00,7.040888e-01)(1.200000e+00,6.879044e-01)(1.300000e+00,6.727280e-01)(1.400000e+00,6.586946e-01)(1.500000e+00,6.454516e-01)(1.600000e+00,6.331385e-01)(1.700000e+00,6.215746e-01)(1.800000e+00,6.107422e-01)(1.900000e+00,6.004700e-01)(2,5.908997e-01)(2.100000e+00,5.817113e-01)(2.200000e+00,5.732004e-01)(2.300000e+00,5.650529e-01)(2.400000e+00,5.572623e-01)(2.500000e+00,5.499706e-01)(2.600000e+00,5.430216e-01)(2.700000e+00,5.364098e-01)(2.800000e+00,5.299840e-01)(2.900000e+00,5.240312e-01)(3,5.181120e-01)(3.200000e+00,5.073713e-01)(3.400000e+00,4.973070e-01)(3.600000e+00,4.881817e-01)(3.800000e+00,4.795563e-01)(4,4.716942e-01)(4.200000e+00,4.641697e-01)(4.400000e+00,4.572464e-01)(4.600000e+00,4.507780e-01)(4.800000e+00,4.446222e-01)(5,4.389062e-01)
};\addlegendentry{Optimal $\epsilon$}

\end{axis}
\end{tikzpicture}
\caption{Performance of H-SVR as a function of $\delta$ when $\sigma=1$ and $\beta=1$ for different values of $\epsilon$. All the curves correspond to the asymptotic performance predictions.}
\label{hardvs.delta_opt_eps}
\end{center}
\end{figure}
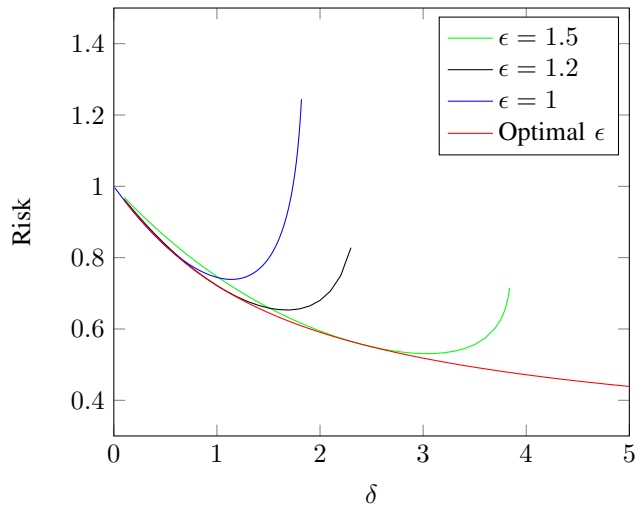

\subsection{S-SVR}

\subsubsection{Impact of the parameters $\epsilon$ and $C$}
In \figref{SMvsepsandC} and \figref{SMvsepsandC3D}, we investigate the effect of the hyper-parameters $C$ and $\epsilon$ on the performance of soft SVR. As shown in these figures, arbitrary choices for the pair $(\epsilon, C)$ may lead to a significant degradation in the test risk performance compared to the optimal performance associated with optimal selection of $(\epsilon, C)$. Thus, our results again emphasize the importance of the appropriate selection of these parameters and suggest the practical relevance of theoretical aided approaches to select these parameters that may complement existing cross-validation techniques.

\begin{figure}[]
\begin{center}
\subfigure[vs. $\epsilon$ for $C=2.4$]{
\begin{tikzpicture}[scale=0.9,font=\fontsize{10}{10}\selectfont]
    \renewcommand{\axisdefaulttryminticks}{4}
    \tikzstyle{every major grid}+=[style=densely dashed]
    \tikzstyle{every axis y label}+=[yshift=-20pt]
    \tikzstyle{every axis x label}+=[yshift=5pt]
    \tikzstyle{every axis legend}+=[cells={anchor=west},fill=white,
        at={(0.98,0.02)},anchor=south east,font=\fontsize{10}{10}\selectfont]
    \begin{axis}[
      xmin=0.0,
      ymin=0.4,
      xmax=2,
      ymax=0.8,
      grid=major,
      scaled ticks=true,
   			xlabel={$\epsilon$},
   			ylabel={Prediction risk}			
      ]
      \addplot[color=blue,mark size=1.4pt,mark =none]  plot coordinates{   
(1.000000e-02,5.187902e-01)(2.000000e-02,5.164148e-01)(3.000000e-02,5.140871e-01)(4.000000e-02,5.118067e-01)(5.000000e-02,5.095735e-01)(6.000000e-02,5.073873e-01)(7.000000e-02,5.052477e-01)(8.000000e-02,5.031547e-01)(9.000000e-02,5.011080e-01)(1.000000e-01,4.991074e-01)(1.100000e-01,4.971526e-01)(1.200000e-01,4.952434e-01)(1.300000e-01,4.933797e-01)(1.400000e-01,4.915612e-01)(1.500000e-01,4.897877e-01)(1.600000e-01,4.880591e-01)(1.700000e-01,4.863750e-01)(1.800000e-01,4.847353e-01)(1.900000e-01,4.831397e-01)(2.000000e-01,4.815882e-01)(2.100000e-01,4.800804e-01)(2.200000e-01,4.786162e-01)(2.300000e-01,4.771954e-01)(2.400000e-01,4.758177e-01)(2.500000e-01,4.744831e-01)(2.600000e-01,4.731912e-01)(2.700000e-01,4.719419e-01)(2.800000e-01,4.707351e-01)(2.900000e-01,4.695705e-01)(3.000000e-01,4.684480e-01)(3.100000e-01,4.673673e-01)(3.200000e-01,4.663283e-01)(3.300000e-01,4.653308e-01)(3.400000e-01,4.643747e-01)(3.500000e-01,4.634598e-01)(3.600000e-01,4.625858e-01)(3.700000e-01,4.617527e-01)(3.800000e-01,4.609603e-01)(3.900000e-01,4.602083e-01)(4.000000e-01,4.594968e-01)(4.100000e-01,4.588254e-01)(4.200000e-01,4.581940e-01)(4.300000e-01,4.576025e-01)(4.400000e-01,4.570508e-01)(4.500000e-01,4.565386e-01)(4.600000e-01,4.560658e-01)(4.700000e-01,4.556323e-01)(4.800000e-01,4.552380e-01)(4.900000e-01,4.548826e-01)(5.000000e-01,4.545661e-01)(5.100000e-01,4.542883e-01)(5.200000e-01,4.540490e-01)(5.300000e-01,4.538482e-01)(5.400000e-01,4.536857e-01)(5.500000e-01,4.535613e-01)(5.600000e-01,4.534749e-01)(5.700000e-01,4.534264e-01)(5.800000e-01,4.534156e-01)(5.900000e-01,4.534425e-01)(6.000000e-01,4.535068e-01)(6.100000e-01,4.536084e-01)(6.200000e-01,4.537473e-01)(6.300000e-01,4.539232e-01)(6.400000e-01,4.541360e-01)(6.500000e-01,4.543856e-01)(6.600000e-01,4.546719e-01)(6.700000e-01,4.549946e-01)(6.800000e-01,4.553538e-01)(6.900000e-01,4.557492e-01)(7.000000e-01,4.561806e-01)(7.100000e-01,4.566480e-01)(7.200000e-01,4.571512e-01)(7.300000e-01,4.576901e-01)(7.400000e-01,4.582645e-01)(7.500000e-01,4.588742e-01)(7.600000e-01,4.595191e-01)(7.700000e-01,4.601990e-01)(7.800000e-01,4.609138e-01)(7.900000e-01,4.616633e-01)(8.000000e-01,4.624473e-01)(8.100000e-01,4.632657e-01)(8.200000e-01,4.641183e-01)(8.300000e-01,4.650049e-01)(8.400000e-01,4.659254e-01)(8.500000e-01,4.668795e-01)(8.600000e-01,4.678670e-01)(8.700000e-01,4.688878e-01)(8.800000e-01,4.699417e-01)(8.900000e-01,4.710284e-01)(9.000000e-01,4.721477e-01)(9.100000e-01,4.732995e-01)(9.200000e-01,4.744835e-01)(9.300000e-01,4.756995e-01)(9.400000e-01,4.769472e-01)(9.500000e-01,4.782264e-01)(9.600000e-01,4.795369e-01)(9.700000e-01,4.808785e-01)(9.800000e-01,4.822508e-01)(9.900000e-01,4.836536e-01)(1,4.850867e-01)(1.010000e+00,4.865498e-01)(1.020000e+00,4.880426e-01)(1.030000e+00,4.895648e-01)(1.040000e+00,4.911161e-01)(1.050000e+00,4.926963e-01)(1.060000e+00,4.943050e-01)(1.070000e+00,4.959420e-01)(1.080000e+00,4.976069e-01)(1.090000e+00,4.992993e-01)(1.100000e+00,5.010191e-01)(1.110000e+00,5.027659e-01)(1.120000e+00,5.045392e-01)(1.130000e+00,5.063388e-01)(1.140000e+00,5.081644e-01)(1.150000e+00,5.100155e-01)(1.160000e+00,5.118919e-01)(1.170000e+00,5.137931e-01)(1.180000e+00,5.157189e-01)(1.190000e+00,5.176687e-01)(1.200000e+00,5.196423e-01)(1.210000e+00,5.216393e-01)(1.220000e+00,5.236592e-01)(1.230000e+00,5.257017e-01)(1.240000e+00,5.277664e-01)(1.250000e+00,5.298529e-01)(1.260000e+00,5.319608e-01)(1.270000e+00,5.340897e-01)(1.280000e+00,5.362391e-01)(1.290000e+00,5.384087e-01)(1.300000e+00,5.405980e-01)(1.310000e+00,5.428067e-01)(1.320000e+00,5.450342e-01)(1.330000e+00,5.472803e-01)(1.340000e+00,5.495443e-01)(1.350000e+00,5.518260e-01)(1.360000e+00,5.541248e-01)(1.370000e+00,5.564404e-01)(1.380000e+00,5.587723e-01)(1.390000e+00,5.611201e-01)(1.400000e+00,5.634833e-01)(1.410000e+00,5.658615e-01)(1.420000e+00,5.682543e-01)(1.430000e+00,5.706612e-01)(1.440000e+00,5.730817e-01)(1.450000e+00,5.755155e-01)(1.460000e+00,5.779622e-01)(1.470000e+00,5.804211e-01)(1.480000e+00,5.828920e-01)(1.490000e+00,5.853744e-01)(1.500000e+00,5.878678e-01)(1.510000e+00,5.903718e-01)(1.520000e+00,5.928861e-01)(1.530000e+00,5.954100e-01)(1.540000e+00,5.979433e-01)(1.550000e+00,6.004854e-01)(1.560000e+00,6.030360e-01)(1.570000e+00,6.055947e-01)(1.580000e+00,6.081609e-01)(1.590000e+00,6.107343e-01)(1.600000e+00,6.133145e-01)(1.610000e+00,6.159011e-01)(1.620000e+00,6.184936e-01)(1.630000e+00,6.210916e-01)(1.640000e+00,6.236947e-01)(1.650000e+00,6.263026e-01)(1.660000e+00,6.289148e-01)(1.670000e+00,6.315309e-01)(1.680000e+00,6.341506e-01)(1.690000e+00,6.367734e-01)(1.700000e+00,6.393989e-01)(1.710000e+00,6.420269e-01)(1.720000e+00,6.446568e-01)(1.730000e+00,6.472884e-01)(1.740000e+00,6.499212e-01)(1.750000e+00,6.525550e-01)(1.760000e+00,6.551893e-01)(1.770000e+00,6.578237e-01)(1.780000e+00,6.604580e-01)(1.790000e+00,6.630918e-01)(1.800000e+00,6.657247e-01)(1.810000e+00,6.683564e-01)(1.820000e+00,6.709866e-01)(1.830000e+00,6.736149e-01)(1.840000e+00,6.762411e-01)(1.850000e+00,6.788647e-01)(1.860000e+00,6.814855e-01)(1.870000e+00,6.841032e-01)(1.880000e+00,6.867174e-01)(1.890000e+00,6.893279e-01)(1.900000e+00,6.919344e-01)(1.910000e+00,6.945365e-01)(1.920000e+00,6.971341e-01)(1.930000e+00,6.997267e-01)(1.940000e+00,7.023142e-01)(1.950000e+00,7.048962e-01)(1.960000e+00,7.074725e-01)(1.970000e+00,7.100428e-01)(1.980000e+00,7.126068e-01)(1.990000e+00,7.151644e-01)(2,7.177153e-01)

};\addlegendentry{$\sigma=1$}

      \addplot[color=red,mark size=1.4pt,mark =none]  plot coordinates{   
(1.000000e-02,7.283857e-01)(2.000000e-02,7.258398e-01)(3.000000e-02,7.233249e-01)(4.000000e-02,7.208409e-01)(5.000000e-02,7.183877e-01)(6.000000e-02,7.159651e-01)(7.000000e-02,7.135730e-01)(8.000000e-02,7.112113e-01)(9.000000e-02,7.088799e-01)(1.000000e-01,7.065786e-01)(1.100000e-01,7.043074e-01)(1.200000e-01,7.020660e-01)(1.300000e-01,6.998544e-01)(1.400000e-01,6.976724e-01)(1.500000e-01,6.955199e-01)(1.600000e-01,6.933969e-01)(1.700000e-01,6.913030e-01)(1.800000e-01,6.892383e-01)(1.900000e-01,6.872026e-01)(2.000000e-01,6.851958e-01)(2.100000e-01,6.832177e-01)(2.200000e-01,6.812683e-01)(2.300000e-01,6.793473e-01)(2.400000e-01,6.774548e-01)(2.500000e-01,6.755905e-01)(2.600000e-01,6.737544e-01)(2.700000e-01,6.719462e-01)(2.800000e-01,6.701660e-01)(2.900000e-01,6.684135e-01)(3.000000e-01,6.666887e-01)(3.100000e-01,6.649914e-01)(3.200000e-01,6.633216e-01)(3.300000e-01,6.616790e-01)(3.400000e-01,6.600636e-01)(3.500000e-01,6.584753e-01)(3.600000e-01,6.569139e-01)(3.700000e-01,6.553793e-01)(3.800000e-01,6.538714e-01)(3.900000e-01,6.523901e-01)(4.000000e-01,6.509353e-01)(4.100000e-01,6.495069e-01)(4.200000e-01,6.481047e-01)(4.300000e-01,6.467287e-01)(4.400000e-01,6.453787e-01)(4.500000e-01,6.440546e-01)(4.600000e-01,6.427563e-01)(4.700000e-01,6.414837e-01)(4.800000e-01,6.402267e-01)(4.900000e-01,6.389943e-01)(5.000000e-01,6.378406e-01)(5.100000e-01,6.366588e-01)(5.200000e-01,6.355017e-01)(5.300000e-01,6.343708e-01)(5.400000e-01,6.332643e-01)(5.500000e-01,6.322358e-01)(5.600000e-01,6.311790e-01)(5.700000e-01,6.301469e-01)(5.800000e-01,6.291395e-01)(5.900000e-01,6.281566e-01)(6.000000e-01,6.271981e-01)(6.100000e-01,6.262640e-01)(6.200000e-01,6.253541e-01)(6.300000e-01,6.245298e-01)(6.400000e-01,6.236779e-01)(6.500000e-01,6.228499e-01)(6.600000e-01,6.220456e-01)(6.700000e-01,6.212649e-01)(6.800000e-01,6.205078e-01)(6.900000e-01,6.197742e-01)(7.000000e-01,6.190639e-01)(7.100000e-01,6.183769e-01)(7.200000e-01,6.177131e-01)(7.300000e-01,6.170724e-01)(7.400000e-01,6.164546e-01)(7.500000e-01,6.158598e-01)(7.600000e-01,6.152878e-01)(7.700000e-01,6.147385e-01)(7.800000e-01,6.142118e-01)(7.900000e-01,6.137077e-01)(8.000000e-01,6.132261e-01)(8.100000e-01,6.127668e-01)(8.200000e-01,6.123298e-01)(8.300000e-01,6.119150e-01)(8.400000e-01,6.115222e-01)(8.500000e-01,6.111516e-01)(8.600000e-01,6.108028e-01)(8.700000e-01,6.104759e-01)(8.800000e-01,6.101707e-01)(8.900000e-01,6.098873e-01)(9.000000e-01,6.096254e-01)(9.100000e-01,6.093850e-01)(9.200000e-01,6.091660e-01)(9.300000e-01,6.089684e-01)(9.400000e-01,6.087920e-01)(9.500000e-01,6.086368e-01)(9.600000e-01,6.085026e-01)(9.700000e-01,6.083895e-01)(9.800000e-01,6.082972e-01)(9.900000e-01,6.082258e-01)(1,6.081751e-01)(1.010000e+00,6.081451e-01)(1.020000e+00,6.081356e-01)(1.030000e+00,6.081466e-01)(1.040000e+00,6.081780e-01)(1.050000e+00,6.082297e-01)(1.060000e+00,6.083017e-01)(1.070000e+00,6.083938e-01)(1.080000e+00,6.085059e-01)(1.090000e+00,6.086380e-01)(1.100000e+00,6.087899e-01)(1.110000e+00,6.089617e-01)(1.120000e+00,6.091531e-01)(1.130000e+00,6.093642e-01)(1.140000e+00,6.096468e-01)(1.150000e+00,6.098968e-01)(1.160000e+00,6.101662e-01)(1.170000e+00,6.104548e-01)(1.180000e+00,6.107625e-01)(1.190000e+00,6.110893e-01)(1.200000e+00,6.114351e-01)(1.210000e+00,6.117997e-01)(1.220000e+00,6.121832e-01)(1.230000e+00,6.125853e-01)(1.240000e+00,6.130060e-01)(1.250000e+00,6.134452e-01)(1.260000e+00,6.139027e-01)(1.270000e+00,6.143786e-01)(1.280000e+00,6.148727e-01)(1.290000e+00,6.153848e-01)(1.300000e+00,6.159149e-01)(1.310000e+00,6.164629e-01)(1.320000e+00,6.170287e-01)(1.330000e+00,6.176122e-01)(1.340000e+00,6.182132e-01)(1.350000e+00,6.188317e-01)(1.360000e+00,6.194675e-01)(1.370000e+00,6.201206e-01)(1.380000e+00,6.207907e-01)(1.390000e+00,6.214779e-01)(1.400000e+00,6.221820e-01)(1.410000e+00,6.229029e-01)(1.420000e+00,6.236405e-01)(1.430000e+00,6.243946e-01)(1.440000e+00,6.251651e-01)(1.450000e+00,6.259519e-01)(1.460000e+00,6.267550e-01)(1.470000e+00,6.275741e-01)(1.480000e+00,6.284091e-01)(1.490000e+00,6.292600e-01)(1.500000e+00,6.301265e-01)(1.510000e+00,6.310086e-01)(1.520000e+00,6.319061e-01)(1.530000e+00,6.328189e-01)(1.540000e+00,6.337469e-01)(1.550000e+00,6.346900e-01)(1.560000e+00,6.356479e-01)(1.570000e+00,6.366205e-01)(1.580000e+00,6.376078e-01)(1.590000e+00,6.386096e-01)(1.600000e+00,6.396257e-01)(1.610000e+00,6.406560e-01)(1.620000e+00,6.417004e-01)(1.630000e+00,6.427586e-01)(1.640000e+00,6.438307e-01)(1.650000e+00,6.449163e-01)(1.660000e+00,6.460154e-01)(1.670000e+00,6.471278e-01)(1.680000e+00,6.482534e-01)(1.690000e+00,6.493920e-01)(1.700000e+00,6.505434e-01)(1.710000e+00,6.517075e-01)(1.720000e+00,6.528842e-01)(1.730000e+00,6.540732e-01)(1.740000e+00,6.552745e-01)(1.750000e+00,6.564878e-01)(1.760000e+00,6.577130e-01)(1.770000e+00,6.589500e-01)(1.780000e+00,6.601985e-01)(1.790000e+00,6.614584e-01)(1.800000e+00,6.627296e-01)(1.810000e+00,6.640118e-01)(1.820000e+00,6.653050e-01)(1.830000e+00,6.666089e-01)(1.840000e+00,6.679241e-01)(1.850000e+00,6.692518e-01)(1.860000e+00,6.705896e-01)(1.870000e+00,6.719374e-01)(1.880000e+00,6.732949e-01)(1.890000e+00,6.746621e-01)(1.900000e+00,6.760387e-01)(1.910000e+00,6.774246e-01)(1.920000e+00,6.788195e-01)(1.930000e+00,6.802234e-01)(1.940000e+00,6.816360e-01)(1.950000e+00,6.830020e-01)(1.960000e+00,6.844315e-01)(1.970000e+00,6.858691e-01)(1.980000e+00,6.873148e-01)(1.990000e+00,6.887684e-01)(2,6.902296e-01)};\addlegendentry{$\sigma=1.4$}  

    \addplot[color=blue,mark=diamond,mark size=2.2pt,only marks]  plot coordinates{
(1.000000e-01,4.997016e-01)(2.000000e-01,4.840016e-01)(3.000000e-01,4.705970e-01)(4.000000e-01,4.582569e-01)(5.000000e-01,4.562918e-01)(6.000000e-01,4.544077e-01)(7.000000e-01,4.556651e-01)(8.000000e-01,4.646659e-01)(9.000000e-01,4.735090e-01)(1,4.864847e-01)(1.100000e+00,5.017807e-01)(1.200000e+00,5.225654e-01)(1.300000e+00,5.404625e-01)(1.400000e+00,5.638353e-01)(1.500000e+00,5.913945e-01)(1.600000e+00,6.178057e-01)(1.700000e+00,6.397251e-01)(1.800000e+00,6.677763e-01)(1.900000e+00,6.919393e-01)(2,7.204215e-01)};
    \addplot[color=red,mark=diamond,mark size=2.2pt,only marks]  plot coordinates{
(1.000000e-01,7.077410e-01)(2.000000e-01,6.865045e-01)(3.000000e-01,6.656245e-01)(4.000000e-01,6.513065e-01)(5.000000e-01,6.371856e-01)(6.000000e-01,6.266560e-01)(7.000000e-01,6.220831e-01)(8.000000e-01,6.149113e-01)(9.000000e-01,6.137084e-01)(1,6.100429e-01)(1.100000e+00,6.098920e-01)(1.200000e+00,6.120483e-01)(1.300000e+00,6.196031e-01)(1.400000e+00,6.208740e-01)(1.500000e+00,6.314583e-01)(1.600000e+00,6.390624e-01)(1.700000e+00,6.497390e-01)(1.800000e+00,6.623783e-01)(1.900000e+00,6.758713e-01)(2,6.942621e-01)
};

           \end{axis}
  \end{tikzpicture}}
  \subfigure[vs. $C$ for $\epsilon=0.6$]{
\begin{tikzpicture}[scale=0.9,font=\fontsize{10}{10}\selectfont]
    \renewcommand{\axisdefaulttryminticks}{4}
    \tikzstyle{every major grid}+=[style=densely dashed]
    \tikzstyle{every axis y label}+=[yshift=-20pt]
    \tikzstyle{every axis x label}+=[yshift=5pt]
    \tikzstyle{every axis legend}+=[cells={anchor=west},fill=white,
        at={(0.98,0.02)},anchor=south east,font=\fontsize{10}{10}\selectfont]
    \begin{axis}[
      xmin=0.0,
      ymin=0.4,
      xmax=25,
      ymax=1.5,
      grid=major,
      scaled ticks=true,
   			xlabel={$C$},
   			ylabel={Prediction risk}			
      ]
  
   \addplot[color=blue,mark size=1.4pt,mark =none]  plot coordinates{   
(1.000000e-01,9.061956e-01)(2.000000e-01,8.287913e-01)(3.000000e-01,7.648182e-01)(4.000000e-01,7.118141e-01)(5.000000e-01,6.677686e-01)(6.000000e-01,6.310563e-01)(7.000000e-01,6.003683e-01)(8.000000e-01,5.746523e-01)(9.000000e-01,5.530611e-01)(1,5.349101e-01)(1.100000e+00,5.196441e-01)(1.200000e+00,5.068103e-01)(1.300000e+00,4.960378e-01)(1.400000e+00,4.870208e-01)(1.500000e+00,4.795059e-01)(1.600000e+00,4.732818e-01)(1.700000e+00,4.681716e-01)(1.800000e+00,4.640262e-01)(1.900000e+00,4.607196e-01)(2,4.581443e-01)(2.100000e+00,4.562089e-01)(2.200000e+00,4.548344e-01)(2.300000e+00,4.539533e-01)(2.400000e+00,4.535068e-01)(2.500000e+00,4.534440e-01)(2.600000e+00,4.537207e-01)(2.700000e+00,4.542979e-01)(2.800000e+00,4.551418e-01)(2.900000e+00,4.562224e-01)(3.000000e+00,4.575135e-01)(3.100000e+00,4.589915e-01)(3.200000e+00,4.606360e-01)(3.300000e+00,4.624286e-01)(3.400000e+00,4.643528e-01)(3.500000e+00,4.663943e-01)(3.600000e+00,4.685398e-01)(3.700000e+00,4.707777e-01)(3.800000e+00,4.730977e-01)(3.900000e+00,4.754902e-01)(4,4.779467e-01)(4.100000e+00,4.804598e-01)(4.200000e+00,4.830224e-01)(4.300000e+00,4.856284e-01)(4.400000e+00,4.882722e-01)(4.500000e+00,4.909485e-01)(4.600000e+00,4.936529e-01)(4.700000e+00,4.963812e-01)(4.800000e+00,4.991295e-01)(4.900000e+00,5.018943e-01)(5,5.046727e-01)(5.100000e+00,5.074616e-01)(5.200000e+00,5.102585e-01)(5.300000e+00,5.130610e-01)(5.400000e+00,5.158670e-01)(5.500000e+00,5.186745e-01)(5.600000e+00,5.214817e-01)(5.700000e+00,5.242870e-01)(5.800000e+00,5.270890e-01)(5.900000e+00,5.298862e-01)(6,5.326774e-01)(6.100000e+00,5.354615e-01)(6.200000e+00,5.382376e-01)(6.300000e+00,5.410046e-01)(6.400000e+00,5.437617e-01)(6.500000e+00,5.465082e-01)(6.600000e+00,5.492433e-01)(6.700000e+00,5.519665e-01)(6.800000e+00,5.546772e-01)(6.900000e+00,5.573748e-01)(7,5.600590e-01)(7.100000e+00,5.627292e-01)(7.200000e+00,5.653852e-01)(7.300000e+00,5.680265e-01)(7.400000e+00,5.706530e-01)(7.500000e+00,5.732643e-01)(7.600000e+00,5.758601e-01)(7.700000e+00,5.784404e-01)(7.800000e+00,5.810050e-01)(7.900000e+00,5.835536e-01)(8,5.860861e-01)(8.100000e+00,5.886025e-01)(8.200000e+00,5.911027e-01)(8.300000e+00,5.935865e-01)(8.400000e+00,5.960540e-01)(8.500000e+00,5.985051e-01)(8.600000e+00,6.009397e-01)(8.700000e+00,6.033579e-01)(8.800000e+00,6.057597e-01)(8.900000e+00,6.081450e-01)(9,6.105139e-01)(9.100000e+00,6.128664e-01)(9.200000e+00,6.152026e-01)(9.300000e+00,6.175225e-01)(9.400000e+00,6.198261e-01)(9.500000e+00,6.221135e-01)(9.600000e+00,6.243848e-01)(9.700000e+00,6.266401e-01)(9.800000e+00,6.288793e-01)(9.900000e+00,6.311027e-01)(10,6.333103e-01)(1.010000e+01,6.355021e-01)(1.020000e+01,6.376783e-01)(1.030000e+01,6.398389e-01)(1.040000e+01,6.419841e-01)(1.050000e+01,6.441140e-01)(1.060000e+01,6.462286e-01)(1.070000e+01,6.483281e-01)(1.080000e+01,6.504125e-01)(1.090000e+01,6.524821e-01)(11,6.545368e-01)(1.110000e+01,6.565768e-01)(1.120000e+01,6.586022e-01)(1.130000e+01,6.606132e-01)(1.140000e+01,6.626097e-01)(1.150000e+01,6.645921e-01)(1.160000e+01,6.665602e-01)(1.170000e+01,6.685144e-01)(1.180000e+01,6.704547e-01)(1.190000e+01,6.723811e-01)(12,6.742939e-01)(1.210000e+01,6.761931e-01)(1.220000e+01,6.780789e-01)(1.230000e+01,6.799513e-01)(1.240000e+01,6.818106e-01)(1.250000e+01,6.836567e-01)(1.260000e+01,6.854898e-01)(1.270000e+01,6.873101e-01)(1.280000e+01,6.891176e-01)(1.290000e+01,6.909124e-01)(13,6.926947e-01)(1.310000e+01,6.944646e-01)(1.320000e+01,6.962222e-01)(1.330000e+01,6.979676e-01)(1.340000e+01,6.997008e-01)(1.350000e+01,7.014221e-01)(1.360000e+01,7.031316e-01)(1.370000e+01,7.048292e-01)(1.380000e+01,7.065152e-01)(1.390000e+01,7.081896e-01)(14,7.098526e-01)(1.410000e+01,7.115042e-01)(1.420000e+01,7.131446e-01)(1.430000e+01,7.147739e-01)(1.440000e+01,7.163921e-01)(1.450000e+01,7.179994e-01)(1.460000e+01,7.195958e-01)(1.470000e+01,7.211815e-01)(1.480000e+01,7.227566e-01)(1.490000e+01,7.243211e-01)(15,7.258752e-01)(1.510000e+01,7.274189e-01)(1.520000e+01,7.289524e-01)(1.530000e+01,7.304757e-01)(1.540000e+01,7.319890e-01)(1.550000e+01,7.334922e-01)(1.560000e+01,7.349856e-01)(1.570000e+01,7.364692e-01)(1.580000e+01,7.379431e-01)(1.590000e+01,7.394074e-01)(16,7.408621e-01)(1.610000e+01,7.423075e-01)(1.620000e+01,7.437434e-01)(1.630000e+01,7.451701e-01)(1.640000e+01,7.465876e-01)(1.650000e+01,7.479960e-01)(1.660000e+01,7.493954e-01)(1.670000e+01,7.507858e-01)(1.680000e+01,7.521674e-01)(1.690000e+01,7.535402e-01)(17,7.549043e-01)(1.710000e+01,7.562598e-01)(1.720000e+01,7.576068e-01)(1.730000e+01,7.589453e-01)(1.740000e+01,7.602753e-01)(1.750000e+01,7.615971e-01)(1.760000e+01,7.629107e-01)(1.770000e+01,7.642160e-01)(1.780000e+01,7.655133e-01)(1.790000e+01,7.668026e-01)(18,7.680839e-01)(1.810000e+01,7.693573e-01)(1.820000e+01,7.706229e-01)(1.830000e+01,7.718808e-01)(1.840000e+01,7.731310e-01)(1.850000e+01,7.743736e-01)(1.860000e+01,7.756087e-01)(1.870000e+01,7.768363e-01)(1.880000e+01,7.780564e-01)(1.890000e+01,7.792693e-01)(19,7.804748e-01)(1.910000e+01,7.816732e-01)(1.920000e+01,7.828643e-01)(1.930000e+01,7.840484e-01)(1.940000e+01,7.852255e-01)(1.950000e+01,7.863956e-01)(1.960000e+01,7.875588e-01)(1.970000e+01,7.887151e-01)(1.980000e+01,7.898647e-01)(1.990000e+01,7.910075e-01)(20,7.921436e-01)(2.010000e+01,7.932731e-01)(2.020000e+01,7.943961e-01)(2.030000e+01,7.955125e-01)(2.040000e+01,7.966225e-01)(2.050000e+01,7.977261e-01)(2.060000e+01,7.988234e-01)(2.070000e+01,7.999143e-01)(2.080000e+01,8.009991e-01)(2.090000e+01,8.020776e-01)(21,8.031500e-01)(2.110000e+01,8.042163e-01)(2.120000e+01,8.052766e-01)(2.130000e+01,8.063309e-01)(2.140000e+01,8.073793e-01)(2.150000e+01,8.084218e-01)(2.160000e+01,8.094584e-01)(2.170000e+01,8.104893e-01)(2.180000e+01,8.115144e-01)(2.190000e+01,8.125339e-01)(22,8.135476e-01)(2.210000e+01,8.145558e-01)(2.220000e+01,8.155585e-01)(2.230000e+01,8.165556e-01)(2.240000e+01,8.175472e-01)(2.250000e+01,8.185335e-01)(2.260000e+01,8.195143e-01)(2.270000e+01,8.204898e-01)(2.280000e+01,8.214601e-01)(2.290000e+01,8.224251e-01)(23,8.233848e-01)(2.310000e+01,8.243394e-01)(2.320000e+01,8.252889e-01)(2.330000e+01,8.262333e-01)(2.340000e+01,8.271727e-01)(2.350000e+01,8.281070e-01)(2.360000e+01,8.290364e-01)(2.370000e+01,8.299608e-01)(2.380000e+01,8.308804e-01)(2.390000e+01,8.317951e-01)(24,8.327050e-01)(2.410000e+01,8.336102e-01)(2.420000e+01,8.345106e-01)(2.430000e+01,8.354063e-01)(2.440000e+01,8.362973e-01)(2.450000e+01,8.371837e-01)(2.460000e+01,8.380655e-01)(2.470000e+01,8.389427e-01)(2.480000e+01,8.398155e-01)(2.490000e+01,8.406837e-01)(25,8.415475e-01)
};\addlegendentry{$\sigma=1$}   
   \addplot[color=red,mark size=1.4pt,mark =none]  plot coordinates{   
(1.000000e-02,9.913657e-01)(1.100000e-01,9.137666e-01)(2.100000e-01,8.502630e-01)(3.100000e-01,7.982679e-01)(4.100000e-01,7.558080e-01)(5.100000e-01,7.211988e-01)(6.100000e-01,6.930801e-01)(7.100000e-01,6.703499e-01)(8.100000e-01,6.520577e-01)(9.100000e-01,6.375769e-01)(1.010000e+00,6.262380e-01)(1.110000e+00,6.175640e-01)(1.210000e+00,6.111600e-01)(1.310000e+00,6.066679e-01)(1.410000e+00,6.038152e-01)(1.510000e+00,6.023662e-01)(1.610000e+00,6.020706e-01)(1.710000e+00,6.028642e-01)(1.810000e+00,6.045525e-01)(1.910000e+00,6.070142e-01)(2.010000e+00,6.101452e-01)(2.110000e+00,6.138560e-01)(2.210000e+00,6.180697e-01)(2.310000e+00,6.227196e-01)(2.410000e+00,6.277262e-01)(2.510000e+00,6.330825e-01)(2.610000e+00,6.387444e-01)(2.710000e+00,6.446303e-01)(2.810000e+00,6.507281e-01)(2.910000e+00,6.570083e-01)(3.010000e+00,6.634446e-01)(3.110000e+00,6.700142e-01)(3.210000e+00,6.766966e-01)(3.310000e+00,6.835289e-01)(3.410000e+00,6.903853e-01)(3.510000e+00,6.973063e-01)(3.610000e+00,7.042794e-01)(3.710000e+00,7.112934e-01)(3.810000e+00,7.183948e-01)(3.910000e+00,7.254619e-01)(4.010000e+00,7.325488e-01)(4.110000e+00,7.396525e-01)(4.210000e+00,7.467656e-01)(4.310000e+00,7.538672e-01)(4.410000e+00,7.609580e-01)(4.510000e+00,7.680391e-01)(4.610000e+00,7.751027e-01)(4.710000e+00,7.821452e-01)(4.810000e+00,7.891633e-01)(4.910000e+00,7.961542e-01)(5.010000e+00,8.031153e-01)(5.110000e+00,8.100443e-01)(5.210000e+00,8.169393e-01)(5.310000e+00,8.237984e-01)(5.410000e+00,8.306201e-01)(5.510000e+00,8.374029e-01)(5.610000e+00,8.441456e-01)(5.710000e+00,8.508472e-01)(5.810000e+00,8.574782e-01)(5.910000e+00,8.640923e-01)(6.010000e+00,8.706626e-01)(6.110000e+00,8.771888e-01)(6.210000e+00,8.836703e-01)(6.310000e+00,8.901067e-01)(6.410000e+00,8.964977e-01)(6.510000e+00,9.028431e-01)(6.610000e+00,9.091425e-01)(6.710000e+00,9.153960e-01)(6.810000e+00,9.216035e-01)(6.910000e+00,9.277648e-01)(7.010000e+00,9.338800e-01)(7.110000e+00,9.399492e-01)(7.210000e+00,9.459725e-01)(7.310000e+00,9.519499e-01)(7.410000e+00,9.578817e-01)(7.510000e+00,9.637679e-01)(7.610000e+00,9.696089e-01)(7.710000e+00,9.754047e-01)(7.810000e+00,9.811558e-01)(7.910000e+00,9.868622e-01)(8.010000e+00,9.925243e-01)(8.110000e+00,9.981424e-01)(8.210000e+00,1.003717e+00)(8.310000e+00,1.009248e+00)(8.410000e+00,1.014736e+00)(8.510000e+00,1.020181e+00)(8.610000e+00,1.025583e+00)(8.710000e+00,1.030944e+00)(8.810000e+00,1.036262e+00)(8.910000e+00,1.041540e+00)(9.010000e+00,1.046776e+00)(9.110000e+00,1.051971e+00)(9.210000e+00,1.057127e+00)(9.310000e+00,1.062242e+00)(9.410000e+00,1.067317e+00)(9.510000e+00,1.072354e+00)(9.610000e+00,1.077351e+00)(9.710000e+00,1.082310e+00)(9.810000e+00,1.087231e+00)(9.910000e+00,1.092114e+00)(1.001000e+01,1.096959e+00)(1.011000e+01,1.101768e+00)(1.021000e+01,1.106540e+00)(1.031000e+01,1.111275e+00)(1.041000e+01,1.115975e+00)(1.051000e+01,1.120639e+00)(1.061000e+01,1.125267e+00)(1.071000e+01,1.129861e+00)(1.081000e+01,1.134421e+00)(1.091000e+01,1.138946e+00)(1.101000e+01,1.143437e+00)(1.111000e+01,1.147894e+00)(1.121000e+01,1.152319e+00)(1.131000e+01,1.156711e+00)(1.141000e+01,1.161070e+00)(1.151000e+01,1.165397e+00)(1.161000e+01,1.169692e+00)(1.171000e+01,1.173956e+00)(1.181000e+01,1.178188e+00)(1.191000e+01,1.182390e+00)(1.201000e+01,1.186561e+00)(1.211000e+01,1.190702e+00)(1.221000e+01,1.194813e+00)(1.231000e+01,1.198894e+00)(1.241000e+01,1.202946e+00)(1.251000e+01,1.206969e+00)(1.261000e+01,1.210963e+00)(1.271000e+01,1.214929e+00)(1.281000e+01,1.218866e+00)(1.291000e+01,1.222776e+00)(1.301000e+01,1.226658e+00)(1.311000e+01,1.230513e+00)(1.321000e+01,1.234341e+00)(1.331000e+01,1.238142e+00)(1.341000e+01,1.241916e+00)(1.351000e+01,1.245665e+00)(1.361000e+01,1.249387e+00)(1.371000e+01,1.253084e+00)(1.381000e+01,1.256755e+00)(1.391000e+01,1.260402e+00)(1.401000e+01,1.264023e+00)(1.411000e+01,1.267619e+00)(1.421000e+01,1.271192e+00)(1.431000e+01,1.274740e+00)(1.441000e+01,1.278264e+00)(1.451000e+01,1.281764e+00)(1.461000e+01,1.285241e+00)(1.471000e+01,1.288694e+00)(1.481000e+01,1.292125e+00)(1.491000e+01,1.295533e+00)(1.501000e+01,1.298918e+00)(1.511000e+01,1.302281e+00)(1.521000e+01,1.305622e+00)(1.531000e+01,1.308941e+00)(1.541000e+01,1.312238e+00)(1.551000e+01,1.315514e+00)(1.561000e+01,1.318768e+00)(1.571000e+01,1.322001e+00)(1.581000e+01,1.325214e+00)(1.591000e+01,1.328406e+00)(1.601000e+01,1.331577e+00)(1.611000e+01,1.334728e+00)(1.621000e+01,1.337859e+00)(1.631000e+01,1.340970e+00)(1.641000e+01,1.344062e+00)(1.651000e+01,1.347134e+00)(1.661000e+01,1.350186e+00)(1.671000e+01,1.353220e+00)(1.681000e+01,1.356234e+00)(1.691000e+01,1.359230e+00)(1.701000e+01,1.362207e+00)(1.711000e+01,1.365165e+00)(1.721000e+01,1.368106e+00)(1.731000e+01,1.371028e+00)(1.741000e+01,1.373932e+00)(1.751000e+01,1.376818e+00)(1.761000e+01,1.379687e+00)(1.771000e+01,1.382539e+00)(1.781000e+01,1.385373e+00)(1.791000e+01,1.388190e+00)(1.801000e+01,1.390990e+00)(1.811000e+01,1.393773e+00)(1.821000e+01,1.396539e+00)(1.831000e+01,1.399289e+00)(1.841000e+01,1.402023e+00)(1.851000e+01,1.404740e+00)(1.861000e+01,1.407441e+00)(1.871000e+01,1.410127e+00)(1.881000e+01,1.412796e+00)(1.891000e+01,1.415450e+00)(1.901000e+01,1.418088e+00)(1.911000e+01,1.420711e+00)(1.921000e+01,1.423319e+00)(1.931000e+01,1.425912e+00)(1.941000e+01,1.428489e+00)(1.951000e+01,1.431052e+00)(1.961000e+01,1.433600e+00)(1.971000e+01,1.436133e+00)(1.981000e+01,1.438652e+00)(1.991000e+01,1.441157e+00)(2.001000e+01,1.443647e+00)(2.011000e+01,1.446123e+00)(2.021000e+01,1.448586e+00)(2.031000e+01,1.451034e+00)(2.041000e+01,1.453469e+00)(2.051000e+01,1.455890e+00)(2.061000e+01,1.458297e+00)(2.071000e+01,1.460691e+00)(2.081000e+01,1.463072e+00)(2.091000e+01,1.465440e+00)(2.101000e+01,1.467794e+00)(2.111000e+01,1.470136e+00)(2.121000e+01,1.472464e+00)(2.131000e+01,1.474780e+00)(2.141000e+01,1.477084e+00)(2.151000e+01,1.479375e+00)(2.161000e+01,1.481653e+00)(2.171000e+01,1.483919e+00)(2.181000e+01,1.486173e+00)(2.191000e+01,1.488414e+00)(2.201000e+01,1.490644e+00)(2.211000e+01,1.492862e+00)(2.221000e+01,1.495068e+00)(2.231000e+01,1.497262e+00)(2.241000e+01,1.499444e+00)(2.251000e+01,1.501615e+00)(2.261000e+01,1.503775e+00)(2.271000e+01,1.505923e+00)(2.281000e+01,1.508060e+00)(2.291000e+01,1.510185e+00)(2.301000e+01,1.512300e+00)(2.311000e+01,1.514403e+00)(2.321000e+01,1.516496e+00)(2.331000e+01,1.518578e+00)(2.341000e+01,1.520649e+00)(2.351000e+01,1.522709e+00)(2.361000e+01,1.524759e+00)(2.371000e+01,1.526798e+00)(2.381000e+01,1.528826e+00)(2.391000e+01,1.530845e+00)(2.401000e+01,1.532853e+00)(2.411000e+01,1.534851e+00)(2.421000e+01,1.536839e+00)(2.431000e+01,1.538817e+00)(2.441000e+01,1.540785e+00)(2.451000e+01,1.542743e+00)(2.461000e+01,1.544691e+00)(2.471000e+01,1.546629e+00)(2.481000e+01,1.548558e+00)(2.491000e+01,1.550477e+00)
};\addlegendentry{$\sigma=1.4$} 

             \addplot[color=blue,mark=diamond,mark size=2.2pt,only marks]  plot coordinates{
(1.000000e-01,9.063455e-01)(2.000000e-01,8.291647e-01)(3.000000e-01,7.650076e-01)(4.000000e-01,7.120849e-01)(5.000000e-01,6.692648e-01)(6.000000e-01,6.316769e-01)(7.000000e-01,6.006729e-01)(8.000000e-01,5.756834e-01)(9.000000e-01,5.537504e-01)(1,5.355507e-01)(1.200000e+00,5.070228e-01)(1.600000e+00,4.712811e-01)(2,4.568498e-01)(2.400000e+00,4.536262e-01)(3,4.589592e-01)(3.500000e+00,4.671010e-01)(4,4.813954e-01)(6,5.327586e-01)(8,5.913972e-01)(10,6.339357e-01)(12,6.771024e-01)(14,7.149580e-01)(16,7.383642e-01)(18,7.745812e-01)(20,7.956241e-01)(22,8.158185e-01)(24,8.338825e-01)
};
   \addplot[color=red,mark=diamond,mark size=2.2pt,only marks]  plot coordinates{
  (2.000000e-01,8.561958e-01)(6.000000e-01,6.957835e-01)(1,6.291499e-01)(1.400000e+00,6.036073e-01)(1.800000e+00,6.032562e-01)(2.400000e+00,6.269303e-01)(3,6.663036e-01)(5,7.995289e-01)(7,9.370498e-01)(9,1.045049e+00)(11,1.140987e+00)(13,1.233288e+00)(15,1.302845e+00)(17,1.367215e+00)(19,1.424551e+00)(21,1.471076e+00)(23,1.520065e+00)(25,1.555158e+00)};

           \end{axis}
  \end{tikzpicture}}
\caption{Prediction risk of  S-SVR vs $\epsilon$ and $C$ when $\delta=2$, $\beta=1$ and $\sigma=1$. The continuous line curves correspond to the theoretical predictions while the points denote finite-sample performance when $p=200$ and and $n= \lfloor{\delta p}\rfloor$.}
\label{SMvsepsandC}
\end{center}
\end{figure}
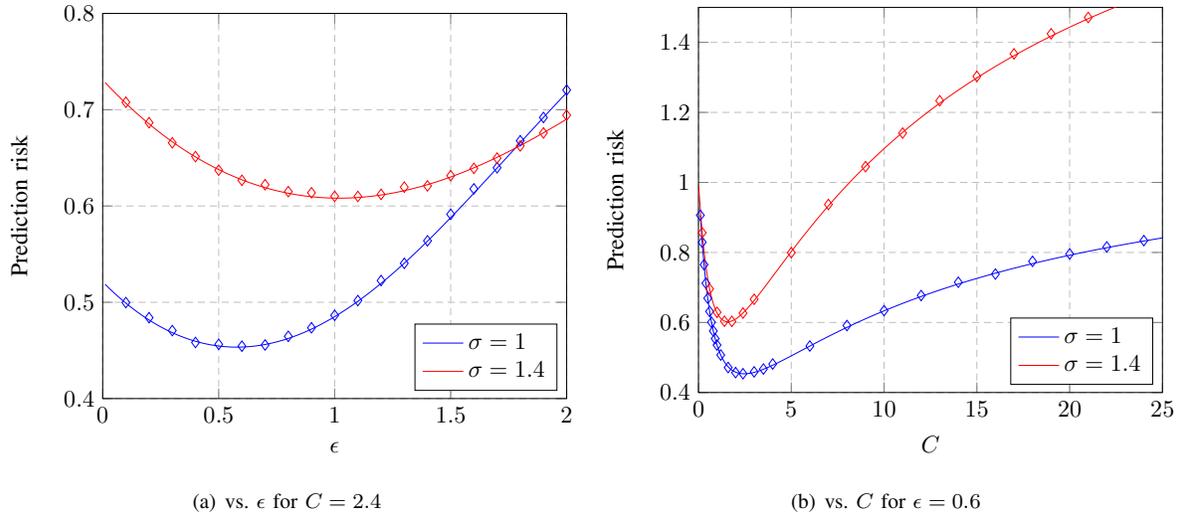

\begin{figure}[]
\centering
\includegraphics[scale=0.8]{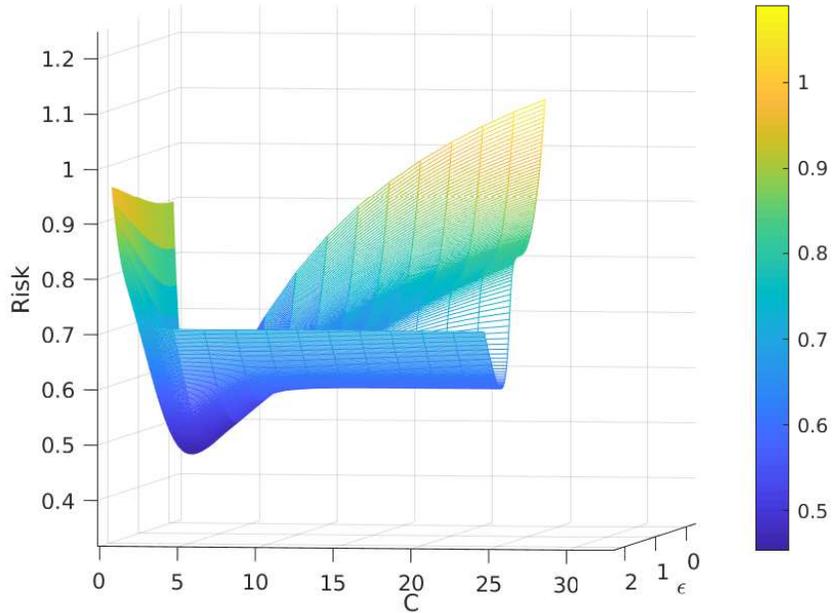}
\caption{Prediction risk vs ($C$,$\epsilon$) for $\delta=2$, $\beta=1$ and $\sigma=1$.}
\label{SMvsepsandC3D}
\end{figure}

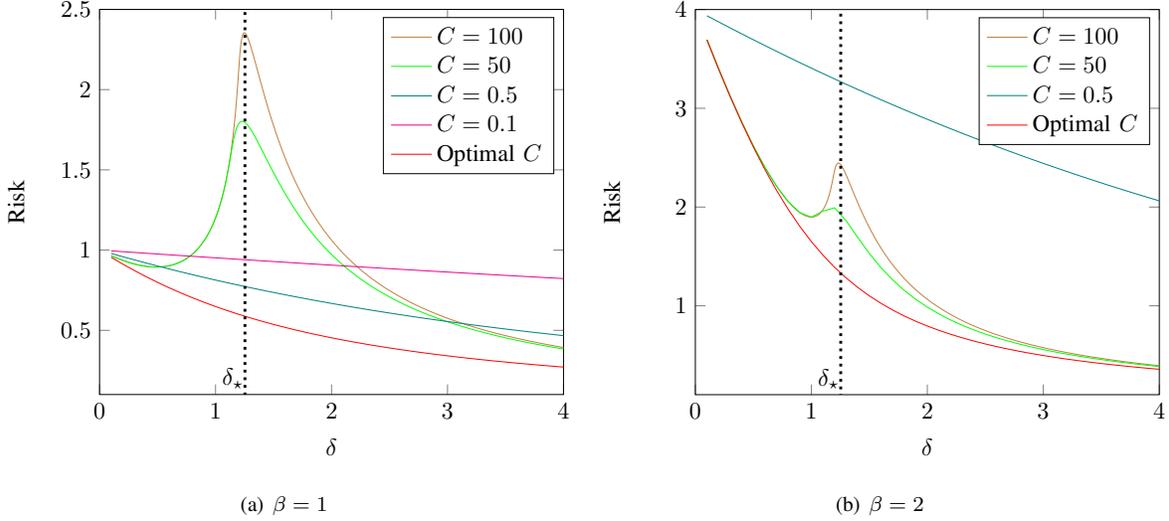
\begin{figure}[]
\begin{center}
\subfigure[$\beta=1$]{
\begin{tikzpicture}[scale=0.9,font=\fontsize{10}{10}\selectfont]
   \tikzstyle{every axis y label}+=[yshift=0pt]
   \tikzstyle{every axis x label}+=[yshift=5pt]
   \tikzstyle{every axis legend}+=[cells={anchor=west},fill=white,
        at={(0.98,0.98)}, anchor=north east, font=\fontsize{10}{10}\selectfont]
   \begin{axis}[
      xmin=0,
      ymin=0.1,
      xmax=4,
      ymax=2.5,
xlabel={$\delta$},
ylabel={Risk }	]		

\addplot[color=brown] coordinates{
(1.000000e-01,9.649447e-01)(1.100000e-01,9.617422e-01)(1.200000e-01,9.586624e-01)(1.300000e-01,9.555751e-01)(1.400000e-01,9.525479e-01)(1.500000e-01,9.495962e-01)(1.600000e-01,9.467075e-01)(1.700000e-01,9.438830e-01)(1.800000e-01,9.411235e-01)(1.900000e-01,9.384303e-01)(2.000000e-01,9.358045e-01)(2.100000e-01,9.332472e-01)(2.200000e-01,9.307598e-01)(2.300000e-01,9.282792e-01)(2.400000e-01,9.259355e-01)(2.500000e-01,9.236656e-01)(2.600000e-01,9.214712e-01)(2.700000e-01,9.193432e-01)(2.800000e-01,9.172883e-01)(2.900000e-01,9.153109e-01)(3.000000e-01,9.134124e-01)(3.100000e-01,9.115943e-01)(3.200000e-01,9.098580e-01)(3.300000e-01,9.082050e-01)(3.400000e-01,9.066371e-01)(3.500000e-01,9.051559e-01)(3.600000e-01,9.037630e-01)(3.700000e-01,9.024603e-01)(3.800000e-01,9.012495e-01)(3.900000e-01,9.001327e-01)(4.000000e-01,8.991116e-01)(4.100000e-01,8.981885e-01)(4.200000e-01,8.973652e-01)(4.300000e-01,8.966440e-01)(4.400000e-01,8.960235e-01)(4.500000e-01,8.955137e-01)(4.600000e-01,8.951132e-01)(4.700000e-01,8.948246e-01)(4.800000e-01,8.946505e-01)(4.900000e-01,8.945936e-01)(5.000000e-01,8.946567e-01)(5.100000e-01,8.948426e-01)(5.200000e-01,8.951541e-01)(5.300000e-01,8.955943e-01)(5.400000e-01,8.961660e-01)(5.500000e-01,8.968720e-01)(5.600000e-01,8.977152e-01)(5.700000e-01,8.987613e-01)(5.800000e-01,8.998878e-01)(5.900000e-01,9.011893e-01)(6.000000e-01,9.026486e-01)(6.100000e-01,9.042696e-01)(6.200000e-01,9.060582e-01)(6.300000e-01,9.080189e-01)(6.400000e-01,9.101571e-01)(6.500000e-01,9.124783e-01)(6.600000e-01,9.149884e-01)(6.700000e-01,9.177579e-01)(6.800000e-01,9.206673e-01)(6.900000e-01,9.237937e-01)(7.000000e-01,9.271505e-01)(7.100000e-01,9.307007e-01)(7.200000e-01,9.345133e-01)(7.300000e-01,9.385714e-01)(7.400000e-01,9.428951e-01)(7.500000e-01,9.474849e-01)(7.600000e-01,9.523420e-01)(7.700000e-01,9.574580e-01)(7.800000e-01,9.628733e-01)(7.900000e-01,9.686774e-01)(8.000000e-01,9.747443e-01)(8.100000e-01,9.811641e-01)(8.200000e-01,9.879469e-01)(8.300000e-01,9.950880e-01)(8.400000e-01,1.002693e+00)(8.500000e-01,1.010635e+00)(8.600000e-01,1.019041e+00)(8.700000e-01,1.027911e+00)(8.800000e-01,1.037254e+00)(8.900000e-01,1.047118e+00)(9.000000e-01,1.057525e+00)(9.100000e-01,1.068511e+00)(9.200000e-01,1.080097e+00)(9.300000e-01,1.092365e+00)(9.400000e-01,1.105253e+00)(9.500000e-01,1.118936e+00)(9.600000e-01,1.133387e+00)(9.700000e-01,1.148719e+00)(9.800000e-01,1.164936e+00)(9.900000e-01,1.182136e+00)(1,1.200394e+00)(1.010000e+00,1.219794e+00)(1.020000e+00,1.240436e+00)(1.030000e+00,1.262429e+00)(1.040000e+00,1.285901e+00)(1.050000e+00,1.310995e+00)(1.060000e+00,1.337872e+00)(1.070000e+00,1.366803e+00)(1.080000e+00,1.397858e+00)(1.090000e+00,1.431381e+00)(1.100000e+00,1.467668e+00)(1.110000e+00,1.507154e+00)(1.120000e+00,1.550221e+00)(1.130000e+00,1.597358e+00)(1.140000e+00,1.649385e+00)(1.150000e+00,1.707045e+00)(1.160000e+00,1.771408e+00)(1.170000e+00,1.843982e+00)(1.180000e+00,1.926581e+00)(1.190000e+00,2.020787e+00)(1.200000e+00,2.124801e+00)(1.210000e+00,2.223324e+00)(1.220000e+00,2.293908e+00)(1.230000e+00,2.333553e+00)(1.240000e+00,2.351110e+00)(1.250000e+00,2.354397e+00)(1.260000e+00,2.348325e+00)(1.270000e+00,2.335917e+00)(1.280000e+00,2.319101e+00)(1.290000e+00,2.299166e+00)(1.300000e+00,2.277001e+00)(1.310000e+00,2.253245e+00)(1.320000e+00,2.228366e+00)(1.330000e+00,2.202715e+00)(1.340000e+00,2.176558e+00)(1.350000e+00,2.150103e+00)(1.360000e+00,2.123511e+00)(1.370000e+00,2.096810e+00)(1.380000e+00,2.070296e+00)(1.390000e+00,2.043949e+00)(1.400000e+00,2.017829e+00)(1.410000e+00,1.991969e+00)(1.420000e+00,1.966425e+00)(1.430000e+00,1.941225e+00)(1.440000e+00,1.916393e+00)(1.450000e+00,1.891947e+00)(1.460000e+00,1.867899e+00)(1.470000e+00,1.844258e+00)(1.480000e+00,1.821030e+00)(1.490000e+00,1.798219e+00)(1.500000e+00,1.775825e+00)(1.510000e+00,1.753847e+00)(1.520000e+00,1.732284e+00)(1.530000e+00,1.711132e+00)(1.540000e+00,1.690387e+00)(1.550000e+00,1.670043e+00)(1.560000e+00,1.650107e+00)(1.570000e+00,1.630549e+00)(1.580000e+00,1.611373e+00)(1.590000e+00,1.592573e+00)(1.600000e+00,1.574224e+00)(1.610000e+00,1.556152e+00)(1.620000e+00,1.538433e+00)(1.630000e+00,1.521058e+00)(1.640000e+00,1.504018e+00)(1.650000e+00,1.487337e+00)(1.660000e+00,1.470979e+00)(1.670000e+00,1.454962e+00)(1.680000e+00,1.439203e+00)(1.690000e+00,1.423746e+00)(1.700000e+00,1.408635e+00)(1.710000e+00,1.393786e+00)(1.720000e+00,1.379217e+00)(1.730000e+00,1.364900e+00)(1.740000e+00,1.350897e+00)(1.750000e+00,1.337136e+00)(1.760000e+00,1.323630e+00)(1.770000e+00,1.310374e+00)(1.780000e+00,1.297361e+00)(1.790000e+00,1.284585e+00)(1.800000e+00,1.272041e+00)(1.810000e+00,1.259723e+00)(1.820000e+00,1.247625e+00)(1.830000e+00,1.235743e+00)(1.840000e+00,1.224071e+00)(1.850000e+00,1.212604e+00)(1.860000e+00,1.201337e+00)(1.870000e+00,1.190265e+00)(1.880000e+00,1.179385e+00)(1.890000e+00,1.168690e+00)(1.900000e+00,1.158177e+00)(1.910000e+00,1.147842e+00)(1.920000e+00,1.137681e+00)(1.930000e+00,1.127689e+00)(1.940000e+00,1.117863e+00)(1.950000e+00,1.108199e+00)(1.960000e+00,1.098692e+00)(1.970000e+00,1.089340e+00)(1.980000e+00,1.080139e+00)(1.990000e+00,1.071085e+00)(2,1.062175e+00)(2.010000e+00,1.053406e+00)(2.020000e+00,1.044775e+00)(2.030000e+00,1.036279e+00)(2.040000e+00,1.027914e+00)(2.050000e+00,1.019678e+00)(2.060000e+00,1.011568e+00)(2.070000e+00,1.003581e+00)(2.080000e+00,9.957144e-01)(2.090000e+00,9.879518e-01)(2.100000e+00,9.803330e-01)(2.110000e+00,9.728279e-01)(2.120000e+00,9.654218e-01)(2.130000e+00,9.581235e-01)(2.140000e+00,9.509308e-01)(2.150000e+00,9.438414e-01)(2.160000e+00,9.368532e-01)(2.170000e+00,9.299323e-01)(2.180000e+00,9.231444e-01)(2.190000e+00,9.164749e-01)(2.200000e+00,9.098708e-01)(2.210000e+00,9.033579e-01)(2.220000e+00,8.969344e-01)(2.230000e+00,8.905983e-01)(2.240000e+00,8.843481e-01)(2.250000e+00,8.782444e-01)(2.260000e+00,8.721604e-01)(2.270000e+00,8.661572e-01)(2.280000e+00,8.602355e-01)(2.290000e+00,8.543996e-01)(2.300000e+00,8.486405e-01)(2.310000e+00,8.429567e-01)(2.320000e+00,8.373470e-01)(2.330000e+00,8.318098e-01)(2.340000e+00,8.263438e-01)(2.350000e+00,8.209477e-01)(2.360000e+00,8.156202e-01)(2.370000e+00,8.103599e-01)(2.380000e+00,8.051658e-01)(2.390000e+00,8.000365e-01)(2.400000e+00,7.949710e-01)(2.410000e+00,7.899680e-01)(2.420000e+00,7.850264e-01)(2.430000e+00,7.801451e-01)(2.440000e+00,7.753231e-01)(2.450000e+00,7.705593e-01)(2.460000e+00,7.658527e-01)(2.470000e+00,7.612023e-01)(2.480000e+00,7.566071e-01)(2.490000e+00,7.520661e-01)(2.500000e+00,7.475785e-01)(2.510000e+00,7.431433e-01)(2.520000e+00,7.387596e-01)(2.530000e+00,7.344265e-01)(2.540000e+00,7.301432e-01)(2.550000e+00,7.259089e-01)(2.560000e+00,7.217227e-01)(2.570000e+00,7.175838e-01)(2.580000e+00,7.134915e-01)(2.590000e+00,7.094450e-01)(2.600000e+00,7.054435e-01)(2.610000e+00,7.014863e-01)(2.620000e+00,6.975727e-01)(2.630000e+00,6.937020e-01)(2.640000e+00,6.898735e-01)(2.650000e+00,6.860866e-01)(2.660000e+00,6.823404e-01)(2.670000e+00,6.785796e-01)(2.680000e+00,6.749135e-01)(2.690000e+00,6.712863e-01)(2.700000e+00,6.676974e-01)(2.710000e+00,6.641464e-01)(2.720000e+00,6.606325e-01)(2.730000e+00,6.571552e-01)(2.740000e+00,6.537140e-01)(2.750000e+00,6.503083e-01)(2.760000e+00,6.469376e-01)(2.770000e+00,6.436013e-01)(2.780000e+00,6.402989e-01)(2.790000e+00,6.370299e-01)(2.800000e+00,6.337939e-01)(2.810000e+00,6.305891e-01)(2.820000e+00,6.274152e-01)(2.830000e+00,6.242727e-01)(2.840000e+00,6.211612e-01)(2.850000e+00,6.180803e-01)(2.860000e+00,6.150294e-01)(2.870000e+00,6.120082e-01)(2.880000e+00,6.090162e-01)(2.890000e+00,6.060531e-01)(2.900000e+00,6.031183e-01)(2.910000e+00,6.002116e-01)(2.920000e+00,5.973324e-01)(2.930000e+00,5.944805e-01)(2.940000e+00,5.916555e-01)(2.950000e+00,5.888569e-01)(2.960000e+00,5.860844e-01)(2.970000e+00,5.833376e-01)(2.980000e+00,5.806162e-01)(2.990000e+00,5.779199e-01)(3,5.752483e-01)(3.010000e+00,5.726011e-01)(3.020000e+00,5.699780e-01)(3.030000e+00,5.673785e-01)(3.040000e+00,5.648025e-01)(3.050000e+00,5.622495e-01)(3.060000e+00,5.597194e-01)(3.070000e+00,5.572117e-01)(3.080000e+00,5.547262e-01)(3.090000e+00,5.522719e-01)(3.100000e+00,5.498321e-01)(3.110000e+00,5.474136e-01)(3.120000e+00,5.450160e-01)(3.130000e+00,5.426393e-01)(3.140000e+00,5.402830e-01)(3.150000e+00,5.379228e-01)(3.160000e+00,5.356187e-01)(3.170000e+00,5.332864e-01)(3.180000e+00,5.310575e-01)(3.190000e+00,5.287513e-01)(3.200000e+00,5.265127e-01)(3.210000e+00,5.242928e-01)(3.220000e+00,5.220914e-01)(3.230000e+00,5.199083e-01)(3.240000e+00,5.177433e-01)(3.250000e+00,5.155961e-01)(3.260000e+00,5.134665e-01)(3.270000e+00,5.113543e-01)(3.280000e+00,5.092593e-01)(3.290000e+00,5.071813e-01)(3.300000e+00,5.051201e-01)(3.310000e+00,5.030753e-01)(3.320000e+00,5.010455e-01)(3.330000e+00,4.990319e-01)(3.340000e+00,4.970343e-01)(3.350000e+00,4.950525e-01)(3.360000e+00,4.930864e-01)(3.370000e+00,4.911357e-01)(3.380000e+00,4.892003e-01)(3.390000e+00,4.872800e-01)(3.400000e+00,4.853746e-01)(3.410000e+00,4.834840e-01)(3.420000e+00,4.816080e-01)(3.430000e+00,4.797464e-01)(3.440000e+00,4.778990e-01)(3.450000e+00,4.760658e-01)(3.460000e+00,4.742464e-01)(3.470000e+00,4.724409e-01)(3.480000e+00,4.706489e-01)(3.490000e+00,4.688705e-01)(3.500000e+00,4.671053e-01)(3.510000e+00,4.653533e-01)(3.520000e+00,4.636143e-01)(3.530000e+00,4.618883e-01)(3.540000e+00,4.601749e-01)(3.550000e+00,4.584742e-01)(3.560000e+00,4.567859e-01)(3.570000e+00,4.551099e-01)(3.580000e+00,4.534461e-01)(3.590000e+00,4.517944e-01)(3.600000e+00,4.501546e-01)(3.610000e+00,4.485266e-01)(3.620000e+00,4.469103e-01)(3.630000e+00,4.453055e-01)(3.640000e+00,4.437122e-01)(3.650000e+00,4.421302e-01)(3.660000e+00,4.405593e-01)(3.670000e+00,4.389996e-01)(3.680000e+00,4.374507e-01)(3.690000e+00,4.359128e-01)(3.700000e+00,4.343855e-01)(3.710000e+00,4.328689e-01)(3.720000e+00,4.313627e-01)(3.730000e+00,4.298670e-01)(3.740000e+00,4.283816e-01)(3.750000e+00,4.269063e-01)(3.760000e+00,4.254411e-01)(3.770000e+00,4.239859e-01)(3.780000e+00,4.225406e-01)(3.790000e+00,4.211051e-01)(3.800000e+00,4.196792e-01)(3.810000e+00,4.182629e-01)(3.820000e+00,4.168561e-01)(3.830000e+00,4.154587e-01)(3.840000e+00,4.140706e-01)(3.850000e+00,4.126917e-01)(3.860000e+00,4.113220e-01)(3.870000e+00,4.099612e-01)(3.880000e+00,4.086094e-01)(3.890000e+00,4.072664e-01)(3.900000e+00,4.059322e-01)(3.910000e+00,4.046067e-01)(3.920000e+00,4.032898e-01)(3.930000e+00,4.019814e-01)(3.940000e+00,4.006814e-01)(3.950000e+00,3.993898e-01)(3.960000e+00,3.981064e-01)(3.970000e+00,3.968312e-01)(3.980000e+00,3.955641e-01)(3.990000e+00,3.943051e-01)(4,3.930541e-01)};\addlegendentry{$C=100$}

\addplot[color=green] coordinates{
(1.000000e-01,9.649447e-01)(1.100000e-01,9.617422e-01)(1.200000e-01,9.586624e-01)(1.300000e-01,9.555751e-01)(1.400000e-01,9.525479e-01)(1.500000e-01,9.495962e-01)(1.600000e-01,9.467075e-01)(1.700000e-01,9.438830e-01)(1.800000e-01,9.411235e-01)(1.900000e-01,9.384303e-01)(2.000000e-01,9.358045e-01)(2.100000e-01,9.332472e-01)(2.200000e-01,9.307598e-01)(2.300000e-01,9.282792e-01)(2.400000e-01,9.259355e-01)(2.500000e-01,9.236656e-01)(2.600000e-01,9.214712e-01)(2.700000e-01,9.193432e-01)(2.800000e-01,9.172883e-01)(2.900000e-01,9.153109e-01)(3.000000e-01,9.134124e-01)(3.100000e-01,9.115943e-01)(3.200000e-01,9.098580e-01)(3.300000e-01,9.082050e-01)(3.400000e-01,9.066371e-01)(3.500000e-01,9.051559e-01)(3.600000e-01,9.037630e-01)(3.700000e-01,9.024603e-01)(3.800000e-01,9.012495e-01)(3.900000e-01,9.001327e-01)(4.000000e-01,8.991116e-01)(4.100000e-01,8.981885e-01)(4.200000e-01,8.973652e-01)(4.300000e-01,8.966440e-01)(4.400000e-01,8.960235e-01)(4.500000e-01,8.955137e-01)(4.600000e-01,8.951132e-01)(4.700000e-01,8.948246e-01)(4.800000e-01,8.946505e-01)(4.900000e-01,8.945936e-01)(5.000000e-01,8.946567e-01)(5.100000e-01,8.948426e-01)(5.200000e-01,8.951541e-01)(5.300000e-01,8.955943e-01)(5.400000e-01,8.961660e-01)(5.500000e-01,8.968720e-01)(5.600000e-01,8.977152e-01)(5.700000e-01,8.987613e-01)(5.800000e-01,8.998878e-01)(5.900000e-01,9.011893e-01)(6.000000e-01,9.026486e-01)(6.100000e-01,9.042696e-01)(6.200000e-01,9.060582e-01)(6.300000e-01,9.080189e-01)(6.400000e-01,9.101571e-01)(6.500000e-01,9.124783e-01)(6.600000e-01,9.149884e-01)(6.700000e-01,9.177579e-01)(6.800000e-01,9.206673e-01)(6.900000e-01,9.237937e-01)(7.000000e-01,9.271597e-01)(7.100000e-01,9.306858e-01)(7.200000e-01,9.345261e-01)(7.300000e-01,9.385755e-01)(7.400000e-01,9.428832e-01)(7.500000e-01,9.474534e-01)(7.600000e-01,9.523067e-01)(7.700000e-01,9.574543e-01)(7.800000e-01,9.629042e-01)(7.900000e-01,9.686666e-01)(8.000000e-01,9.747593e-01)(8.100000e-01,9.811544e-01)(8.200000e-01,9.879329e-01)(8.300000e-01,9.951085e-01)(8.400000e-01,1.002663e+00)(8.500000e-01,1.010633e+00)(8.600000e-01,1.019039e+00)(8.700000e-01,1.027905e+00)(8.800000e-01,1.037256e+00)(8.900000e-01,1.047119e+00)(9.000000e-01,1.057524e+00)(9.100000e-01,1.068516e+00)(9.200000e-01,1.080110e+00)(9.300000e-01,1.092352e+00)(9.400000e-01,1.105252e+00)(9.500000e-01,1.118950e+00)(9.600000e-01,1.133411e+00)(9.700000e-01,1.148722e+00)(9.800000e-01,1.164938e+00)(9.900000e-01,1.182132e+00)(1,1.200380e+00)(1.010000e+00,1.219775e+00)(1.020000e+00,1.240453e+00)(1.030000e+00,1.262417e+00)(1.040000e+00,1.285870e+00)(1.050000e+00,1.311024e+00)(1.060000e+00,1.337885e+00)(1.070000e+00,1.366658e+00)(1.080000e+00,1.397626e+00)(1.090000e+00,1.430913e+00)(1.100000e+00,1.466712e+00)(1.110000e+00,1.504872e+00)(1.120000e+00,1.545229e+00)(1.130000e+00,1.586916e+00)(1.140000e+00,1.628571e+00)(1.150000e+00,1.668160e+00)(1.160000e+00,1.703642e+00)(1.170000e+00,1.733706e+00)(1.180000e+00,1.757702e+00)(1.190000e+00,1.775822e+00)(1.200000e+00,1.788665e+00)(1.210000e+00,1.796957e+00)(1.220000e+00,1.801376e+00)(1.230000e+00,1.802563e+00)(1.240000e+00,1.801042e+00)(1.250000e+00,1.797261e+00)(1.260000e+00,1.791591e+00)(1.270000e+00,1.784342e+00)(1.280000e+00,1.775772e+00)(1.290000e+00,1.766096e+00)(1.300000e+00,1.755497e+00)(1.310000e+00,1.744126e+00)(1.320000e+00,1.732114e+00)(1.330000e+00,1.719572e+00)(1.340000e+00,1.706594e+00)(1.350000e+00,1.693261e+00)(1.360000e+00,1.679643e+00)(1.370000e+00,1.665802e+00)(1.380000e+00,1.651789e+00)(1.390000e+00,1.637649e+00)(1.400000e+00,1.623422e+00)(1.410000e+00,1.609143e+00)(1.420000e+00,1.594861e+00)(1.430000e+00,1.580565e+00)(1.440000e+00,1.566295e+00)(1.450000e+00,1.552072e+00)(1.460000e+00,1.537911e+00)(1.470000e+00,1.523829e+00)(1.480000e+00,1.509838e+00)(1.490000e+00,1.495949e+00)(1.500000e+00,1.482172e+00)(1.510000e+00,1.468515e+00)(1.520000e+00,1.454986e+00)(1.530000e+00,1.441590e+00)(1.540000e+00,1.428253e+00)(1.550000e+00,1.415139e+00)(1.560000e+00,1.402171e+00)(1.570000e+00,1.389352e+00)(1.580000e+00,1.376685e+00)(1.590000e+00,1.364184e+00)(1.600000e+00,1.351807e+00)(1.610000e+00,1.339594e+00)(1.620000e+00,1.327537e+00)(1.630000e+00,1.315635e+00)(1.640000e+00,1.303889e+00)(1.650000e+00,1.292296e+00)(1.660000e+00,1.280858e+00)(1.670000e+00,1.269573e+00)(1.680000e+00,1.258440e+00)(1.690000e+00,1.247458e+00)(1.700000e+00,1.236626e+00)(1.710000e+00,1.225942e+00)(1.720000e+00,1.215404e+00)(1.730000e+00,1.205013e+00)(1.740000e+00,1.194764e+00)(1.750000e+00,1.184658e+00)(1.760000e+00,1.174691e+00)(1.770000e+00,1.164863e+00)(1.780000e+00,1.155172e+00)(1.790000e+00,1.145615e+00)(1.800000e+00,1.136191e+00)(1.810000e+00,1.126897e+00)(1.820000e+00,1.117733e+00)(1.830000e+00,1.108696e+00)(1.840000e+00,1.099784e+00)(1.850000e+00,1.090995e+00)(1.860000e+00,1.082327e+00)(1.870000e+00,1.073778e+00)(1.880000e+00,1.065347e+00)(1.890000e+00,1.057031e+00)(1.900000e+00,1.048829e+00)(1.910000e+00,1.040739e+00)(1.920000e+00,1.032758e+00)(1.930000e+00,1.024886e+00)(1.940000e+00,1.017120e+00)(1.950000e+00,1.009458e+00)(1.960000e+00,1.001899e+00)(1.970000e+00,9.944417e-01)(1.980000e+00,9.870834e-01)(1.990000e+00,9.798228e-01)(2,9.726669e-01)(2.010000e+00,9.655995e-01)(2.020000e+00,9.586245e-01)(2.030000e+00,9.517403e-01)(2.040000e+00,9.449455e-01)(2.050000e+00,9.382384e-01)(2.060000e+00,9.315986e-01)(2.070000e+00,9.250775e-01)(2.080000e+00,9.186384e-01)(2.090000e+00,9.122583e-01)(2.100000e+00,9.059681e-01)(2.110000e+00,8.997571e-01)(2.120000e+00,8.936238e-01)(2.130000e+00,8.875669e-01)(2.140000e+00,8.815852e-01)(2.150000e+00,8.757398e-01)(2.160000e+00,8.699043e-01)(2.170000e+00,8.641403e-01)(2.180000e+00,8.584465e-01)(2.190000e+00,8.528244e-01)(2.200000e+00,8.472767e-01)(2.210000e+00,8.417963e-01)(2.220000e+00,8.363820e-01)(2.230000e+00,8.310328e-01)(2.240000e+00,8.257476e-01)(2.250000e+00,8.205253e-01)(2.260000e+00,8.153650e-01)(2.270000e+00,8.102656e-01)(2.280000e+00,8.052261e-01)(2.290000e+00,8.002455e-01)(2.300000e+00,7.953229e-01)(2.310000e+00,7.904574e-01)(2.320000e+00,7.856480e-01)(2.330000e+00,7.808939e-01)(2.340000e+00,7.761941e-01)(2.350000e+00,7.715478e-01)(2.360000e+00,7.669541e-01)(2.370000e+00,7.624122e-01)(2.380000e+00,7.579213e-01)(2.390000e+00,7.534806e-01)(2.400000e+00,7.490893e-01)(2.410000e+00,7.447466e-01)(2.420000e+00,7.404518e-01)(2.430000e+00,7.362042e-01)(2.440000e+00,7.320029e-01)(2.450000e+00,7.278473e-01)(2.460000e+00,7.237367e-01)(2.470000e+00,7.196704e-01)(2.480000e+00,7.156478e-01)(2.490000e+00,7.116681e-01)(2.500000e+00,7.077307e-01)(2.510000e+00,7.038350e-01)(2.520000e+00,6.999803e-01)(2.530000e+00,6.961661e-01)(2.540000e+00,6.923918e-01)(2.550000e+00,6.886566e-01)(2.560000e+00,6.849602e-01)(2.570000e+00,6.813018e-01)(2.580000e+00,6.776810e-01)(2.590000e+00,6.740971e-01)(2.600000e+00,6.705498e-01)(2.610000e+00,6.670383e-01)(2.620000e+00,6.635622e-01)(2.630000e+00,6.600669e-01)(2.640000e+00,6.566603e-01)(2.650000e+00,6.532875e-01)(2.660000e+00,6.499482e-01)(2.670000e+00,6.466419e-01)(2.680000e+00,6.433680e-01)(2.690000e+00,6.401261e-01)(2.700000e+00,6.369158e-01)(2.710000e+00,6.337367e-01)(2.720000e+00,6.305882e-01)(2.730000e+00,6.274701e-01)(2.740000e+00,6.243817e-01)(2.750000e+00,6.213229e-01)(2.760000e+00,6.182930e-01)(2.770000e+00,6.152918e-01)(2.780000e+00,6.123186e-01)(2.790000e+00,6.093715e-01)(2.800000e+00,6.064518e-01)(2.810000e+00,6.035593e-01)(2.820000e+00,6.006934e-01)(2.830000e+00,5.978540e-01)(2.840000e+00,5.950405e-01)(2.850000e+00,5.922527e-01)(2.860000e+00,5.894903e-01)(2.870000e+00,5.867528e-01)(2.880000e+00,5.840401e-01)(2.890000e+00,5.813516e-01)(2.900000e+00,5.786872e-01)(2.910000e+00,5.760465e-01)(2.920000e+00,5.734293e-01)(2.930000e+00,5.708351e-01)(2.940000e+00,5.682638e-01)(2.950000e+00,5.657150e-01)(2.960000e+00,5.631884e-01)(2.970000e+00,5.606838e-01)(2.980000e+00,5.582008e-01)(2.990000e+00,5.557393e-01)(3,5.532988e-01)(3.010000e+00,5.508793e-01)(3.020000e+00,5.484804e-01)(3.030000e+00,5.461018e-01)(3.040000e+00,5.437433e-01)(3.050000e+00,5.414047e-01)(3.060000e+00,5.390857e-01)(3.070000e+00,5.367704e-01)(3.080000e+00,5.345057e-01)(3.090000e+00,5.322696e-01)(3.100000e+00,5.300285e-01)(3.110000e+00,5.277574e-01)(3.120000e+00,5.255530e-01)(3.130000e+00,5.233666e-01)(3.140000e+00,5.211980e-01)(3.150000e+00,5.190469e-01)(3.160000e+00,5.169132e-01)(3.170000e+00,5.147966e-01)(3.180000e+00,5.126969e-01)(3.190000e+00,5.106140e-01)(3.200000e+00,5.085477e-01)(3.210000e+00,5.064977e-01)(3.220000e+00,5.044639e-01)(3.230000e+00,5.024461e-01)(3.240000e+00,5.004440e-01)(3.250000e+00,4.984563e-01)(3.260000e+00,4.964838e-01)(3.270000e+00,4.945267e-01)(3.280000e+00,4.925847e-01)(3.290000e+00,4.906576e-01)(3.300000e+00,4.887452e-01)(3.310000e+00,4.868475e-01)(3.320000e+00,4.849642e-01)(3.330000e+00,4.830952e-01)(3.340000e+00,4.812403e-01)(3.350000e+00,4.793994e-01)(3.360000e+00,4.775723e-01)(3.370000e+00,4.757588e-01)(3.380000e+00,4.739588e-01)(3.390000e+00,4.721722e-01)(3.400000e+00,4.703988e-01)(3.410000e+00,4.686385e-01)(3.420000e+00,4.668911e-01)(3.430000e+00,4.651565e-01)(3.440000e+00,4.634346e-01)(3.450000e+00,4.617251e-01)(3.460000e+00,4.600281e-01)(3.470000e+00,4.583433e-01)(3.480000e+00,4.566706e-01)(3.490000e+00,4.550099e-01)(3.500000e+00,4.533611e-01)(3.510000e+00,4.517240e-01)(3.520000e+00,4.500985e-01)(3.530000e+00,4.484845e-01)(3.540000e+00,4.468819e-01)(3.550000e+00,4.452906e-01)(3.560000e+00,4.437104e-01)(3.570000e+00,4.421412e-01)(3.580000e+00,4.405830e-01)(3.590000e+00,4.390355e-01)(3.600000e+00,4.374987e-01)(3.610000e+00,4.359725e-01)(3.620000e+00,4.344568e-01)(3.630000e+00,4.329515e-01)(3.640000e+00,4.314564e-01)(3.650000e+00,4.299715e-01)(3.660000e+00,4.284966e-01)(3.670000e+00,4.270317e-01)(3.680000e+00,4.255767e-01)(3.690000e+00,4.241314e-01)(3.700000e+00,4.226958e-01)(3.710000e+00,4.212697e-01)(3.720000e+00,4.198532e-01)(3.730000e+00,4.184460e-01)(3.740000e+00,4.170481e-01)(3.750000e+00,4.156594e-01)(3.760000e+00,4.142798e-01)(3.770000e+00,4.129093e-01)(3.780000e+00,4.115477e-01)(3.790000e+00,4.101949e-01)(3.800000e+00,4.088509e-01)(3.810000e+00,4.075155e-01)(3.820000e+00,4.061888e-01)(3.830000e+00,4.048706e-01)(3.840000e+00,4.035608e-01)(3.850000e+00,4.022594e-01)(3.860000e+00,4.009663e-01)(3.870000e+00,3.996813e-01)(3.880000e+00,3.984045e-01)(3.890000e+00,3.971357e-01)(3.900000e+00,3.958749e-01)(3.910000e+00,3.946220e-01)(3.920000e+00,3.933769e-01)(3.930000e+00,3.921396e-01)(3.940000e+00,3.909099e-01)(3.950000e+00,3.896879e-01)(3.960000e+00,3.884734e-01)(3.970000e+00,3.872664e-01)(3.980000e+00,3.860667e-01)(3.990000e+00,3.848744e-01)(4,3.836894e-01)};\addlegendentry{$C=50$}

\addplot[color=teal] coordinates{
(1.000000e-01,9.791844e-01)(2.000000e-01,9.589190e-01)(3.000000e-01,9.391091e-01)(4.000000e-01,9.197739e-01)(5.000000e-01,9.009041e-01)(6.000000e-01,8.824903e-01)(7.000000e-01,8.645231e-01)(8.000000e-01,8.469930e-01)(9.000000e-01,8.299513e-01)(1,8.132666e-01)(1.100000e+00,7.970094e-01)(1.200000e+00,7.811574e-01)(1.300000e+00,7.656999e-01)(1.400000e+00,7.506217e-01)(1.500000e+00,7.359246e-01)(1.600000e+00,7.215952e-01)(1.700000e+00,7.076248e-01)(1.800000e+00,6.940046e-01)(1.900000e+00,6.807263e-01)(2,6.677815e-01)(2.100000e+00,6.551618e-01)(2.200000e+00,6.428591e-01)(2.300000e+00,6.308653e-01)(2.400000e+00,6.191202e-01)(2.500000e+00,6.077214e-01)(2.600000e+00,5.966083e-01)(2.700000e+00,5.857734e-01)(2.800000e+00,5.752094e-01)(2.900000e+00,5.649016e-01)(3,5.548494e-01)(3.100000e+00,5.450464e-01)(3.200000e+00,5.354857e-01)(3.300000e+00,5.261608e-01)(3.400000e+00,5.170652e-01)(3.500000e+00,5.082073e-01)(3.600000e+00,4.995532e-01)(3.700000e+00,4.911102e-01)(3.800000e+00,4.828723e-01)(3.900000e+00,4.747880e-01)(4,4.669439e-01)};\addlegendentry{$C=0.5$}

\addplot[color=magenta] coordinates{
(1.000000e-01,9.950620e-01)(2.000000e-01,9.901544e-01)(3.000000e-01,9.852772e-01)(4.000000e-01,9.804301e-01)(5.000000e-01,9.755474e-01)(6.000000e-01,9.707608e-01)(7.000000e-01,9.659833e-01)(8.000000e-01,9.612296e-01)(9.000000e-01,9.565020e-01)(1,9.518005e-01)(1.100000e+00,9.471249e-01)(1.200000e+00,9.424750e-01)(1.300000e+00,9.378509e-01)(1.400000e+00,9.332523e-01)(1.500000e+00,9.286791e-01)(1.600000e+00,9.241313e-01)(1.700000e+00,9.196087e-01)(1.800000e+00,9.151113e-01)(1.900000e+00,9.106388e-01)(2,9.061912e-01)(2.100000e+00,9.017684e-01)(2.200000e+00,8.973703e-01)(2.300000e+00,8.929966e-01)(2.400000e+00,8.886474e-01)(2.500000e+00,8.843226e-01)(2.600000e+00,8.800219e-01)(2.700000e+00,8.757453e-01)(2.800000e+00,8.714926e-01)(2.900000e+00,8.672639e-01)(3,8.630588e-01)(3.100000e+00,8.588774e-01)(3.200000e+00,8.547195e-01)(3.300000e+00,8.505850e-01)(3.400000e+00,8.464737e-01)(3.500000e+00,8.423857e-01)(3.600000e+00,8.383206e-01)(3.700000e+00,8.342786e-01)(3.800000e+00,8.303201e-01)(3.900000e+00,8.263233e-01)(4,8.223492e-01)};\addlegendentry{$C=0.1$}

\addplot[color=red] coordinates{
(1.000000e-01,9.552221e-01)(2.000000e-01,9.128242e-01)(3.000000e-01,8.726703e-01)(4.000000e-01,8.347108e-01)(5.000000e-01,7.988974e-01)(6.000000e-01,7.650368e-01)(7.000000e-01,7.331897e-01)(8.000000e-01,7.031717e-01)(9.000000e-01,6.749157e-01)(1,6.483055e-01)(1.100000e+00,6.232504e-01)(1.200000e+00,5.996472e-01)(1.300000e+00,5.773850e-01)(1.400000e+00,5.564375e-01)(1.500000e+00,5.367292e-01)(1.600000e+00,5.181118e-01)(1.700000e+00,5.005419e-01)(1.800000e+00,4.839492e-01)(1.900000e+00,4.682699e-01)(2,4.534324e-01)(2.100000e+00,4.393939e-01)(2.200000e+00,4.260952e-01)(2.300000e+00,4.134867e-01)(2.400000e+00,4.015204e-01)(2.500000e+00,3.901789e-01)(2.600000e+00,3.793744e-01)(2.700000e+00,3.691103e-01)(2.800000e+00,3.593326e-01)(2.900000e+00,3.500243e-01)(3,3.411486e-01)(3.100000e+00,3.326785e-01)(3.200000e+00,3.245891e-01)(3.300000e+00,3.168566e-01)(3.400000e+00,3.094571e-01)(3.500000e+00,3.023736e-01)(3.600000e+00,2.955875e-01)(3.700000e+00,2.890819e-01)(3.800000e+00,2.828405e-01)(3.900000e+00,2.768511e-01)(4,2.710993e-01)};\addlegendentry{Optimal $C$}

  \addplot[color=black,mark =none,dotted,line width=1.2] coordinates{(1.2534,0)(1.2534,1)(1.2534,2.5)};
\node at (axis cs:1.15,0.2)(source1){${\delta_\star}$}; 
           \end{axis}
         \end{tikzpicture}}
         \subfigure[$\beta=2$]{
\begin{tikzpicture}[scale=0.9,font=\fontsize{10}{10}\selectfont]
   \tikzstyle{every axis y label}+=[yshift=0pt]
   \tikzstyle{every axis x label}+=[yshift=5pt]
   \tikzstyle{every axis legend}+=[cells={anchor=west},fill=white,
        at={(0.98,0.98)}, anchor=north east, font=\fontsize{10}{10}\selectfont]
   \begin{axis}[
      xmin=0,
      ymin=0.1,
      xmax=4,
      ymax=4,
xlabel={$\delta$},
ylabel={Risk }	]		

\addplot[color=brown] coordinates{
(1.000000e-01,3.694321e+00)(1.100000e-01,3.664419e+00)(1.200000e-01,3.634644e+00)(1.300000e-01,3.605000e+00)(1.400000e-01,3.575487e+00)(1.500000e-01,3.546108e+00)(1.600000e-01,3.516865e+00)(1.700000e-01,3.487759e+00)(1.800000e-01,3.458793e+00)(1.900000e-01,3.429969e+00)(2.000000e-01,3.401288e+00)(2.100000e-01,3.372753e+00)(2.200000e-01,3.344487e+00)(2.300000e-01,3.316249e+00)(2.400000e-01,3.288162e+00)(2.500000e-01,3.260230e+00)(2.600000e-01,3.232464e+00)(2.700000e-01,3.204868e+00)(2.800000e-01,3.177435e+00)(2.900000e-01,3.150167e+00)(3.000000e-01,3.123069e+00)(3.100000e-01,3.096140e+00)(3.200000e-01,3.069386e+00)(3.300000e-01,3.042807e+00)(3.400000e-01,3.016408e+00)(3.500000e-01,2.990190e+00)(3.600000e-01,2.964157e+00)(3.700000e-01,2.938311e+00)(3.800000e-01,2.912656e+00)(3.900000e-01,2.887194e+00)(4.000000e-01,2.861816e+00)(4.100000e-01,2.836751e+00)(4.200000e-01,2.811890e+00)(4.300000e-01,2.787235e+00)(4.400000e-01,2.762788e+00)(4.500000e-01,2.738538e+00)(4.600000e-01,2.714502e+00)(4.700000e-01,2.690683e+00)(4.800000e-01,2.667086e+00)(4.900000e-01,2.643715e+00)(5.000000e-01,2.620571e+00)(5.100000e-01,2.597660e+00)(5.200000e-01,2.574985e+00)(5.300000e-01,2.552550e+00)(5.400000e-01,2.530360e+00)(5.500000e-01,2.508418e+00)(5.600000e-01,2.486728e+00)(5.700000e-01,2.465295e+00)(5.800000e-01,2.444124e+00)(5.900000e-01,2.423219e+00)(6.000000e-01,2.402585e+00)(6.100000e-01,2.382227e+00)(6.200000e-01,2.362150e+00)(6.300000e-01,2.342351e+00)(6.400000e-01,2.322852e+00)(6.500000e-01,2.303650e+00)(6.600000e-01,2.284752e+00)(6.700000e-01,2.266164e+00)(6.800000e-01,2.247893e+00)(6.900000e-01,2.229944e+00)(7.000000e-01,2.212326e+00)(7.100000e-01,2.195045e+00)(7.200000e-01,2.178109e+00)(7.300000e-01,2.161527e+00)(7.400000e-01,2.145305e+00)(7.500000e-01,2.129454e+00)(7.600000e-01,2.113981e+00)(7.700000e-01,2.098895e+00)(7.800000e-01,2.084206e+00)(7.900000e-01,2.069923e+00)(8.000000e-01,2.056148e+00)(8.100000e-01,2.042710e+00)(8.200000e-01,2.029760e+00)(8.300000e-01,2.017277e+00)(8.400000e-01,2.005274e+00)(8.500000e-01,1.993768e+00)(8.600000e-01,1.982778e+00)(8.700000e-01,1.972320e+00)(8.800000e-01,1.962510e+00)(8.900000e-01,1.953183e+00)(9.000000e-01,1.944477e+00)(9.100000e-01,1.936384e+00)(9.200000e-01,1.929007e+00)(9.300000e-01,1.922283e+00)(9.400000e-01,1.916341e+00)(9.500000e-01,1.911142e+00)(9.600000e-01,1.906722e+00)(9.700000e-01,1.903201e+00)(9.800000e-01,1.900596e+00)(9.900000e-01,1.898929e+00)(1,1.898331e+00)(1.010000e+00,1.898891e+00)(1.020000e+00,1.900563e+00)(1.030000e+00,1.903555e+00)(1.040000e+00,1.908050e+00)(1.050000e+00,1.913945e+00)(1.060000e+00,1.921577e+00)(1.070000e+00,1.931046e+00)(1.080000e+00,1.942541e+00)(1.090000e+00,1.956221e+00)(1.100000e+00,1.972509e+00)(1.110000e+00,1.991644e+00)(1.120000e+00,2.014005e+00)(1.130000e+00,2.040104e+00)(1.140000e+00,2.070486e+00)(1.150000e+00,2.105979e+00)(1.160000e+00,2.147432e+00)(1.170000e+00,2.195752e+00)(1.180000e+00,2.251316e+00)(1.190000e+00,2.311544e+00)(1.200000e+00,2.368547e+00)(1.210000e+00,2.411555e+00)(1.220000e+00,2.436490e+00)(1.230000e+00,2.445887e+00)(1.240000e+00,2.443939e+00)(1.250000e+00,2.433843e+00)(1.260000e+00,2.418053e+00)(1.270000e+00,2.398217e+00)(1.280000e+00,2.375489e+00)(1.290000e+00,2.350795e+00)(1.300000e+00,2.324570e+00)(1.310000e+00,2.297347e+00)(1.320000e+00,2.269454e+00)(1.330000e+00,2.241148e+00)(1.340000e+00,2.212630e+00)(1.350000e+00,2.184057e+00)(1.360000e+00,2.155550e+00)(1.370000e+00,2.127206e+00)(1.380000e+00,2.099100e+00)(1.390000e+00,2.071196e+00)(1.400000e+00,2.043734e+00)(1.410000e+00,2.016652e+00)(1.420000e+00,1.989956e+00)(1.430000e+00,1.963686e+00)(1.440000e+00,1.937856e+00)(1.450000e+00,1.912478e+00)(1.460000e+00,1.887557e+00)(1.470000e+00,1.863098e+00)(1.480000e+00,1.839102e+00)(1.490000e+00,1.815568e+00)(1.500000e+00,1.792493e+00)(1.510000e+00,1.769873e+00)(1.520000e+00,1.747703e+00)(1.530000e+00,1.725977e+00)(1.540000e+00,1.704689e+00)(1.550000e+00,1.683830e+00)(1.560000e+00,1.663394e+00)(1.570000e+00,1.643385e+00)(1.580000e+00,1.623769e+00)(1.590000e+00,1.604550e+00)(1.600000e+00,1.585719e+00)(1.610000e+00,1.567351e+00)(1.620000e+00,1.549271e+00)(1.630000e+00,1.531552e+00)(1.640000e+00,1.514186e+00)(1.650000e+00,1.497169e+00)(1.660000e+00,1.480511e+00)(1.670000e+00,1.464181e+00)(1.680000e+00,1.448141e+00)(1.690000e+00,1.432499e+00)(1.700000e+00,1.417083e+00)(1.710000e+00,1.401993e+00)(1.720000e+00,1.387191e+00)(1.730000e+00,1.372668e+00)(1.740000e+00,1.358431e+00)(1.750000e+00,1.344463e+00)(1.760000e+00,1.330759e+00)(1.770000e+00,1.317312e+00)(1.780000e+00,1.304115e+00)(1.790000e+00,1.291162e+00)(1.800000e+00,1.278448e+00)(1.810000e+00,1.265965e+00)(1.820000e+00,1.253710e+00)(1.830000e+00,1.241675e+00)(1.840000e+00,1.229855e+00)(1.850000e+00,1.218246e+00)(1.860000e+00,1.206842e+00)(1.870000e+00,1.195639e+00)(1.880000e+00,1.184630e+00)(1.890000e+00,1.173812e+00)(1.900000e+00,1.163180e+00)(1.910000e+00,1.152730e+00)(1.920000e+00,1.142457e+00)(1.930000e+00,1.132357e+00)(1.940000e+00,1.122426e+00)(1.950000e+00,1.112661e+00)(1.960000e+00,1.103056e+00)(1.970000e+00,1.093609e+00)(1.980000e+00,1.084316e+00)(1.990000e+00,1.075173e+00)(2,1.066176e+00)(2.010000e+00,1.057324e+00)(2.020000e+00,1.048611e+00)(2.030000e+00,1.040036e+00)(2.040000e+00,1.031594e+00)(2.050000e+00,1.023284e+00)(2.060000e+00,1.015101e+00)(2.070000e+00,1.007044e+00)(2.080000e+00,9.991095e-01)(2.090000e+00,9.912949e-01)(2.100000e+00,9.835975e-01)(2.110000e+00,9.760280e-01)(2.120000e+00,9.685611e-01)(2.130000e+00,9.612037e-01)(2.140000e+00,9.539535e-01)(2.150000e+00,9.468082e-01)(2.160000e+00,9.397656e-01)(2.170000e+00,9.328162e-01)(2.180000e+00,9.259774e-01)(2.190000e+00,9.192332e-01)(2.200000e+00,9.125794e-01)(2.210000e+00,9.060188e-01)(2.220000e+00,8.995488e-01)(2.230000e+00,8.931675e-01)(2.240000e+00,8.868731e-01)(2.250000e+00,8.807265e-01)(2.260000e+00,8.746005e-01)(2.270000e+00,8.685564e-01)(2.280000e+00,8.625924e-01)(2.290000e+00,8.567153e-01)(2.300000e+00,8.509179e-01)(2.310000e+00,8.451968e-01)(2.320000e+00,8.395505e-01)(2.330000e+00,8.339776e-01)(2.340000e+00,8.284768e-01)(2.350000e+00,8.230466e-01)(2.360000e+00,8.176858e-01)(2.370000e+00,8.123931e-01)(2.380000e+00,8.071672e-01)(2.390000e+00,8.020069e-01)(2.400000e+00,7.969110e-01)(2.410000e+00,7.918782e-01)(2.420000e+00,7.869076e-01)(2.430000e+00,7.819980e-01)(2.440000e+00,7.771482e-01)(2.450000e+00,7.723572e-01)(2.460000e+00,7.676240e-01)(2.470000e+00,7.629475e-01)(2.480000e+00,7.583268e-01)(2.490000e+00,7.537609e-01)(2.500000e+00,7.492488e-01)(2.510000e+00,7.447897e-01)(2.520000e+00,7.403825e-01)(2.530000e+00,7.360265e-01)(2.540000e+00,7.317208e-01)(2.550000e+00,7.274644e-01)(2.560000e+00,7.232566e-01)(2.570000e+00,7.190966e-01)(2.580000e+00,7.149836e-01)(2.590000e+00,7.109167e-01)(2.600000e+00,7.068953e-01)(2.610000e+00,7.029186e-01)(2.620000e+00,6.989858e-01)(2.630000e+00,6.950964e-01)(2.640000e+00,6.912494e-01)(2.650000e+00,6.874444e-01)(2.660000e+00,6.836805e-01)(2.670000e+00,6.799022e-01)(2.680000e+00,6.762190e-01)(2.690000e+00,6.725751e-01)(2.700000e+00,6.689698e-01)(2.710000e+00,6.654026e-01)(2.720000e+00,6.618729e-01)(2.730000e+00,6.583801e-01)(2.740000e+00,6.549236e-01)(2.750000e+00,6.515029e-01)(2.760000e+00,6.481174e-01)(2.770000e+00,6.447666e-01)(2.780000e+00,6.414500e-01)(2.790000e+00,6.381671e-01)(2.800000e+00,6.349174e-01)(2.810000e+00,6.317001e-01)(2.820000e+00,6.285129e-01)(2.830000e+00,6.253575e-01)(2.840000e+00,6.222332e-01)(2.850000e+00,6.191397e-01)(2.860000e+00,6.160765e-01)(2.870000e+00,6.130432e-01)(2.880000e+00,6.100393e-01)(2.890000e+00,6.070644e-01)(2.900000e+00,6.041181e-01)(2.910000e+00,6.012000e-01)(2.920000e+00,5.983098e-01)(2.930000e+00,5.954469e-01)(2.940000e+00,5.926110e-01)(2.950000e+00,5.898018e-01)(2.960000e+00,5.870189e-01)(2.970000e+00,5.842618e-01)(2.980000e+00,5.815304e-01)(2.990000e+00,5.788241e-01)(3,5.761427e-01)(3.010000e+00,5.734859e-01)(3.020000e+00,5.708532e-01)(3.030000e+00,5.682444e-01)(3.040000e+00,5.656592e-01)(3.050000e+00,5.630972e-01)(3.060000e+00,5.605581e-01)(3.070000e+00,5.580417e-01)(3.080000e+00,5.555476e-01)(3.090000e+00,5.530840e-01)(3.100000e+00,5.506358e-01)(3.110000e+00,5.482091e-01)(3.120000e+00,5.458034e-01)(3.130000e+00,5.434187e-01)(3.140000e+00,5.410545e-01)(3.150000e+00,5.386987e-01)(3.160000e+00,5.363866e-01)(3.170000e+00,5.340464e-01)(3.180000e+00,5.317986e-01)(3.190000e+00,5.294851e-01)(3.200000e+00,5.272392e-01)(3.210000e+00,5.250122e-01)(3.220000e+00,5.228038e-01)(3.230000e+00,5.206138e-01)(3.240000e+00,5.184420e-01)(3.250000e+00,5.162881e-01)(3.260000e+00,5.141519e-01)(3.270000e+00,5.120331e-01)(3.280000e+00,5.099317e-01)(3.290000e+00,5.078473e-01)(3.300000e+00,5.057799e-01)(3.310000e+00,5.037290e-01)(3.320000e+00,5.016936e-01)(3.330000e+00,4.996739e-01)(3.340000e+00,4.976704e-01)(3.350000e+00,4.956828e-01)(3.360000e+00,4.937109e-01)(3.370000e+00,4.917545e-01)(3.380000e+00,4.898135e-01)(3.390000e+00,4.878877e-01)(3.400000e+00,4.859768e-01)(3.410000e+00,4.840808e-01)(3.420000e+00,4.821994e-01)(3.430000e+00,4.803326e-01)(3.440000e+00,4.784800e-01)(3.450000e+00,4.766416e-01)(3.460000e+00,4.748172e-01)(3.470000e+00,4.730067e-01)(3.480000e+00,4.712098e-01)(3.490000e+00,4.694265e-01)(3.500000e+00,4.676565e-01)(3.510000e+00,4.658998e-01)(3.520000e+00,4.641561e-01)(3.530000e+00,4.624254e-01)(3.540000e+00,4.607075e-01)(3.550000e+00,4.590023e-01)(3.560000e+00,4.573095e-01)(3.570000e+00,4.556292e-01)(3.580000e+00,4.539610e-01)(3.590000e+00,4.523050e-01)(3.600000e+00,4.506610e-01)(3.610000e+00,4.490288e-01)(3.620000e+00,4.474084e-01)(3.630000e+00,4.457995e-01)(3.640000e+00,4.442021e-01)(3.650000e+00,4.426161e-01)(3.660000e+00,4.410413e-01)(3.670000e+00,4.394777e-01)(3.680000e+00,4.379250e-01)(3.690000e+00,4.363832e-01)(3.700000e+00,4.348522e-01)(3.710000e+00,4.333319e-01)(3.720000e+00,4.318221e-01)(3.730000e+00,4.303227e-01)(3.740000e+00,4.288337e-01)(3.750000e+00,4.273549e-01)(3.760000e+00,4.258862e-01)(3.770000e+00,4.244275e-01)(3.780000e+00,4.229788e-01)(3.790000e+00,4.215399e-01)(3.800000e+00,4.201107e-01)(3.810000e+00,4.186911e-01)(3.820000e+00,4.172810e-01)(3.830000e+00,4.158804e-01)(3.840000e+00,4.144891e-01)(3.850000e+00,4.131070e-01)(3.860000e+00,4.117341e-01)(3.870000e+00,4.103703e-01)(3.880000e+00,4.090154e-01)(3.890000e+00,4.076694e-01)(3.900000e+00,4.063322e-01)(3.910000e+00,4.050037e-01)(3.920000e+00,4.036839e-01)(3.930000e+00,4.023726e-01)(3.940000e+00,4.010697e-01)(3.950000e+00,3.997753e-01)(3.960000e+00,3.984891e-01)(3.970000e+00,3.972112e-01)(3.980000e+00,3.959414e-01)(3.990000e+00,3.946796e-01)(4,3.934259e-01)
};\addlegendentry{$C=100$}

\addplot[color=green] coordinates{
(1.000000e-01,3.694321e+00)(2.000000e-01,3.401288e+00)(3.000000e-01,3.123069e+00)(4.000000e-01,2.861816e+00)(5.000000e-01,2.620571e+00)(6.000000e-01,2.402585e+00)(7.000000e-01,2.212326e+00)(8.000000e-01,2.056148e+00)(9.000000e-01,1.944473e+00)(1,1.898258e+00)(1.100000e+00,1.958317e+00)(1.200000e+00,1.992587e+00)(1.300000e+00,1.861339e+00)(1.400000e+00,1.691567e+00)(1.500000e+00,1.529934e+00)(1.600000e+00,1.386983e+00)(1.700000e+00,1.263477e+00)(1.800000e+00,1.157246e+00)(1.900000e+00,1.065710e+00)(2,9.864484e-01)(2.100000e+00,9.174156e-01)(2.200000e+00,8.569059e-01)(2.300000e+00,8.035172e-01)(2.400000e+00,7.561391e-01)(2.500000e+00,7.138523e-01)(2.600000e+00,6.759094e-01)(2.700000e+00,6.416429e-01)(2.800000e+00,6.106526e-01)(2.900000e+00,5.824405e-01)(3,5.566708e-01)(3.100000e+00,5.330243e-01)(3.200000e+00,5.113068e-01)(3.300000e+00,4.912605e-01)(3.400000e+00,4.726988e-01)(3.500000e+00,4.554717e-01)(3.600000e+00,4.394420e-01)(3.700000e+00,4.244905e-01)(3.800000e+00,4.105132e-01)(3.900000e+00,3.974187e-01)(4,3.851267e-01)
};\addlegendentry{$C=50$}


\addplot[color=teal] coordinates{
(1.000000e-01,3.937286e+00)(2.000000e-01,3.875356e+00)(3.000000e-01,3.814081e+00)(4.000000e-01,3.753671e+00)(5.000000e-01,3.693989e+00)(6.000000e-01,3.635063e+00)(7.000000e-01,3.576893e+00)(8.000000e-01,3.519477e+00)(9.000000e-01,3.462813e+00)(1,3.406899e+00)(1.100000e+00,3.351733e+00)(1.200000e+00,3.297312e+00)(1.300000e+00,3.243635e+00)(1.400000e+00,3.190698e+00)(1.500000e+00,3.138500e+00)(1.600000e+00,3.087036e+00)(1.700000e+00,3.036304e+00)(1.800000e+00,2.986300e+00)(1.900000e+00,2.937136e+00)(2,2.888578e+00)(2.100000e+00,2.840738e+00)(2.200000e+00,2.793624e+00)(2.300000e+00,2.747231e+00)(2.400000e+00,2.701546e+00)(2.500000e+00,2.656566e+00)(2.600000e+00,2.612285e+00)(2.700000e+00,2.568699e+00)(2.800000e+00,2.525804e+00)(2.900000e+00,2.483594e+00)(3,2.442065e+00)(3.100000e+00,2.401212e+00)(3.200000e+00,2.361029e+00)(3.300000e+00,2.321510e+00)(3.400000e+00,2.282550e+00)(3.500000e+00,2.244345e+00)(3.600000e+00,2.206788e+00)(3.700000e+00,2.169872e+00)(3.800000e+00,2.133592e+00)(3.900000e+00,2.097926e+00)(4,2.062879e+00)
};\addlegendentry{$C=0.5$}

\addplot[color=red] coordinates{
(1.000000e-01,3.693609e+00)(2.000000e-01,3.400277e+00)(3.000000e-01,3.120410e+00)(4.000000e-01,2.856036e+00)(5.000000e-01,2.608633e+00)(6.000000e-01,2.379387e+00)(7.000000e-01,2.169056e+00)(8.000000e-01,1.977718e+00)(9.000000e-01,1.805229e+00)(1,1.650685e+00)(1.100000e+00,1.512970e+00)(1.200000e+00,1.390639e+00)(1.300000e+00,1.282222e+00)(1.400000e+00,1.186137e+00)(1.500000e+00,1.100905e+00)(1.600000e+00,1.025198e+00)(1.700000e+00,9.577855e-01)(1.800000e+00,8.975892e-01)(1.900000e+00,8.437265e-01)(2,7.952600e-01)(2.100000e+00,7.515694e-01)(2.200000e+00,7.120411e-01)(2.300000e+00,6.761559e-01)(2.400000e+00,6.434711e-01)(2.500000e+00,6.136072e-01)(2.600000e+00,5.861874e-01)(2.700000e+00,5.610312e-01)(2.800000e+00,5.378489e-01)(2.900000e+00,5.164123e-01)(3,4.965451e-01)(3.100000e+00,4.780926e-01)(3.200000e+00,4.609126e-01)(3.300000e+00,4.448833e-01)(3.400000e+00,4.298958e-01)(3.500000e+00,4.158438e-01)(3.600000e+00,4.026967e-01)(3.700000e+00,3.902960e-01)(3.800000e+00,3.786167e-01)(3.900000e+00,3.676173e-01)(4,3.572279e-01)
};\addlegendentry{Optimal $C$}

  \addplot[color=black,mark =none,dotted,line width=1.2] coordinates{(1.2534,0)(1.2534,1)(1.2534,4)};
\node at (axis cs:1.15,0.26)(source1){${\delta_\star}$}; 
           \end{axis}
         \end{tikzpicture}}
\caption{Risk of S-SVR vs. $\delta$ for different values of $C$ when $\epsilon=0.6$ and $\sigma=1$. Illustration of the sample-wise double decent and how optimal regularization mitigates it.}
\label{effect_of_reg}
\end{center}
\end{figure}

\subsubsection{Sample-wise descent phenomenon}
In \figref{effect_of_reg}, we plot the test risk with respect to $\delta$ for the S-SVR for different choices of $C$ and fixed $\epsilon$. As can be seen, for $\delta$ tending to zero, the S-SVR test risk converges to that of the null risk, which was also predicted by our analysis in {\bf Remark 6}. Moreover, depending on the value of $C$, the test risk may manifest a cusp at $\delta\sim \delta_\star$ that becomes more pronounced as $C$ increases. This can be explained by the fact that as $C$ increases, the behavior of S-SVR approaches that of H-SVR, for which the problem becomes unfeasible when $\delta >\delta_\star$.  We also note that the choice of the parameter $C$ plays a fundamental role in the test risk performance. Optimal values of $C$ always guarantee that the test risk performance decreases with more training samples being used, while arbitrary choices can lead to the test risk increasing for more training samples. This emphasizes the importance of the appropriate selection of the parameter $C$ to avoid the double descent phenomenon.

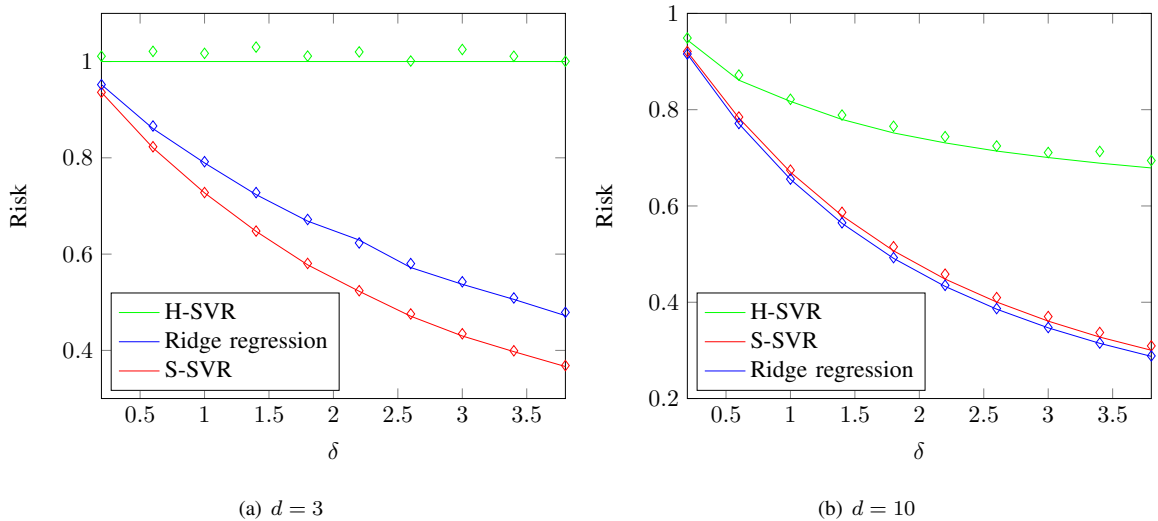
\begin{figure}[]
\begin{center}
         \subfigure[$d=3$]{
\begin{tikzpicture}[scale=0.9,font=\fontsize{10}{10}\selectfont]
   \tikzstyle{every axis y label}+=[yshift=0pt]
   \tikzstyle{every axis x label}+=[yshift=5pt]
   \tikzstyle{every axis legend}+=[cells={anchor=west},fill=white,
        at={(0.02,0.02)}, anchor=south west, font=\fontsize{10}{10}\selectfont]
   \begin{axis}[
      xmin=0.2,
      ymin=0.3,
      xmax=3.8,
      ymax=1.1,
xlabel={$\delta$},
ylabel={Risk }	]		

 \addplot[color=green] coordinates{
(2.000000e-01,1)(6.000000e-011)(1,1)(1.400000e+00,1)(1.800000e+00,1)(2.200000e+00,1)(2.600000e+00,1)(3,1)(3.400000e+00,1)(3.800000e+00,1)};\addlegendentry{H-SVR}   
   \addplot[color=blue] coordinates{
(2.000000e-01,9.520762e-01)(6.000000e-01,8.606062e-01)(1,7.893686e-01)(1.400000e+00,7.236206e-01)(1.800000e+00,6.685102e-01)(2.200000e+00,6.291970e-01)(2.600000e+00,5.720672e-01)(3,5.372788e-01)(3.400000e+00,5.059387e-01)(3.800000e+00,4.722029e-01)};\addlegendentry{Ridge regression}   
\addplot[color=red] coordinates{
(2.000000e-01,9.359955e-01)(6.000000e-01,8.206216e-01)(1,7.271578e-01)(1.400000e+00,6.472629e-01)(1.800000e+00,5.777095e-01)(2.200000e+00,5.227607e-01)(2.600000e+00,4.709949e-01)(3,4.299647e-01)(3.400000e+00,3.973753e-01)(3.800000e+00,3.664450e-01)

};\addlegendentry{S-SVR}

  \addplot[color=red,mark=diamond,mark size=2.2pt,only marks] plot coordinates{
  (2.000000e-01,9.360167e-01)(6.000000e-01,8.229980e-01)(1,7.278447e-01)(1.400000e+00,6.478332e-01)(1.800000e+00,5.806656e-01)(2.200000e+00,5.239564e-01)(2.600000e+00,4.757013e-01)(3,4.344863e-01)(3.400000e+00,3.991929e-01)(3.800000e+00,3.684876e-01)

};

  \addplot[color=blue,mark=diamond,mark size=2.2pt,only marks] plot coordinates{(2.000000e-01,9.519984e-01)(6.000000e-01,8.660899e-01)(1,7.919473e-01)(1.400000e+00,7.277469e-01)(1.800000e+00,6.719188e-01)(2.200000e+00,6.231370e-01)(2.600000e+00,5.802947e-01)(3,5.424738e-01)(3.400000e+00,5.089148e-01)(3.800000e+00,4.789894e-01)
};
 \addplot[color=green,mark=diamond,mark size=2.2pt,only marks] plot coordinates{(2.000000e-01,1
.0110)(6.000000e-01,1.02111)(1,1.017)(1.400000e+00,1.03)(1.800000e+00,1.0113)(2.200000e+00,1.02)(2.600000e+00,1.001)(3,1.025)(3.400000e+00,1.011)(3.800000e+00,1.0005)
};
\end{axis}
\end{tikzpicture}}
         \subfigure[$d=10$]{
\begin{tikzpicture}[scale=0.9,font=\fontsize{10}{10}\selectfont]
   \tikzstyle{every axis y label}+=[yshift=0pt]
   \tikzstyle{every axis x label}+=[yshift=5pt]
   \tikzstyle{every axis legend}+=[cells={anchor=west},fill=white,
        at={(0.02,0.02)}, anchor=south west, font=\fontsize{10}{10}\selectfont]
   \begin{axis}[
      xmin=0.2,
      ymin=0.2,
      xmax=3.8,
      ymax=1,
xlabel={$\delta$},
ylabel={Risk }	]		

  \addplot[color=green] coordinates{
(2.000000e-01,9.447840e-01)(6.000000e-01,8.611840e-01)  (1,8.172160e-01)(1.400000e+00,7.796890e-01)(1.800000e+00,7.516890e-01)(2.200000e+00,7.310250e-01)(2.600000e+00,7.140250e-01)(3,7.005690e-01)(3.400000e+00,6.889000e-01)(3.800000e+00,6.789760e-01)

};\addlegendentry{H-SVR}
\addplot[color=red] coordinates{(2.000000e-01,9.202934e-01)(6.000000e-01,7.826746e-01)(1,6.701915e-01)(1.400000e+00,5.797876e-01)(1.800000e+00,5.068961e-01)(2.200000e+00,4.481193e-01)(2.600000e+00,4.004840e-01)(3,3.608087e-01)(3.400000e+00,3.276503e-01)(3.800000e+00,3.002828e-01)
};\addlegendentry{S-SVR}

   \addplot[color=blue] coordinates{(2.000000e-01,9.153162e-01)(6.000000e-01,7.707434e-01)(1,6.553860e-01)(1.400000e+00,5.639061e-01)(1.800000e+00,4.911791e-01)(2.200000e+00,4.328872e-01)(2.600000e+00,3.856409e-01)(3,3.468725e-01)(3.400000e+00,3.146648e-01)(3.800000e+00,2.875891e-01)

};\addlegendentry{Ridge regression} 
  \addplot[color=green,mark=diamond,mark size=2.2pt,only marks] plot coordinates{
(2.000000e-01,9.488500e-01)(6.000000e-01,8.719813e-01)(1,8.215785e-01)(1.400000e+00,7.887948e-01)(1.800000e+00,7.655552e-01)(2.200000e+00,7.439050e-01)(2.600000e+00,7.246347e-01)(3,7.109726e-01)(3.400000e+00,7.132233e-01)(3.800000e+00,6.943865e-01)

};

  \addplot[color=red,mark=diamond,mark size=2.2pt,only marks] plot coordinates{(2.000000e-01,9.208593e-01)(6.000000e-01,7.847026e-01)(1,6.743737e-01)(1.400000e+00,5.867644e-01)(1.800000e+00,5.152613e-01)(2.200000e+00,4.583352e-01)(2.600000e+00,4.096577e-01)(3,3.702781e-01)(3.400000e+00,3.371762e-01)(3.800000e+00,3.093228e-01)

};
 \addplot[color=blue,mark=diamond,mark size=2.2pt,only marks] plot coordinates{(2.000000e-01,9.156435e-01)(6.000000e-01,7.713544e-01)(1,6.557908e-01)(1.400000e+00,5.650104e-01)(1.800000e+00,4.926866e-01)(2.200000e+00,4.351770e-01)(2.600000e+00,3.868783e-01)(3,3.475646e-01)(3.400000e+00,3.152278e-01)(3.800000e+00,2.885432e-01)

};
\end{axis}
\end{tikzpicture}}
\caption{Risk of S-SVR and ridge regression vs. $\delta$ when optimal regularization is used. Comparison of S-SVR and ridge regression with elliptic noise. The continuous line curves correspond to the asymptotic performance while the points denote finite-sample performance when $p=200$ and $n= \lfloor{\delta p}\rfloor$.}
\label{ellipticNoise}
\end{center}
\end{figure}

\subsection{Comparison with ridge regression estimators under impulsive noises:}
In this experiment, we investigate S-SVR  and H-SVR's resilience when optimally designed (optimal $\epsilon$ for H-SVR and optimal $C$ and $\epsilon$ for S-SVR) to impulsive noises and compare them to the ridge regression with optimal regularization. Particularly, we consider the case in which the noise is sampled from the distributional model:
$$
n=\sqrt{\tau} \mathcal{N}(0,1)
$$
where $\tau$ follows an inverse Gamma distribution  with shape $\frac{d}{2}$ and scale $\frac{2}{d}$, that is 
$$
\tau\sim \frac{d}{\chi_d^2}
$$
with $\chi$ being the chi-square distribution with $d$ degrees of freedom.

\figref{ellipticNoise} represents the test risk performance of the aforementioned estimators for $d=3$ and $d=10$. As can be seen, H-SVR is very sensitive to impulsive noises with a performance approaching that of the null estimator for highly impulsive noises. This behavior can be explained by the fact that in highly impulsive noises (small $d$), H-SVR needs a very large $\epsilon$ to guarantee that outliers satisfy the feasibility conditions. However, with a large $\epsilon$, the constraints in \eqref{eq:hard_margin} becomes irrelevant for the remaining well-behaved observations. This favors the H-SVR to select the null estimator as it would minimize the objective in \eqref{eq:hard_margin} while satisfying the constraints. On the other hand, the S-SVR overcomes such behavior since it does not have to satisfy the constraints of \eqref{eq:hard_margin}. By selecting the parameter $C$ to the value that minimizes the test risk, it will adaptively control the effect of outliers by relaxing the most unlikely constraints in \eqref{eq:hard_margin}. Moreover, it can be seen that, although S-SVR is slightly less efficient than the ridge regression estimator in mild impulsive noises, it presents a much lower risk under highly impulsive noises. 

While the robustness of regularized support vector machine methods for both regression and classification is a well-known fact reported in many previous works \cite{xu2009robustness,hable2011qualitative}. Our result contributes to quantitatively assess such robustness by measuring the test risk under impulsive noises.

\section{Conclusion}
\label{conclusion}
In this paper,  we studied the asymptotic test risk of hard and soft support vector regression techniques with isotropic Gaussian features and under symmetric noise distributions in the regime of high dimensions. We used these results to illustrate the impact of the intervening parameters on the test risk behavior. Particularly, we demonstrate that arbitrary choices of the parameters of the hard SVR and the soft SVR may lead to the test risk presenting a non-monotonic behavior as a function of the number of samples, which illustrates the fact that adding more samples may be harmful to the performance. On the contrary, we show that optimally-tuned hard SVR and soft SVR present a decreasing test risk curve, which shows the importance of carefully selecting their parameters to minimize the test risk and guarantee the positive impact of more data on the test risk performance. Our findings are consistent with similar results obtained for linear regression and neural networks\cite{abs-2003-01897}. However, as compared to linear regression, we demonstrate that soft support vector regression with optimal regularization is more robust to the presence of outliers, corroborating similar previous findings in earlier works in \cite{xu2009robustness,hable2011qualitative}.

Several extensions of our work are worth investigating. One important research direction is to understand the effect of correlated features on the test risk of hard and support regression techniques.   Some recent works have investigated the role of correlation and regularization on the test risk for linear regression models \cite{Lolas2020,Kobak2020TheOR}. A significant advantage of such an analysis is that it can illustrate the importance of investing efforts in theoretically-aided approaches to assist in setting the regularization parameters. Another important research direction is investigating the use of kernel support vector regression methods and understanding their underlying mechanisms to handle involved non-linear data models.

\appendices
\input{cgmt_review}

\input{proof_th_1}

\input{Proof_Soft}
\input{technical_lemmata}

\bibliographystyle{IEEEtran}
\bibliography{references}

\end{document}

%% file: cgmt_review.tex
\section{CGMT framework}
\label{sec:cgmt_review}
The proofs of our results are based on the extension of the CGMT framework, recently proposed in \cite{kam-chris}, that allows for handling optimization problems in which the  optimization sets may not be compact  or are not necessarily feasible. For the reader convenience, we provide a review of the main results \cite{kam-chris} that will be extensively used in our proofs.

In order to summarize the essential ideas, we consider the following two Gaussian processes:
\begin{align}
X_{{\bf w},{\bf u}}&:={\bf u}^{T}{\bf G}{\bf w}+\psi({\bf w},{\bf u}) \label{eq:Xwu}\\
Y_{{\bf w},{\bf u}}&:= \|{\bf w}\|_2{\bf g}^{T}{\bf u} + \|{\bf u}\|_2{\bf h}^{T}{\bf w} +\psi({\bf w},{\bf u})
\end{align}      
\noindent where ${\bf G}\in\mathbb{R}^{n\times d}$, ${\bf g}\in\mathbb{R}^n$, ${\bf h}\in\mathbb{R}^d$ have standard normal i.i.d. entries  and $\psi:\mathbb{R}^d\times \mathbb{R}^n\to \mathbb{R}$ is a continuous  function. Let $\mathcal{S}_{{\bf w}}\subset\mathbb{R}^d$ and $\mathcal{S}_{\bf u} \subset \mathbb{R}^n$ be two convex sets but not necessarily compact.
We consider the following random min-max optimization problems which refer to as the primary optimization (PO) problem and the auxiliary optimization problem (AO):
\begin{align}
{\rm (PO)} \ \ \  \Phi({\bf G})&=\min_{{\bf w}\in\mathcal{S}_{{\bf w}}} \max_{{\bf u}\in\mathcal{S}_{{\bf u}}} X_{{\bf w},{\bf u}}  \label{eq:PO}\\ 
{\rm (AO) } \ \ \ \phi({\bf g},{\bf h})&=\min_{{\bf w}\in\mathcal{S}_{{\bf w}}}  \max_{{\bf u}\in\mathcal{S}_{{\bf u}}} Y_{{\bf w},{\bf u}}  \label{eq:AO}
\end{align}
The direct analysis of the (PO) is in general  difficult. The power of the CGMT is that it transfers the asymptotic behavior of the PO  to that of the AO. Particularly, if $\mathcal{S}_{\bf w}$ and $\mathcal{S}_{\bf u}$ are  compact sets, the use of Gordon's inequality \cite{Gor88,gordon85} implies that:
\begin{equation}
\mathbb{P}\left(\Phi({\bf G})<c\right)\leq 2\mathbb{P}\left(\phi({\bf g},{\bf h})<c\right)
\label{eq:gordon}
\end{equation}
The only conditions required by the above inequality is that the sets are compact and $\psi$ is continuous. Particularly, Gordon's inequality holds even if the optimization sets are not convex and $\psi$ is not convex-concave. As a by product of \eqref{eq:gordon}, it follows that if $c$ is a high-probability lower bound of the AO than it is also a lower bound of the PO.  
If the set $\mathcal{S}_{\bf w}$ and $\mathcal{S}_{\bf u}$ are additionally convex and $\psi$ is convex-concave then, for any $\nu\in\mathbb{R}$ and $t>0$, it holds
$$
\mathbb{P}\left(|\Phi({\bf G})-\nu|>t\right) \leq 2\mathbb{P}(|\phi({\bf g},{\bf h})-\nu|>t)
$$
In other words, if the AO cost concentrates around $\nu^\star$ then the optimal cost of the PO concentrates also around the same value $\nu_\star$. Moreover, it has been shown in \cite{thrampoulidis-IT} that under strict convexity conditions of the asymptotic AO, the concentration of the optimal solution of the AO translates into the concentration of the optimal solution of the PO around the same value. More precisely, denote by ${\bf w}_{\phi}({\bf g},{\bf h})$ the optimal solution of the AO. If one can prove that the minimizers of the AO satisfy $\|{\bf w}_{\phi}({\bf g},{\bf h})\|\asto \alpha_\star$ where $\alpha_\star$ is the unique solution of a certain limiting  optimization problem whose cost is asymptotically equivalent to the AO. Then, the solution of the PO, denoted by ${\bf w}_{\Phi}({\bf G})$ also satisfies $\|{\bf w}_\Phi({\bf G})\|\asto\alpha_\star$. 

In \cite{kam-chris}, the authors developed a principled machinery that extends the results of the CGMT to problems in which the optimization sets are not compact and may not be necessarily feasible. Before reviewing the results of \cite{kam-chris}, we shall introduce some important notations. For fixed $R$ and $\Gamma$, we consider the following ``$(R,\Gamma)$-bounded'' version of the PO:
$$
\Phi_{R,\Gamma}({\bf G})=\min_{\substack{{\bf w}\in\mathcal{S}_{{\bf w}}\\ \|{\bf w}\|_2\leq R}} \max_{\substack{{\bf u}\in\mathcal{S}_{{\bf u}}\\ \|{\bf u}\|_2\leq \Gamma}} X_{{\bf w},{\bf u}}
$$  
with which we associate the following $(R,\Gamma)$ bounded version of the AO:
\begin{equation}
\phi_{R,\Gamma}({\bf g},{\bf h})=\min_{\substack{{\bf w}\in\mathcal{S}_{{\bf w}}\\ \|{\bf w}\|_2\leq R}} \max_{\substack{{\bf u}\in\mathcal{S}_{{\bf u}}\\ \|{\bf u}\|_2\leq \Gamma}} Y_{{\bf w},{\bf u}}
\label{eq:phi_RGammagh}
\end{equation}
Since $\Gamma\mapsto \max_{\substack{{\bf u}\in\mathcal{S}_{\bf u}\\ \|{\bf u}\|_2\leq \Gamma}} X_{{\bf w},{\bf u}}$ is concave in $\Gamma$, we can establish using Sion's min-max theorem \cite{Sion58} that 
$$
\Phi({\bf G})=\inf_{R\geq 0} \sup_{\Gamma\geq 0}\Phi_{R,\Gamma}
$$
Similarly, we define $\phi_R({\bf g},{\bf h})$ as  follows:
\begin{equation}
\phi_R(g,h):=\sup_{\Gamma\geq 0} \phi_{R,\Gamma}({\bf g},{\bf h}) \label{eq:phi_Rgh}
\end{equation}
However, since $Y_{{\bf w},{\bf u}}$ is not convex-concave in $({\bf w},{\bf u})$, we only have:
$$
\sup_{\Gamma\geq 0}\phi_{R,\Gamma}({\bf g},{\bf h}) \leq \min_{\substack{{\bf w}\in\mathcal{S}_{{\bf w}}\\ \|{\bf w}\|_2\leq R}} \max_{{\bf u}\in\mathcal{S}_{\bf u}} Y_{{\bf w},{\bf u}}
$$
The extension of the CGMT allows us to connect properties of the unbounded PO to those of the sequence of AO problems $\phi_{R}({\bf g},{\bf h})$. Theorem \ref{th:feas} illustrates how the study of the cost of the sequence of AO problems can help determine the feasibility of the original PO problem, whereas Theorem \ref{th:conv} shows how the asymptotic behavior of the optimal cost of the original problem boils down to analyzing that of the sequence of AO problems. 
\begin{theorem}[Feasibility \cite{kam-chris}]
\label{th:feas}
Recall the definitions of $\Phi({\bf G}), \phi_{R,\Gamma}({\bf g},{\bf h})$ and $\phi_R({\bf g},{\bf h})$ in \eqref{eq:AO}, \eqref{eq:phi_RGammagh} and \eqref{eq:phi_Rgh}, respectively. Assume that $({\bf w},{\bf u})\mapsto X_{{\bf w},{\bf u}}$ in \eqref{eq:Xwu} is convex-concave and that the constraint sets $\mathcal{S}_{{\bf w}},\mathcal{S}_{{\bf u}}$ are convex (but not necessarily bounded). The following two statements hold true.
	
	\begin{enumerate}
		\item[(i)] Assume that { for any fixed $R,\Gamma\geq 0$} there exists a  positive constant $C$ (independent of $R$ and $\Gamma$) and a continuous increasing function $f:\mathbb{R}_+\rightarrow\mathbb{R}$ tending to infinity such that, for $n$ sufficiently large (independent of $R$ and $\Gamma$):
			\begin{equation}
				\phi_{R,\Gamma}({\bf g},{\bf h})\geq C\,f(\Gamma).
				\label{eq:cond_inf}
			\end{equation}
	Then, with probability $1$, for $n$ sufficiently large,
	$			\Phi({\bf G})=\infty$. 
\item[(ii)] Assume that there exists $k_0\in\mathbb{N}$ and a positive constant $C$ such that:
	\begin{equation}
		\mathbb{P}\left[ \left\{\phi_{k_0}({\bf g},{\bf h})\geq C\right\}, \ \text{i.o.}\right]=0.
			\label{eq:bnd_cond}
	\end{equation}
			Then, 
$			\mathbb{P}\left[\Phi({\bf G})=\infty, \ \text{i.o}\right]=0.$
	\end{enumerate}
\end{theorem}
\begin{theorem}
	Assume the same notation as in Theorem \ref{th:feas}.	Assume that there exists $\overline{\phi}$ and $k_0\in\mathbb{N}$ such that  the following statements hold true:
 \begin{subequations}\label{eq:conditions}
\begin{align}
	&\text{For any}~ \epsilon>0:\quad\mathbb{P}\left[\cup_{k=k_0}^\infty\Big\{ \phi_{k}({\bf g},{\bf h}) \leq \overline{\phi}-\epsilon\Big\}, ~~\text{i.o.}\right]=0,\label{eq:condition1}\\
	&\text{For any}~ \epsilon>0:\quad\mathbb{P}\left[\Big\{\phi_{k_0}({\bf g},{\bf h}) \geq \overline{\phi}+\epsilon\Big\}, ~~\text{i.o.}\right]=0\label{eq:condition2},\\
\end{align}
\end{subequations}
Then,
\begin{equation}
	\Phi({\bf G})\asto  \overline{\phi}.
		\label{eq:almost_sure}
	\end{equation}

 Further, 
	let $\mathcal{S}$ be an  open subset of $\mathcal{S}_{\bf w}$ and $\mathcal{S}^{c}=\mathcal{S}_{\bf w}\backslash \mathcal{S}$. Denote by $\widetilde{\phi}_{R,\Gamma}({\bf g},{\bf h})$ the optimal cost of  \eqref{eq:phi_RGammagh} when the minimization over ${\bf w}$ is now further constrained over ${\bf w}\in\mathcal{S}^c$. Define $\widetilde{\phi}_{R}({\bf g},{\bf h})$ in a similar way to \eqref{eq:phi_Rgh}.
Assume the following statement hold true, 
\begin{equation}
\text{There exists}~ \zeta>0:\quad\mathbb{P}\left[\cup_{k=k_0}^\infty\Big\{ \widetilde\phi_{k}({\bf g},{\bf h}) \leq \overline{\phi}+\zeta\Big\}, ~~\text{i.o.}\right]=0.\label{eq:condition3}
\end{equation}
	Then, letting ${\bf w}_\Phi$ denote a minimizer of the PO in \eqref{eq:PO}, it also holds that
	\begin{equation}
	\mathbb{P}\left[{\bf w}_{\Phi}\in \mathcal{S}, \ \ \text{for sufficiently large} \  n\right]=1.
		\label{eq:property}
	\end{equation}
\label{th:conv}
\end{theorem}

%% file: proof_th_1.tex
\section{Proof of Theorem 1 and Theorem 2}
In this section, we analyze the statistical behavior of hard-SVR. 
\subsection{Preliminaries}
\noindent{\bf Identification of the PO}
The Lagrangian function associated with H-SVR is given by:
$$
\mathcal{L}({\bf w},\boldsymbol{\lambda},\boldsymbol{\alpha}):=\frac{1}{2}\|{\bf w}\|^2+\sum_{i=1}^n \lambda_i(\boldsymbol{\beta}_\star^{T}{\bf x}_i+\sigma n_i-{\bf w}^{T}{\bf x}_i-\epsilon)+\sum_{i=1}^n \alpha_i({\bf w}^{T}{\bf x}_i+\boldsymbol{\beta}_\star^{T}{\bf x}_i-\sigma n_i-\epsilon)
$$
where $\boldsymbol{\lambda}=[\lambda_1,\cdots,\lambda_n]$ and $\boldsymbol{\alpha}=[\alpha_1,\cdots,\alpha_n]$ are the lagrange multipliers.  
From the first order optimality conditions, we have:
\begin{align}
&\frac{\partial \mathcal{L}}{\partial {\bf w}} = 0 \  \ \Longrightarrow {\bf w}=\sum_{i=1}^n (\lambda_i-\alpha_i){\bf x}_i \label{eq:lagrange}
\end{align}
Leveraging the above relations, the optimization problem becomes equivalent to solving:
$$
\max_{\substack{\boldsymbol{\lambda}\geq 0,\boldsymbol{\alpha}\geq 0 }} -\frac{1}{2}(\boldsymbol{\lambda}-\boldsymbol{\alpha})^{T}{\bf X}^{T}{\bf X}(\boldsymbol{\lambda}-\boldsymbol{\alpha}) + \sigma {\bf n}^{T}(\boldsymbol{\lambda}-\boldsymbol{\alpha}) + \boldsymbol{\beta}_\star^{T}{\bf X}(\boldsymbol{\lambda}-\boldsymbol{\alpha}) -\epsilon {\bf 1}^{T}(\boldsymbol{\lambda}+\boldsymbol{\alpha})
$$
where ${\bf X}=\left[{\bf x}_1,\cdots,{\bf x}_n\right]$.
Performing the change of variables ${\bf u}=\sqrt{p}(\boldsymbol{\lambda}-\boldsymbol{\alpha})$ and ${\bf v}=\sqrt{p}(\boldsymbol{\lambda}+\boldsymbol{\alpha})$, the above problem simplifies to:
\begin{equation}
\max_{\substack{{\bf u},{\bf v}\\ |{\bf u}|\leq {\bf v}\\ {\bf v}\geq 0}} -\frac{1}{2p} {\bf u}^{T}{\bf X}^{T}{\bf X}{\bf u} + \frac{1}{\sqrt{p}} \sigma {\bf n}^{T}{\bf u} + \frac{1}{\sqrt{p}} \boldsymbol{\beta}_\star^{T}{\bf X}{\bf u}  - \frac{1}{\sqrt{p}} \epsilon {\bf 1}^{T}{\bf v}
\label{eq:sim}
\end{equation}
From \eqref{eq:lagrange},  the solution of the H-SVR is related to an optimal solution $\hat{\bf u}$ of \eqref{eq:sim} as:
\begin{align}
\hat{\bf w}_H=\frac{1}{\sqrt{p}}{\bf X}\hat{\bf u}
\label{eq:wu}
\end{align}
Obviously, for a given ${\bf u}$ the optimum over ${\bf v}$ of the objective cost in \eqref{eq:sim} is given by  ${\bf v}=|{\bf u}|$. Replacing ${\bf v}$ by its optimal value, we obtain
\begin{equation}
\max_{\substack{{\bf u}}} -\frac{1}{2p} {\bf u}^{T}{\bf X}^{T}{\bf X}{\bf u} + \frac{1}{\sqrt{p}} \sigma {\bf n}^{T}{\bf u} + \frac{1}{\sqrt{p}} \boldsymbol{\beta}_\star^{T}{\bf X}{\bf u}  - \frac{1}{\sqrt{p}} \epsilon {\bf 1}^{T}|{\bf u}|\label{eq:u}
\end{equation}
To write the above problem in the form required by the CGMT, we shall use the following relation  \begin{equation}-\frac{1}{2p} {\bf u}^{T}{\bf X}^{T}{\bf X}{\bf u} +\frac{1}{\sqrt{p}}\boldsymbol{\beta}_\star^{T}{\bf Xu}=\min_{\tilde{\bf w}} \frac{1}{\sqrt{p}} \tilde{\bf w}^{T}{\bf X}{\bf u}+\frac{1}{2}\|\tilde{\bf w}\|^2-\boldsymbol{\beta}_\star^{T}\tilde{\bf w}+\frac{1}{2}\|\boldsymbol{\beta}_\star\|^2, \label{eq:rel}
\end{equation}
where in the right-hand side term, the optimal $\hat{{\tilde{\bf w}}}_H$ can be easily verified to be unique and given by:
\begin{equation}
\hat{\tilde{{\bf w}}}_H=-\frac{1}{\sqrt{p}}{\bf X}\hat{{\bf u}}+{\boldsymbol{\beta}}_\star
\label{eq:tilde_w}
\end{equation}
Plugging \eqref{eq:rel} into \eqref{eq:sim}, we identify the following problem which takes the form of a primary optimization problem as required by the CGTM framework:
\begin{align}
\Phi^{(n)}
&= \max_{\substack{\bu}}\min_{\tilde{\bw}} \frac{1}{\sqrt{p}}\bu^T\bX\tilde{\bw}+\frac{1}{2}\|\tilde{\bw}\|^2-\boldsymbol{\beta}_\star^{T}\tilde{\bf w} +\sigma \frac{1}{\sqrt{p}}\bn^T\bu-\frac{1}{\sqrt{p}}\epsilon \bone^T|\bu|+\frac{1}{2}\|\boldsymbol{\beta}_\star\|^2 \label{eq:primary_bis}
\end{align}
From \eqref{eq:wu} and \eqref{eq:tilde_w} and due to the uniqueness of $\hat{{\tilde{{\bf w}}}}_H$ in \eqref{eq:rel},  the  H-SVR solution  satisfies:
\begin{equation}
\hat{{\tilde{{\bf w}}}}_H=-\hat{\bf w}_H+\boldsymbol{\beta}_\star \label{eq:relation}
\end{equation}
Let ${\bf P}_\perp={\bf I}_p-\frac{\boldsymbol{\beta}_\star\boldsymbol{\beta}_\star^{T}}{\boldsymbol{\beta}_\star^{T}\boldsymbol{\beta}_\star}$. Then, we can decompose ${\bf u}^{T}{\bf X}^{T}\tilde{\bf w}$ as:
$$
{\bf u}^{T}{\bf X}^{T}\tilde{\bf w}={\bf u}^{T}{\bf X}^{T}{\bf P}_{\perp}\tilde{\bf w}+\frac{1}{\|\boldsymbol{\beta}_\star\|^2}{\bf u}^{T}{\bf X}^{T}\boldsymbol{\beta}_\star\boldsymbol{\beta}_\star^{T}\tilde{\bf w}
$$
Plugging this decomposition into \eqref{eq:primary_bis}, we obtain:
$$
\Phi^{(n)}=\max_{\substack{{\bf u} }} \min_{\tilde{\bf w}} \ \ \frac{1}{\sqrt{p}} {\bf u}^{T}{\bf X}^{T}{\bf P}_{\perp}\tilde{\bf w} +\frac{1}{\sqrt{p}} \frac{1}{\|\boldsymbol{\beta}_\star\|^2}{\bf u}^{T}{\bf X}^{T}\boldsymbol{\beta}_\star\boldsymbol{\beta}_\star^{T}\tilde{\bf w}+\frac{1}{2}\|\boldsymbol{\beta}_\star-{\bf w}\|^2+ \frac{1}{\sqrt{p}}\sigma \bn^T\bu-\frac{1}{\sqrt{p}}\epsilon \bone^T|\bu|
$$ 
Let ${\bf z}=\frac{1}{\|\boldsymbol{\beta}_\star\|}{\bf X}^{T}\boldsymbol{\beta}_\star$. Clearly, ${\bf z}$ is a standard Gaussian random vector and is independent of ${\bf X}^{T}{\bf P}_{\perp}$. Replacing $\frac{1}{\|\boldsymbol{\beta}_\star\|}{\bf X}^{T}\boldsymbol{\beta}_\star$ by ${\bf z}$,  $\Phi^{(n)}$ writes as:
$$
\Phi^{(n)}= \min_{{\bf v}} \max_{\substack{{\bf u}}}\ \ \frac{1}{\sqrt{p}} {\bf u}^{T}{\bf X}^{T}{\bf P}_{\perp}\tilde{\bf w} +\frac{1}{\sqrt{p}} \frac{1}{\|\boldsymbol{\beta}_\star\|}{\bf u}^{T}{\bf z} \ \boldsymbol{\beta}_\star^{T}\tilde{\bf w}+\frac{1}{2}\|\boldsymbol{\beta}_\star-\tilde{\bf w}\|^2+ \frac{1}{\sqrt{p}}\sigma \bn^T\bu-\frac{1}{\sqrt{p}}\epsilon \bone^T|\bu|
$$ 

\noindent{\bf Identification of the AO sequences}
Following the notations of section \ref{sec:cgmt_review} in the Appendix, we associate with the PO a sequence of auxiliary problems indexed by $r$ and $\theta$, where each one is associated with a ``double bounded'' version of the PO. Particularly, for fixed $r$ and $\theta$, we define every element of the AO sequence as: 
$$
\phi_{r,\theta}^{(n)}:=  \min_{\substack{\tilde{\bf w}\\ \|\tilde{\bf w}\|\leq r}} \max_{\substack{{\bf u}\\  \|{\bf u}\|\leq \theta}}\ \ \frac{1}{\sqrt{p}}\|{\bf P}_{\perp}\tilde{\bf w}\|_2{\bf g}^{T}{\bf u} + \frac{1}{\sqrt{p}} \|{\bf u}\| {\bf h}^{T}{\bf P}_{\perp}\tilde{\bf w} +\frac{1}{\sqrt{p}} \frac{1}{\|\boldsymbol{\beta}_\star\|}{\bf u}^{T}{\bf z} \ \boldsymbol{\beta}_\star^{T}\tilde{\bf w}+\frac{1}{2}\|\boldsymbol{\beta}_\star-\tilde{\bf w}\|^2+ \frac{1}{\sqrt{p}}\sigma \bn^T\bu-\frac{1}{\sqrt{p}}\epsilon \bone^T|\bu|
$$
\noindent{\bf Simplification of the AO problems}
Having identified the AO sequence, we proceed now with some elementary manipulations to reduce each AO problem  to a scalar  optimization problem. For that and considering the fact that $\tilde{\bf w}$ intervenes in the objective through its norm, its projection along ${\boldsymbol{\beta}}_\star$ and the space orthgonal to it, it is sensible to decompose it as:
$$
\tilde{\bf w}=\gamma_1\frac{{\bf P}_{\perp}\tilde{\bf w}}{\|{\bf P}_\perp\tilde{\bf w}\|}+ \gamma_2\frac{\boldsymbol{\beta}_\star}{\|\boldsymbol{\beta}_\star\|}
$$
Using this decomposition, the AO problem simplifies to:
\begin{align}
\phi_{r,\theta}^{(n)}&=\min_{\gamma_1^2+\gamma_2^2\leq r^2}  \max_{0\leq m \leq \theta} \max_{\substack{\|{\bf u}\|_2=m}} \frac{1}{2}\gamma_1^2+\frac{1}{2}\gamma_2^2+\gamma_1\frac{1}{\sqrt{p}} {\bf g}^{T}{\bf u} -m\gamma_1\frac{1}{\sqrt{p}}\|{\bf P}_\perp {\bf h}\|+\gamma_2\frac{1}{\sqrt{p}}{\bf u}^{T}{\bf z}+\frac{1}{2}\|\boldsymbol{\beta}_\star\|^2-\gamma_1\|\boldsymbol{\beta}_\star\|\\
&+\frac{1}{\sqrt{p}}\sigma \bn^T\bu-\frac{1}{\sqrt{p}}\epsilon \bone^T|\bu|
\end{align}
Hence,
\begin{align}
\phi_{r,\theta}^{(n)}=\min_{\gamma_1^2+\gamma_2^2\leq r^2}  \max_{0\leq m \leq \theta} \max_{\substack{\|{\bf u}\|_2=m}} &\frac{1}{2}\gamma_1^2+\frac{1}{2}\gamma_2^2+\frac{1}{\sqrt{p}}{\bf u}^{T}\left(\gamma_1{\bf g}+\gamma_2{\bf z}+\sigma {\bf n}\right) -\frac{\epsilon}{\sqrt{p}} \|{\bf u} \|_1 -m\gamma_1\frac{1}{\sqrt{p}}\|{\bf P}_{\perp}{\bf h}\|_2\nonumber\\
&+\frac{1}{2}\|\boldsymbol{\beta}_\star\|_2^2-\gamma_2\|\boldsymbol{\beta}_\star\|
\end{align}
Using Lemma \ref{lem:max}, we obtain:
\begin{equation}
\begin{aligned}
\phi_{r,\theta}^{(n)}&=\min_{\gamma_1^2+\gamma_2^2\leq r^2}  \max_{0\leq m \leq \theta}  \frac{1}{2}\gamma_1^2+\frac{1}{2}\gamma_2^2+m\left(\sqrt{\frac{1}{p}\sum_{i=1}^n\Big(|\gamma_1g_i+\gamma_2z_i+\sigma n_i|-\epsilon\Big)_{+}^2}-\gamma_1\frac{1}{\sqrt{p}}\|{\bf P}_{\perp}{\bf h}\|_2\right)\\
&+\frac{1}{2}\|\boldsymbol{\beta}_\star\|_2^2-\gamma_2\|\boldsymbol{\beta}_\star\|_2\\
&=\min_{\gamma_1^2+\gamma_2^2\leq r^2}    \frac{1}{2}\gamma_1^2+\frac{1}{2}\gamma_2^2+\theta\left(\left(\sqrt{\frac{1}{p}\sum_{i=1}^n\Big(|\gamma_1g_i+\gamma_2z_i+\sigma n_i|-\epsilon\Big)_{+}^2}-\gamma_1\frac{1}{\sqrt{p}}\|{\bf P}_{\perp}{\bf h}\|_2\right)\right)_{+}\\
&+\frac{1}{2}\|\boldsymbol{\beta}_\star\|_2^2-\gamma_2\|\boldsymbol{\beta}_\star\|_2
\label{eq:31}
\end{aligned}
\end{equation}
At this point, it is worth pointing out that the new formulation of the AO sequence problems is reduced to  optimization problems that involves only few scalar variables. It relies on a deterministic analysis that does not involve any asymptotic approximation. 

\noindent{\bf Asymptotic behavior of the AOs}
Let us consider the following sequence of functions:
$$
D_n(\gamma_1,\gamma_2):= \left\|\Big(\left|\frac{\gamma_1}{\sqrt{p}}{\bf g} + \frac{\gamma_2}{\sqrt{p}}{\bf z}+\sigma {\bf n}\right|-\epsilon\Big)_{+}\right\|_2 -\frac{\gamma_1}{\sqrt{p}} \|{\bf P}_{\perp}{\bf h}\|_2
$$
defined for $\gamma_1$ and $\gamma_2$ such that $\gamma_1^2+\gamma_2^2\leq r^2$. It is easy to note that $(\gamma_1,\gamma_2)\mapsto D_n(\gamma_1,\gamma_2)$ is jointly convex in its arguments $(\gamma_1,\gamma_2)$ and by \cite[Lemma 10]{thrampoulidis-IT} converges pointwise to
$$
(\gamma_1,\gamma_2)\mapsto  \overline{D}(\gamma_1,\gamma_2)
$$
with
$$
\overline{D}(\gamma_1,\gamma_2):=\sqrt{\delta}\sqrt{\mathbb{E}\Big(|\sqrt{\gamma_1^2+\gamma_2^2}{G}+\sigma N|-\epsilon\Big)_{+}^2}-\gamma_1
$$
where ${G}\sim\mathcal{N}(0,1)$. 
Since convergence of convex functions is uniform over compacts, then letting $\mathcal{C}$ be an arbitrary compact in $\mathbb{R}^2$, and any $\eta>0$, for $n$ sufficiently large, it holds that:
\begin{equation}
\overline{D}(\gamma_1,\gamma_2)-\eta\leq D_n(\gamma_1,\gamma_2) \leq \overline{D}(\gamma_1,\gamma_2)+\eta,  \ \ \forall \  (\gamma_1,\gamma_2)\in\mathcal{C}
\label{eq:uniform}
\end{equation}
\subsection{Proof of Theorem 1}
Letting $\delta_\star=\frac{1}{\displaystyle\inf_{t\in\mathbb{R}}\mathbb{E}\left(|{G}+t\sigma N|-t\epsilon\right)_{+}^2}$, we will prove the following two statements:
\begin{align}
\delta > \delta_\star &\Rightarrow \mathbb{P}\left[\text{The hard-margin SVR is feasible for sufficiently large} \ \ n\right] = 0 \label{eq:first}\\ 
\delta<\delta_\star &\Rightarrow  \mathbb{P}\left[\text{The hard-margin SVR is feasible for sufficiently large} \ \ n\right] = 1. \label{eq:second}
\end{align}
In the sequel, we will sequentially establish \eqref{eq:first} and \eqref{eq:second}.
 
\noindent{{\bf Proof of \eqref{eq:first}.}}
To prove the desired result, we will apply Theorem \ref{th:feas}. More precisely, in view of \eqref{eq:cond_inf} in Theorem \ref{th:feas}, it suffices to prove that under the condition $\delta>\delta_\star$, for any fixed $r$ and $\theta$, there exists constant $C>0$ that is independent of $r$ and $\theta$ such that for sufficiently large $n$ (taken independently of $r$ and $\theta$):
\begin{equation}
\phi_{r,\theta}^{(n)} \geq C \theta \label{eq:ineq}
\end{equation}
To prove \eqref{eq:ineq}, we start by noting that:
\begin{equation}
\phi_{r,\theta}^{(n)}=\min_{\gamma_1^2+\gamma_2^2\leq r^2}   \frac{1}{2}(\gamma_2-\|\boldsymbol{\beta}_\star\|)^2+\frac{1}{2}\gamma_1^2 +\theta\Big(D_n(\gamma_1,\gamma_2)\Big)_+  \label{eq:phirtheta}
\end{equation}
Hence,
\begin{equation}
\phi_{r,\theta}^{(n)}\geq \theta\Big(\min_{\substack{\gamma_1,\gamma_2}}  D_n(\gamma_1,\gamma_2)\Big)_+
\label{eq:phi_rn}
\end{equation}
Function $(\gamma_1,\gamma_2)\mapsto D_n(\gamma_1,\gamma_2)$ is jointly convex in its arguments $(\gamma_1,\gamma_2)$ and converges pointwise to $(\gamma_1,\gamma_2)\mapsto \overline{D}(\gamma_1,\gamma_2)$.    Using Lemma \ref{lem_conc}, for $\gamma_1$ fixed, $\gamma_2\mapsto \ D_n(\gamma_1,\gamma_2)$ is convex. It converges pointwise to $\gamma_2\mapsto  \overline{D}(\gamma_1,\gamma_2)$. As $\lim_{\gamma_2\to \pm\infty} \overline{D}(\gamma_1,\gamma_2)=\infty$,  we may use  \cite[Lemma 10]{thrampoulidis-IT} to obtain,
$$
\displaystyle \min_{\gamma_2\in\mathbb{R}} D_n(\gamma_1,\gamma_2) \asto \min_{\gamma_2\in\mathbb{R}} \overline{D}(\gamma_1,\gamma_2)  
$$   
Moreover, it is easy to see that $\gamma_2\mapsto  \overline{D}(\gamma_1,\gamma_2)$ is an even function. As it is convex, from Lemma \ref{lem:even}, the minimum is taken at $\gamma_2=0$ and hence $\min_{\gamma_2} \overline{D}(\gamma_1,\gamma_2)=\overline{D}(\gamma_1,0)$. 
Similarly, $\gamma_1\mapsto  \displaystyle\min_{\gamma_2} D_n(\gamma_1,\gamma_2) $ is convex and converges pointwise to $\gamma_1\mapsto \overline{D}(\gamma_1,0) $. Moreover, for fixed $\gamma_1$, it is easy to see that
\begin{align}
 \overline{D}(\gamma_1,0)&= |\gamma_1|\left(\sqrt{\delta}\sqrt{\mathbb{E}\Big(|{G}+\frac{\sigma}{\gamma_1}N|-\frac{1}{{|\gamma_1|}}\epsilon\Big)_{+}^2}-\frac{\gamma_1}{|\gamma_1|}\right)\\
&\geq |\gamma_1|\Big(\sqrt{\delta}\sqrt{\inf_{t\in \mathbb{R}}\mathbb{E}\left(|{G}+t\sigma N|-t\epsilon\right)_{+}^2}-1\Big)
\end{align}
Since by assumption $\delta >\delta_\star$, 
$$
\inf_{t\in\mathbb{R}}\sqrt{\delta}\mathbb{E}\left(|{G}+t\sigma N|-t\epsilon\right)_{+}^2>1
$$
Hence, $\lim_{\gamma_1\to\pm\infty} \min_{\gamma_2} \overline{D}(\gamma_1,\gamma_2)=\lim_{\gamma_1\to\infty}\overline{D}(\gamma_1,0)=\infty$. Using again \cite[Lemma 10]{thrampoulidis-IT}, we obtain:
$$
\min_{\substack{\gamma_1,\gamma_2}} {D}_n(\gamma_1,\gamma_2)\asto \min_{\substack{\gamma_1}} \overline{D}(\gamma_1,0) 
$$ 
It follows from \eqref{eq:phi_rn} that for any $\eta>0$, the following holds true once $n$ is taken sufficiently large  independently of $r$ and $\theta$,  
$$
\phi_{r,\theta}^{(n)}\geq \theta \left(\min_{\gamma_1} \overline{D}(\gamma_1,0)-\eta\right)_{+}
$$
Assume that 
\begin{equation}
\min_{\gamma_1} \overline{D}(\gamma_1,0)> 2C.  
\label{eq:D}
\end{equation}
where $C$ is some strictly positive constant. Then we can choose $\eta$ sufficiently small, for instance smaller than $C$, to obtain \eqref{eq:ineq}.
To conclude, it suffices thus to establish \eqref{eq:D}. To this end, first notice that for all $\gamma_1$ such that $|\gamma_1|\leq \tilde{\eta}:=\frac{1}{2}\sqrt{\delta\mathbb{E}\left[\left(\sigma|N|-\epsilon\right)_{+}^2\right]}$, the following holds true:
\begin{align}
\min_{\substack{\gamma_1\\ |\gamma_1|\leq \tilde{\eta}}}\overline{D}(\gamma_1,0)&\geq \min_{\substack{\gamma_1\\ |\gamma_1|\leq \tilde{\eta}}}\sqrt{\delta\mathbb{E}\left( |\gamma_1 G +\sigma N|-\epsilon\right)_{+}^2} -\tilde{\eta}\\
&\geq \sqrt{\delta\mathbb{E}\left( |\sigma N|-\epsilon\right)_{+}^2} -\tilde{\eta}\\
&\geq \frac{1}{2} \sqrt{\delta\mathbb{E}\left( |\sigma N|-\epsilon\right)_{+}^2} =\tilde{\eta}>0 \label{eq:f2}.
\end{align}
where the last inequality follows by Lemma \ref{lem:even}.  
On the other hand, 
\begin{align}
\min_{\substack{\gamma_1\\ |\gamma_1|\geq \tilde{\eta}}}\overline{D}(\gamma_1,0) &\geq \min_{\substack{\gamma_1\\ |\gamma_1|\geq \tilde{\eta}}}|\gamma_1|\Big(\sqrt{\delta\mathbb{E}\left(|G+\frac{\sigma}{\gamma_1}N|-\frac{1}{|\gamma_1|}\epsilon\right)_{+}^2}-1\Big)\\
&\geq \tilde{\eta}\left(\sqrt{\delta\inf_{t\in\mathbb{R}}\mathbb{E}\left(|{G}+t\sigma N|-t\epsilon\right)_{+}^2}-1\right)=\tilde{\eta}\left(\sqrt{\frac{\delta}{\delta_\star}}-1\right)>0. \label{eq:f1}
\end{align}
Putting \eqref{eq:f2} and \eqref{eq:f1} together, we obtain:
$$
\min_{\gamma_1}\overline{D}(\gamma_1,0) \geq \min\left(\tilde{\eta},\tilde{\eta}\left(\sqrt{\frac{\delta}{\delta_\star}}-1\right)\right)
$$
Letting $C:=\frac{1}{2} \min\left(\tilde{\eta},\tilde{\eta}\left(\sqrt{\frac{\delta}{\delta_\star}}-1\right)\right)$ yields \eqref{eq:D}.

\noindent{{\bf Proof of \eqref{eq:second}}} The aim here is to show that if $\delta<\delta_\star$, then the PO and thus the H-SVR is feasible with probability~1 for sufficiently large $n$. For that, we will apply the second item of Theorem \ref{th:feas}. To begin with, we shall start by showing that there exists  $\kappa_0>0$ such that for all $\kappa\leq \kappa_0$ the following set is non-empty:
\begin{equation}
\mathcal{I}:=\left\{\gamma_1 \  | \ \overline{D}(\gamma_1,0)\leq -\kappa \right\} \neq \emptyset \label{eq:I}\end{equation} and open. 
Let $t_\star$ be such that $t_\star\in\arg\inf_{t\in\mathbb{R}}\mathbb{E}\left(|{G}+t\sigma N|-t\epsilon\right)_{+}^2$. Obviously $t_\star$ is finite. To see this, it suffices to note that:
\begin{align}
\mathbb{E}\left(|{G}+t\sigma N|-t\epsilon\right)_{+}^2&\geq \mathbb{E} \left[\left(|{G}+t\sigma N|-t\epsilon\right)_{+}^2 {\bf 1}_{\{\sigma N \geq 2\epsilon\}}{\bf 1}_{\{{G}>0\}} \right]\\
&\geq t^2\epsilon^2  {\bf 1}_{\{\sigma N \geq 2\epsilon\}}{\bf 1}_{\{{G}>0\}}\\
&\underset{t\to\pm \infty}{\longrightarrow} \infty
\end{align}
By assumption $\sqrt{\delta} \sqrt{\mathbb{E}\left(\left|{G}+t\sigma N\right|-t\epsilon\right)_{+}^2}<1$. Hence $\overline{D}(\frac{1}{|t_\star|},0)<0$. By continuity of $\gamma_1\mapsto \overline{D}(\gamma_1,0)$ we prove that $\mathcal{I}$ is non-empty and open. With this at hand, we let $C_{\kappa}$ be given by:
\begin{equation}
C_\kappa:=\sup_{\gamma_1\in\mathcal{I}} |\gamma_1| 
\label{eq:C_kappa}
\end{equation}
Recall that from the uniform convergence of $\gamma\mapsto D_n(\gamma,0)$ to $\gamma\mapsto \overline{D}({\gamma,0})$ over compacts, we may argue that for any $\gamma_1\in[-C_\kappa,C_\kappa]$ and for $n$ sufficiently large:
$$
 \ \ {D}_n(\gamma_1,0) \leq \overline{D}(\gamma_1,0)+\kappa
$$
Hence, 
$$
\left\{\gamma_1, D_n(\gamma_1,0)\leq 0\right\} \subset \left\{\overline{D}(\gamma_1,0)\leq -\kappa\right\}
$$
and as such:
$$
\Big\{\left\{\gamma_1, D_n(\gamma_1,0)\leq 0\right\}\cap [-C_\kappa,C_\kappa]\Big\} \subset \Big\{\left\{\overline{D}(\gamma_1,0)\leq -\kappa\right\}\cap   [-C_\kappa,C_\kappa]\Big\}
$$ 
For $k\in\mathbb{R}$, define $\phi_{k}^{(n)}$:
$$
\phi_{k}^{(n)}=\sup_{\theta\geq 0} \phi_{r,\theta}^{(n)}
$$
Then, using the min-max inequality, we have:
\begin{align}
\phi_{k}^{(n)}&\leq \min_{\substack{\gamma_1,\gamma_2}}\sup_{\theta\geq 0} \frac{1}{2}(\gamma_2-\|\boldsymbol{\beta}_\star\|)^2+\frac{1}{2}\gamma_1^2 +\theta\left(D_n(\gamma_1,\gamma_2)\right)_{+} \\
&\leq \min_{\substack{\gamma_1\in\mathcal{I}\\ |\gamma_1|\leq C_\kappa}} \sup_{\theta\geq 0} \frac{1}{2}\|\boldsymbol{\beta}_\star\|^2 +\frac{1}{2}\gamma_1^2+\theta\left(D_n(\gamma_1,0)\right)_{+}\\
&\leq \frac{1}{2}\left(C_\kappa^2 +\|\boldsymbol{\beta}\|^2\right)
\end{align}
This shows that there exists a constant $C$ such that almost surely:
$$
\phi_{k}^{(n)}\leq C. 
$$
By Theorem \ref{th:feas}-ii), we conclude that if $\delta<\delta_\star$, the hard-margin SVR is almost surely feasible. 
\subsection{Proof of Theorem 2}
The  proof of Theorem 2 follows from applying the result of Theorem \ref{th:conv} in the Appendix. 
Let $\overline{\phi}$ be defined as:
\begin{equation}
\overline{\phi}= \min_{\overline{D}(\gamma_1,\gamma_2)\leq 0}\frac{1}{2}\gamma_1^2+\frac{1}{2}(\gamma_2-\beta)^2. 
\label{eq:phi}
\end{equation}
and denote by $\gamma_1^\star$ and $\gamma_2^\star$ its corresponding minimizers \footnote{The existence and uniqueness of $\gamma_1^\star$ and $\gamma_2^\star$ follows from the fact that the objective is strictly convex and coercive.}. 
We need to prove that the following statements hold true:
\begin{subequations}
\begin{align}
\text{For any} \  \epsilon>0: &\ \mathbb{P}\left[\cup_{k=k_0}^\infty \left\{\phi_{k}^{(n)}\leq \overline{\phi}-\epsilon\right\}, \ \ \text{i.o.}\right]=0.\label{eq:pp1}\\
\text{For any} \  \epsilon>0: &\ \mathbb{P}\left[ \left\{\phi_{k_0}^{(n)}\geq \overline{\phi}-\epsilon\right\}, \ \ \text{i.o.}\right]=0.  \label{eq:pp2}
\end{align} 
\end{subequations}
where $k_0$ is some integer sufficiently large, and $\phi_{k}^{(n)}$ is defined as:
$$
\phi_{k}^{(n)}=\sup_{\theta\geq 0} \phi_{k,\theta}^{(n)}
$$
We shall first simplify the expression of $\phi_{k}^{(n)}$. It follows from \eqref{eq:phirtheta} that:
\begin{align}
\phi_{k}^{(n)}&=\sup_{\theta\geq 0} \min_{\substack{\gamma_1,\gamma_2\\ \gamma_1^2+\gamma_2^2\leq k^2}} \frac{1}{2}(\gamma_2-\|\boldsymbol{\beta}_\star\|)^2+\frac{1}{2}\gamma_1^2 +\theta \left(D_n(\gamma_1,\gamma_2)\right)_{+}\\
&=\min_{\substack{\gamma_1,\gamma_2\\ \gamma_1^2+\gamma_2^2\leq k^2}}\sup_{\theta\geq 0} \frac{1}{2}(\gamma_2-\|\boldsymbol{\beta}_\star\|)^2+\frac{1}{2}\gamma_1^2 +\theta\left(D_n(\gamma_1,\gamma_2)\right)_{+}\label{eq:1}\\
&=\min_{\substack{\gamma_1,\gamma_2\\ \gamma_1^2+\gamma_2^2\leq k^2\\   D_n(\gamma_1,\gamma_2)\leq 0}} \frac{1}{2}(\gamma_2-\|\boldsymbol{\beta}_\star\|)^2+\frac{1}{2}\gamma_1^2 \label{eq:2}
\end{align}
The second equality \eqref{eq:1} is because  $(\gamma_1,\gamma_2,\theta)\mapsto \theta D_n(\gamma_1,\gamma_2)$ is convex in $(\gamma_1,\gamma_2)$ and concave in $\theta$ while \eqref{eq:2} follows since, as previously proven, under the condition $\delta<\delta_\star$, the set $$\left\{(\gamma_1,\gamma_2) \ | \ (\gamma_1^2+\gamma_2^2)\leq k^2 \  \text{and} \ D_n(\gamma_1,\gamma_2)\leq 0\right\}$$ is not empty.

 The function $(\gamma_1,\gamma_2)\mapsto D_n(\gamma_1,\gamma_2)$ is convex and converges pointwise to $(\gamma_1,\gamma_2)\mapsto \overline{D}(\gamma_1,\gamma_2,0)$. Hence, it converges uniformly over compacts. For $\kappa>0$, any $r \in\mathbb{N}^*$ and sufficiently large $n$, the following holds true: 
$$
\overline{D}(\gamma_1,\gamma_2)-\kappa\leq {D}_n(\gamma_1,\gamma_2)\leq \overline{D}(\gamma_1,\gamma_2)+\kappa, \ \ \forall \ (\gamma_1,\gamma_2) \ \ \text{such that} \ \gamma_1^2+\gamma_2^2\leq r^2
$$
and thus:
\begin{equation}
\Big\{\{(\gamma_1,\gamma_2) | \overline{D}(\gamma_1,\gamma_2)\leq -\kappa   \ \text{and}\ \gamma_1^2+\gamma_2^2\leq r^2\}\Big\} \subset \Big\{\left\{(\gamma_1,\gamma_2) |  \ {D}_n(\gamma_1,\gamma_2)\leq 0 \right\} \ \text{and}\ \ \gamma_1^2+\gamma_2^2\leq r^2\Big\}\label{eq:un}
\end{equation}

Before handling the proof of \eqref{eq:pp1} and \eqref{eq:pp2}, we need first to establish that for integer $k$ sufficiently large $\phi_{k}^{(n)}$ does not change with $k$, which will require often the use of  \eqref{eq:un}.  To start,  for $\kappa>0$, we consider the set:
$$
\mathcal{D}_\kappa=\left\{(\gamma_1,\gamma_2)  \ | \ \overline{D}(\gamma_1,\gamma_2)\leq -\kappa\right\}
$$
We let $(\gamma_1^\kappa,\gamma_2^\kappa)$ be defined as:
$$
(\gamma_1^\kappa,\gamma_2^\kappa)=\arg\min_{\substack{\gamma_1,\gamma_2\\ \overline{D}(\gamma_1,\gamma_2)\leq -\kappa}} \frac{1}{2}(\gamma_2-\beta)^2+\frac{1}{2}\gamma_1^2
$$
and define $k_0$ as $k_0:={\rm max}\left(\lceil\sqrt{2}\left(\sqrt{(\gamma_2^\kappa-\beta)^2+(\gamma_1^\kappa)^2}+\beta\right)\rceil+1,2\sqrt{(\gamma_1^\kappa)^2+(\gamma_2^\kappa)^2}\right).$ 
Note that $\gamma_1^\kappa$ and $\gamma_2^\kappa$ are also given by:
$$
(\gamma_1^\kappa,\gamma_2^\kappa)=\arg\min_{\substack{\gamma_1,\gamma_2\\ \overline{D}(\gamma_1,\gamma_2)\leq -\kappa\\ \gamma_1^2+\gamma_2^2\leq k_0^2}} \frac{1}{2}(\gamma_2-\beta)^2+\frac{1}{2}\gamma_1^2
$$
We will thus prove that for all $k\geq k_0$, with probability $1$ for sufficiently large $n$,  
\begin{equation}
\phi_{k}^{(n)}=\phi_{k_0}^{(n)}. \label{eq:phik}
\end{equation}
To begin with, we invoke the uniform convergence of $(\gamma_1,\gamma_2)\mapsto \frac{1}{2}(\gamma_2-\|\boldsymbol{\beta}_\star\|)^2+\frac{1}{2}\gamma_1^2$ to  $(\gamma_1,\gamma_2)\mapsto \frac{1}{2}(\gamma_2-\beta)^2+\frac{1}{2}\gamma_1^2$, to claim that for any $\eta>0$, and $(\gamma_1,\gamma_2)\in\left\{(\gamma_1,\gamma_2) \ | \ \gamma_1^2+\gamma_2^2\leq k_0^2\right\}$
$$
\frac{1}{2}(\gamma_2-\|\boldsymbol{\beta}_\star\|)^2+\frac{1}{2}\gamma_1^2\leq \frac{1}{2}(\gamma_2-\beta)^2+\frac{1}{2}\gamma_1^2+\eta
$$
and hence, 
$$
\phi_{k_0}^{(n)}\leq \min_{\substack{\gamma_1,\gamma_2\\ D_n(\gamma_1,\gamma_2)}} \frac{1}{2}(\gamma_2-\beta)^2+\frac{1}{2}\gamma_1^2+\eta
$$
Next, we use \eqref{eq:un} to establish that \begin{equation}\phi_{k_0}^{(n)} \leq \frac{1}{2}(\gamma_2^{\kappa}-\beta)^2 +\frac{1}{2}(\gamma_1^\kappa)^2 +\eta. \label{eq:inequality}\end{equation}
Clearly, $\phi_{k}^{(n)}\leq \phi_{k_0}^{(n)}$.  Let $\tilde{\gamma}_1$ and $\tilde{\gamma}_2$ be the minimizers of ${\phi}_{k}^{(n)}$. Then,
$$
\phi_{k}^{(n)} < \phi_{k_0}^{n} \ \ \Longrightarrow  \ \ (\tilde{\gamma}_1,\tilde{\gamma}_2) \notin \left\{(\gamma_1,\gamma_2)  \ | \ \gamma_1^2+\gamma_2^2\leq k_0^2\right\}
$$
To prove \eqref{eq:phik}, we proceed by contradiction and assume that \begin{equation}\phi_{k}^{(n)}<\phi_{k_0}^{n}.\label{eq:strict}\end{equation} Then, it is easy to see that in this case, we must have:
$$
\tilde{\gamma}_1^2+\tilde{\gamma}_2^2\geq k_0^2
$$
Therefore, either $\tilde{\gamma}_1^2\geq \frac{1}{2}{k_0^2}$ or $\tilde{\gamma}_2^2\geq \frac{1}{2}k_0^2$.

\noindent{\underline{First case: $\tilde{\gamma}_1^2\geq \frac{1}{2}k_0^2$}}. In this case: 
$$
\phi_{k}^{(n)} = \frac{1}{2}(\tilde{\gamma}_2-\|\boldsymbol{\beta}_\star\|)^2+\frac{1}{2}\tilde{\gamma}_1^2\geq \frac{1}{4}k_0^2\geq \frac{1}{2}(\gamma_2^\kappa-\beta)^2+\frac{1}{2}(\gamma_1^\kappa)^2+\frac{1}{4}\geq\phi_{k_0}^{(n)}
$$
where the last inequality follows from \eqref{eq:inequality}
This shows that $\phi_{k}^{(n)}\geq \phi_{k_0}^{(n)}$ which contradicts \eqref{eq:strict}.

\noindent{\underline{Second case: $ \tilde{\gamma}_2^2\geq \frac{1}{2}k_0^2$}} Using the relation $(x-y)^2\geq (|x|-|y|)^2$, we obtain 
$$
\phi_{k}^{(n)}\geq \frac{1}{2}(\sqrt{\frac{1}{2}}k_0-\|\boldsymbol{\beta}_\star\|)^2\geq \frac{1}{2}\left(\sqrt{(\gamma_2^\kappa-\beta)^2+(\gamma_1^\kappa)^2}+\beta-\|\boldsymbol{\beta}_\star\|+\frac{1}{\sqrt{2}} \right)^2\geq\phi_{k_0}^{(n)}
$$
where we used the fact $|\beta-\|\boldsymbol{\beta}_\star\||$ can be assumed  sufficiently small (let say smaller than $\frac{1}{2\sqrt{2}}$) for $n$ sufficiently large. 
This again contradicts \eqref{eq:strict} which proves that $\phi_{k}^{(n)}=\phi_{k_0}^{(n)}$ for $k\geq k_0$ and $n$ sufficiently large but independent of $k$. 
With the proof of \eqref{eq:phik} at hand, we are now ready to show \eqref{eq:pp1} and \eqref{eq:pp2}.

\noindent{\underline{Proof of \eqref{eq:pp1}}} 
The proof relies on using the  uniform convergence of $(\gamma_1,\gamma_2)\mapsto \min_{\tilde{b}} D_n(\gamma_1,\gamma_2,\tilde{b})$ to $(\gamma_1,\gamma_2)\mapsto \overline{D}(\gamma_1,\gamma_2,0)$ along with Lemma \ref{lem:equality}. More specifically, recalling \eqref{eq:uniform}, it holds that:
\begin{equation}
\Big\{\{(\gamma_1,\gamma_2) | \overline{D}_n(\gamma_1,\gamma_2)\leq 0   \ \text{and}\ \gamma_1^2+\gamma_2^2\leq r^2\}\Big\} \subset \Big\{\left\{(\gamma_1,\gamma_2) |\ \overline{D}(\gamma_1,\gamma_2)\leq \kappa \right\} \ \text{and}\ \ \gamma_1^2+\gamma_2^2\leq r^2\Big\}
\end{equation} 
Hence,
\begin{align}
\phi_{k_0}^{(n)} &= \min_{\substack{\gamma_1^2+\gamma_2^2\leq k_0^2\\ {D}_n(\gamma_1,\gamma_2)\leq 0}} \frac{1}{2}(\gamma_2-\|\boldsymbol{\beta}_\star\|^2) +\frac{1}{2}\gamma_1^2 \\ &\geq  \min_{\substack{\gamma_1^2+\gamma_2^2\leq k_0^2\\ \overline{D}(\gamma_1,\gamma_2)\leq \kappa}} \frac{1}{2}(\gamma_2-\|\boldsymbol{\beta}_\star\|)^2 +\frac{1}{2}\gamma_1^2 
\end{align}
Since $(\gamma_1,\gamma_2)\mapsto  \frac{1}{2}(\gamma_2-\|\boldsymbol{\beta}_\star\|^2) +\frac{1}{2}\gamma_1^2$ converges uniformly to $(\gamma_1,\gamma_2)\mapsto \frac{1}{2}(\gamma_2-\beta)^2 +\frac{1}{2}\gamma_1^2$ over compacts, for any $\eta>0$, there exists $n$ sufficiently large such that:
\begin{align}
\phi_{k_0}^{(n)}&\geq \left(\min_{\substack{\gamma_1^2+\gamma_2^2\leq k_0^2\\ \overline{D}(\gamma_1,\gamma_2)\leq \kappa}} \frac{1}{2}(\gamma_2-\beta)^2 +\frac{1}{2}\gamma_1^2\right) -\eta\geq  \left(\min_{\substack{\gamma_1^2+\gamma_2^2\leq k_0^2\\ \overline{D}(\gamma_1,\gamma_2)\leq \kappa}} \frac{1}{2}(\gamma_2-\beta)^2 +\frac{1}{2}\gamma_1^2\right) -2\eta
\end{align}
where the last inequality follows by applying Lemma \ref{lem:equality}.


\noindent{\underline{Proof of \eqref{eq:pp2}}}
Fix any $\eta>0$. It follows from \eqref{eq:un} that 
\begin{align}
\phi_{k_0}^{(n)} \leq  \left(\min_{\substack{\gamma_1^2+\gamma_2^2\leq k_0^2\\ \overline{D}(\gamma_1,\gamma_2,0)\leq -\kappa}} \frac{1}{2}(\gamma_2-\beta)^2 +\frac{1}{2}\gamma_1^2\right) +\eta
\end{align}
Similarly, we conclude by invoking Lemma \ref{lem:equality} and using the $\eta$-definition of the infimum.

With \eqref{eq:pp1} and \eqref{eq:pp2} at hand, we now use Theorem \ref{th:conv} to conclude that:
\begin{equation}
\Phi^{(n)}=\|\hat{\bf w}_H\|\to \overline{\phi}= \frac{1}{2}(\gamma_2^\star-\beta)^2+\frac{1}{2}(\gamma_1^\star)^2
\label{eq:conv_res}
\end{equation}
However, this convergence result is not sufficient to establish that of $\|\hat{\bf w}_H-\boldsymbol{\beta}_\star\|^2$ and the cosine similarity.  Hopefully, we can easily prove these results by working with the solution $\hat{{{\tilde{{\bf w}}}}}_H$ of \eqref{eq:primary_bis} which, as per \eqref{eq:relation}, writes as $\hat{\tilde{{\bf w}}}_H=-\hat{\bf w}_H+\boldsymbol{\beta}_\star$. The norm of $\hat{\tilde{{\bf w}}}_H$ represents the risk while the cosine similarity can be retrieved by using the relation $$\frac{\hat{\tilde{{\bf w}}}_H^{T}\boldsymbol{\beta}_\star^{T}}{\|\boldsymbol{\beta}_\star\|}=\frac{\hat{\bf w}_H^{T}\boldsymbol{\beta}_\star}{\|\boldsymbol{\beta}_\star\|}-\|\boldsymbol{\beta}_\star\|$$ 
In the sequel, we will focus on the convergence of the quantity $\frac{\hat{\tilde{{\bf w}}}_H^{T}\boldsymbol{\beta}_\star^{T}}{\|\boldsymbol{\beta}_\star\|}$, from which the cosine similarity easily follows. The convergence of the risk can be done in a similar way and is omitted for brevity.
 
For that, we need to consider as suggested by Theorem \ref{th:conv} a perturbed version of the sequence of the AO problem in which $\tilde{\bf w}$ is further constrained on the set $\mathcal{S}_\xi$:
$$
\mathcal{S}_\xi=\left\{\tilde{\bf w}\in\mathbb{R}^{p} \ | \ \left|\frac{\tilde{\bf w}^T\boldsymbol{\beta}_\star}{\|\boldsymbol{\beta}_\star\|}-\gamma_2^\star\right|>\xi\right\}
$$
where $\xi$ is any positive scalar.  Particularly, we define for a given $r$ and $\theta$,  $\tilde{\phi}_{r,\theta}$ as: 
$$
\tilde{\phi}_{r,\theta}^{(n)}=\min_{\substack{\tilde{\bf w}\\ \|\tilde{\bf w}\|\leq r\\ \tilde{\bf w}\notin\mathcal{S}_\xi}} \max_{\substack{{\bf u}\\\|{\bf u}\|\leq \theta}} 
\frac{1}{\sqrt{p}}\|{\bf P}_{\perp}\tilde{\bf w}\|_2{\bf g}^{T}{\bf u} + \frac{1}{\sqrt{p}} \|{\bf u}\| {\bf h}^{T}{\bf P}_{\perp}\tilde{\bf w} +\frac{1}{\sqrt{p}} \frac{1}{\|\boldsymbol{\beta}_\star\|}{\bf u}^{T}{\bf z} \ \boldsymbol{\beta}_\star^{T}\tilde{\bf w}+\frac{1}{2}\|\boldsymbol{\beta}_\star-\tilde{\bf w}\|^2+ \frac{1}{\sqrt{p}}\sigma \bn^T\bu-\frac{1}{\sqrt{p}}\epsilon \bone^T|\bu|
$$
Also, we define for $r\geq 0$, $\tilde{\phi}_r$ as:
$$
\tilde{\phi}_{r}^{(n)}=\sup_{\theta\geq 0} \tilde{\phi}_{r,\theta}
$$
Following the same calculations that led to \eqref{eq:31}, we can simplify $\tilde{\phi}_{r,\theta}^{(n)}$ as:
$$
\tilde{\phi}_{r,\theta}^{(n)}=\min_{\substack{\gamma_1^2+\gamma_2^2\leq r^2\\ |\gamma_2-\gamma_2^\star|\geq \xi}}\frac{1}{2}\gamma_1^2+\frac{1}{2}\gamma_2^2+\frac{1}{2}\|\boldsymbol{\beta}_\star\|_2^2-\gamma_2\|\boldsymbol{\beta}_\star\|_2+\theta \left(D_n(\gamma_1,\gamma_2)\right)_{+}
$$
Since the objective is convex in $(\gamma_1,\gamma_2)$, using \cite[Remark 3]{kam-chris}, $\tilde{\phi}_r^{(n)}$ can be simplified as:
$$
\tilde{\phi}_r^{(n)}=\min_{\substack{\gamma_1^2+\gamma_2^2\leq r^2\\ |\gamma_2-\gamma_2^\star|\geq \xi}}\sup_{\theta\geq 0} \frac{1}{2}\gamma_1^2+\frac{1}{2}\gamma_2^2+\frac{1}{2}\|\boldsymbol{\beta}_\star\|_2^2-\gamma_2\|\boldsymbol{\beta}_\star\|_2+\theta \left(D_n(\gamma_1,\gamma_2)\right)_{+}
=\min_{\substack{\gamma_1^2+\gamma_2^2\leq r^2\\ |\gamma_2-\gamma_2^\star|\geq \xi\\ D_n(\gamma_1,\gamma_2)\leq 0}} \frac{1}{2}\gamma_1^2+\frac{1}{2}\gamma_2^2+\frac{1}{2}\|\boldsymbol{\beta}_\star\|_2^2-\gamma_2\|\boldsymbol{\beta}_\star\|_2
$$
By reference to Theorem \ref{th:conv}, it suffices to prove that there exists $\zeta>0$ such that for sufficiently large $n$ (independent of $r$) it holds that:
\begin{align}
\tilde{\phi}_r^{(n)}\geq \overline{\phi} +{\zeta}\label{eq:first_p}
\end{align}
or equivalently, 
\begin{align}
\forall r, \ \ \tilde{\phi}_r^{(n)}\geq \overline{\phi} +{\zeta}\label{eq:first_p1}
\end{align}
For that we proceed in two steps. In the first, step, we prove that for any $\tilde{\zeta}>0$:
\begin{equation}
\tilde{\phi}_r^{(n)}\geq \tilde{{\phi}}-\tilde{\zeta}
\label{eq:in}
\end{equation}
where 
$$
\tilde{{\phi}}:=\min_{\substack{\gamma_1,\gamma_2\\ |\gamma_2-\gamma_2^\star|\geq \xi\\ \overline{D}(\gamma_1,\gamma_2)\leq 0}} \frac{1}{2}\gamma_1^2+\frac{1}{2}\gamma_2^2 + \frac{1}{2}\beta^2-\gamma_2{\beta}_2
$$
The proof of \eqref{eq:in} relies on the uniform convergence of $D_n$ to $\overline{D}$ and is identical to what is done in \eqref{eq:f2}. Details are thus omitted for brevity.
In the second step, we show that there exists $\zeta>0$ such that:
\begin{align}
\tilde{\phi}\geq \overline{\phi}+2\zeta.\label{eq:z}
\end{align}
To see this, it suffices to note that $\tilde{\phi}-\overline{\phi}>0$. Indeed, clearly $\tilde{\phi}\geq \overline{\phi}$, however, since $(\gamma_1^\star ,\gamma_2^\star)$ are the unique minimizers of $\overline{\phi}$, we must have  $\tilde{\phi}-\overline{\phi}>0$. Hence, \eqref{eq:z} holds for $\zeta=\frac{\tilde{\phi}-\overline{\phi}}{2}$. Starting from \eqref{eq:in} and using $\tilde{\zeta}$ smaller than  $\zeta$, we obtain \eqref{eq:first_p}. Based on Theorem \ref{th:conv}, we may conclude that the optimizer $\hat{{\tilde{{\bf w}}}}_H$ of  \eqref{eq:primary_bis} satisfies:
\begin{equation}
\frac{\hat{\tilde{{\bf w}}}_H^{T}\boldsymbol{\beta}_\star}{\|\boldsymbol{\beta}_\star\|}\asto \gamma_2^\star \label{eq:c}
\end{equation}
Hence, 
$$
\frac{\hat{{{\bf w}}}_H^{T}\boldsymbol{\beta}_\star}{\|\boldsymbol{\beta}_\star\|}\asto -\gamma_2^\star+\beta_\star
$$
Using \eqref{eq:conv_res}, we thus obtain:
$$
\frac{\hat{{{\bf w}}}_H^{T}\boldsymbol{\beta}_\star}{\|\boldsymbol{\beta}_\star\|\|\hat{\bf w}_H\|}\asto \frac{-\gamma_2^\star+\beta_\star}{\sqrt{\frac{1}{2}(\gamma_2^\star-\beta)^2+\frac{1}{2}(\gamma_1^\star)^2}}
$$
Similarly, by defining the perturbed AO problem in which $\tilde{\bf w}$ is further constrained on the set $${\bf w}\in\left\{\tilde{\bf w}\in\mathbb{R}^p \ | \ \left|\|\tilde{\bf w}\|-\sqrt{(\gamma_1^\star)^2+(\gamma_2^\star)^2}\right|>\xi\right\}$$ and following the same approach, we can prove the convergence of the risk to $(\gamma_1^\star)^2+(\gamma_2^\star)^2$. The results of Theorem \ref{pred_risk_conv_HM} are then obtained by performing the change of variable  $\tilde{\gamma}_1=\frac{\gamma_1}{\sigma}$ and $\tilde{\gamma}_2:=\frac{\gamma_2}{\sigma}$ 


%% file: Proof_Soft.tex
\section{Performance of S-SVR: Proof of Theorem \ref{soft_thm}}
\label{app_thm3}
This section is devoted to the analysis of the statistical behavior of the S-SVR. 
To begin with, we shall note that  the S-SVR can also be written as:
\begin{equation}
\begin{aligned}
\min_{{\bf w}} \quad & \frac{1}{2}\|{\bf w}\|^2 + \frac{C}{p}\sum_{i=1}^n\xi_i\\
\textrm{s.t.} \quad & y_i-{\bf w}^{T}{\bf x}_i \leq \epsilon+\xi_i, \ i=1,\cdots,n \\
  &   {\bf w}^{T}{\bf x}_i-y_i \leq \epsilon+\xi_i ,  \\
&\xi_i\geq 0,  
\end{aligned}
\label{eq:soft_SVR_n}
\end{equation}
where in \eqref{eq:soft_SVR_n} we replaced $\tilde{\xi}_i$ by $\xi_i$. The reason why \eqref{eq:soft_SVR} and \eqref{eq:soft_SVR_n} are equivalent is as follows: First it is easy to see that 
\eqref{eq:soft_SVR_n} is identical to \eqref{eq:soft_SVR} when $\xi_i$ and $\tilde{\xi}_i$ are constrained to be equal. Hence, for any $({\bf w},\{\xi_i\}_{i=1}^n)$ feasible for \eqref{eq:soft_SVR_n}, $({\bf w},\{\xi_i\}_{i=1}^n,\{\xi_i\}_{i=1}^n)$ is also feasible for \eqref{eq:soft_SVR}. Next, let $({\bf w},\{\xi_i\}_{i=1}^n,\{\tilde{\xi}_i\}_{i=1}^n)$ be feasible for \eqref{eq:soft_SVR}. Then, for all $i=1,\cdots,n$,
$$
\begin{aligned}
  y_i-{\bf w}^{T}{\bf x}_i \leq \epsilon+\xi_i \ \text{and} \   {\bf w}^{T}{\bf x}_i-y_i \leq \epsilon+\tilde{\xi}_i  \ &\Longrightarrow |{\bf w}^{T}{\bf x}_i-y_i |\leq \epsilon+{\rm max}(\xi_i,\tilde{\xi}_i)\\
&\Longrightarrow {\bf w}^{T}{\bf x}_i-y_i\leq \epsilon+\max(\xi_i,\tilde{\xi}_i) \ \text{and} \ y_i-{\bf w}^{T}{\bf x}_i\leq \epsilon+\max(\xi_i,\tilde{\xi}_i)
\end{aligned}
$$
From this it follows that if $({\bf w},\{\xi_i\}_{i=1}^n,\{\tilde{\xi}_i\}_{i=1}^n)$ is feasible for \eqref{eq:soft_SVR}, then $({\bf w},\{\max(\xi_i,\tilde{\xi}_i\}_{i=1}^n))$ is feasible for \eqref{eq:soft_SVR_n}. The optimal costs for \eqref{eq:soft_SVR} and \eqref{eq:soft_SVR_n} are as such identical and thus are their respective solutions due to the strict convexity of their objectives.
In the sequel, for the sake of simplicity, we will study \eqref{eq:soft_SVR_n} instead of \eqref{eq:soft_SVR}.

\noindent {\bf Identifiying the PO. } 
The Lagrangian associated with   problem  \eqref{eq:soft_SVR_n}  can be written as:
\begin{align*}
\mathcal{L}({\bw},\boldsymbol{\xi},\boldsymbol{\lambda},\boldsymbol{\alpha})
&=\frac{1}{2}\| \bw\|^2+\frac{C}{p}\sum\limits_{i=1}^n\xi_i+\sum_{i=1}^{n}\lambda_i(\bbeta_\star^T \bx_i- \bw^T\bx_i+\sigma n_i-\epsilon-\xi_i)+\sum_{i=1}^{n}\alpha_i(\bw^T\bx_i-\bbeta_\star^T\bx_i-\sigma n_i-\epsilon-\xi_i)
\end{align*}
where $\boldsymbol{\xi}=[\xi_1,\cdots,\xi_n]$ and  $\boldsymbol{\lambda}=[\lambda_1,\cdots,\lambda_n]$ and $\boldsymbol{\alpha}=[\alpha_1,\cdots,\alpha_n]$ are the Lagrangian coefficients. 
Letting $\tilde{\bw}=\bbeta-\bw$, we have
\begin{align*}
\mathcal{L}(\tilde\bw,\boldsymbol{\xi},\boldsymbol{\lambda},\boldsymbol{\alpha})&=\frac{1}{2}\| \bbeta_\star-\tilde{\bw}\|^2+\frac{C}{p}\sum\limits_{i=1}^n\xi_i+\sum_{i=1}^{n}\lambda_i(\tilde{\bw}^T\bx_i+\sigma n_i-\epsilon-\xi_i)+\sum_{i=1}^{n}\alpha_i(-\tilde{\bw}^T\bx_i-\sigma n_i-\epsilon-\xi_i)
\end{align*}
Writing $\bX=[\bx_1,\cdots,\bx_n]$ leads to the following optimization problem:
\begin{align}
\min_{\tilde{\bw},\bxi\geq 0} \max_{\blamb \geq 0,\balpha\geq 0}&\frac{1}{2}\| \bbeta_\star-\tilde{\bw}\|^2+\sum\limits_{i=1}^n\left(\frac{C}{p}-\lambda_i-\alpha_i\right)\xi_i+\tilde{\bw}^T\bX(\blamb-\balpha)+\sigma \bn^T(\blamb-\balpha)- \epsilon \bone^T(\blamb+\balpha) \label{eq:bwbxi}
\end{align}
Let $\bxi^\star$ be the optimum in $\bxi$ of the above problem. It follows  from the first order condition that for all  $\bxi\geq0$, the following inequality must hold
\begin{equation}
\left(\frac{C}{p}-\lambda_i-\alpha_i\right)(\xi_i-\xi_i^\star)\geq 0, \ \ i=1,\cdots,n 
\label{eq:12}
\end{equation}
For \eqref{eq:12} to be satisfied, we must have $\frac{C}{p}-\lambda_i-\alpha_i\geq 0$ and $\left(\frac{C}{p}-\lambda_i-\alpha_i\right)\xi_i^\star=0$. Using this, the problem in \eqref{eq:bwbxi} simplifies as: 
\begin{align*}
\min_{\tilde{\bw}} \max_{\substack{\blamb \geq 0,\balpha\geq 0\\ \blamb+\balpha\leq \frac{C}{p}}}&\frac{1}{2}\| \bbeta_\star-\tilde{\bw}\|^2+\tilde{\bw}^T\bX(\blamb-\balpha)+\sigma \bn^T(\blamb-\balpha)- \epsilon \bone^T(\blamb+\balpha)
\end{align*}
Considering $\bu=\sqrt{p}(\blamb-\balpha)$ and $\bv=\sqrt{p}(\blamb+\balpha)$. With these notations, the optimization problem can be written as
\begin{equation}
\min_{\tilde{\bw}} \max_{\substack{\bv\geq 0\\|\bu|\leq \bv\\ \bv\leq \frac{C}{\sqrt{p}}}} \ \ \frac{1}{2}\| \bbeta_\star-\tilde{\bw}\|^2+\frac{1}{\sqrt{p}}\tilde{\bw}^T\bX\bu+\frac{\sigma}{\sqrt{p}}\bn^T\bu-\frac{1}{\sqrt{p}} \epsilon \bone^T\bv
\label{opt_prob_1}
\end{equation}
Clearly, and for the same arguments as in the proof of the H-SVR, the optimization over ${\bf v}$ leads to  $\bv^\star=|\bu|$. Replacing ${\bf v}$ by its optimum value,  \eqref{opt_prob_1} simplifies to,
\begin{equation}
\tilde{\Phi}:=\min_{\tilde{\bw}} \max_{\substack{\bu\\ |\bu|\leq \frac{C}{\sqrt{p}}}} \ \ \frac{1}{2} \| \bbeta_\star-\tilde{\bw}\|^2+\frac{1}{\sqrt{p}}\tilde{\bw}^T\bX\bu+\frac{\sigma}{\sqrt{p}}\bn^T\bu-\frac{1}{\sqrt{p}} \epsilon \bone^T|\bu|\label{eq:soft}
\end{equation}
Contrary to the H-SVR, the S-SVR is always feasible. As such, the optimal cost is finite and thus should be attained by a solution with finite norm. As far as the proof is concerned, checking that at optimum, the solution has a bounded norm allows us to work with the original CGMT framework in \cite{thrampoulidis-IT}. Whenever this possible, proceeding in this way should be preferred for the sake of simplicity. We thus start by proving that the norm of the optimum solution $\tilde{\bw}^\star$ is bounded. 
Indeed, it follows from the first order optimality conditions that: 
$$
\bw^\star=\bbeta_\star-\frac{1}{\sqrt{p}}\bX\bu
$$
and hence, 
$$
\|\tilde{\bw}^\star\|\leq \|\bbeta_\star\|+\frac{1}{2}\|\frac{1}{\sqrt{p}}\bX\| \  \|\bu\|
$$
Clearly, the constraint $|\bu|\leq\frac{C}{\sqrt{p}}$ implies that $\|\bu\|$ is bounded. Moreover, it is known from standard results of random matrix theory that $\|\frac{1}{\sqrt{p}}\bX\|$ is almost surely bounded \cite{SIL06} and hence $\tilde{\bw}$  can be assumed without loss of generaily to satisfy $\|\tilde{\bw}\|\leq C_w$ where $C_w$ is a finite constant. With this, the optimization problem can be written in the form  of a primary optimization problem as required by the CGMT framework in \cite{thrampoulidis-IT}:  
\begin{align}
\Phi_S^{(n)}&=\min_{\substack{\tilde{\bw}\\ \|\tilde{\bw}\|\leq C_w}} \max_{\substack{\bu\\ |\bu|\leq \frac{C}{\sqrt{p}}}} \ \ \frac{1}{\sqrt{p}}\tilde{\bw}^T\bX\bu+\frac{\sigma}{\sqrt{p}}\bn^T\bu-\frac{1}{\sqrt{p}} \epsilon \bone^T|\bu|+\frac{1}{2}\| \bbeta_\star-\tilde{\bw}\|^2.
\label{eq:bounded}
\end{align}
To connect \eqref{eq:soft} to \eqref{eq:bounded}, we rely on \cite[Lemma 5]{thrampoulidis-IT}, which ensures that if there exists $\gamma_1^\star$ and $\gamma_2^\star$ such that for sufficiently large $C_w$, the optimizer $\tilde{\bf w}_{C_w}$ satisfies:
\begin{equation}
\|\tilde{\bf w}_{C_w}\|\asto (\gamma_{1}^\star)^2+(\gamma_2^\star)^2, \ \ \ \  \frac{\tilde{\bf w}_{C_w}^{T}\boldsymbol{\beta}_\star}{\|\boldsymbol{\beta}_\star\|}\asto  \gamma_2^\star \label{eq:s}
\end{equation}
then any minimizer $\tilde{\bf w}$ of \eqref{eq:soft} satisfies:
\begin{equation}
\|\tilde{\bf w}\|\asto (\gamma_{1}^\star)^2+(\gamma_2^\star)^2, \ \ \ \ \frac{\tilde{\bf w}^{T}\boldsymbol{\beta}_\star}{\|\boldsymbol{\beta}_\star\|}\asto  \gamma_2^\star \label{eq:st}
\end{equation}
From now onward, we thus focus on analyzing $\Phi_{S}^{(n)}$. 
Using the same trick as in the proof for the H-SVR, we may project ${\tilde{\bw}}$ onto ${\boldsymbol{\beta}}_\star$ and the space orthogonal to it, thereby yielding:
\begin{align}
\Phi_S^{(n)}&=\min_{\substack{\tilde{\bw}\\ \|\tilde{\bw}\|\leq C_w}} \max_{\substack{\bu\\ |\bu|\leq \frac{C}{\sqrt{p}}}} \ \ \frac{1}{\sqrt{p}}\tilde{\bw}^T\bP_{\perp}\bX\bu+\frac{1}{\sqrt{p}\|\bbeta_\star\|^2}\bw^T\bbeta_\star\bbeta_\star^T\bX\bu
+\frac{\sigma}{\sqrt{p}}\bn^T\bu-\frac{1}{\sqrt{p}} \epsilon \bone^T|\bu|+\frac{1}{2}\| \bbeta_\star-\tilde{\bw}\|^2.
\end{align}
where we recall that ${\bf P}_{\perp}={\bf I}_p-\frac{\boldsymbol{\beta}_\star\boldsymbol{\beta}_\star^T}{\boldsymbol{\beta}_\star^{T}\boldsymbol{\beta}_\star}$.
Let $\bz=\frac{1}{\|\bbeta_\star\|}\bX^T\bbeta_\star$. Then, $\bz$ is independent of $\bP_{\perp}\bX$ and $\Phi^{(n)}_S$ writes as:
\begin{align}
\Phi_S^{(n)}&=\min_{\substack{\tilde\bw\\ \|\tilde\bw\|\leq C_w}} \max_{\substack{\bu\\ |\bu|\leq \frac{C}{\sqrt{p}}}} \ \ \frac{1}{\sqrt{p}}\tilde{\bw}^T\bP_{\perp}\bX\bu+\frac{1}{\sqrt{p}\|\bbeta_\star\|}\tilde{\bw}^T\bbeta_\star\bz^T\bu
+\frac{\sigma}{\sqrt{p}}\bn^T\bu-\frac{1}{\sqrt{p}} \epsilon \bone^T|\bu|+\frac{1}{2}\| \bbeta_\star-\tilde{\bw}\|^2.
\label{PO_SM}
\end{align}

\noindent{\bf Identification and simplification of the AO.}
With the primary problem in \eqref{PO_SM}, we associate the following AO problem given by:
\begin{equation}
\phi_S^{(n)}=\min_{\substack{\tilde{\bw}\\ \|\tilde{\bw}\|\leq C_w}} \max_{\substack{\bu\\ |\bu|\leq \frac{C}{\sqrt{p}}}} \ \frac{1}{\sqrt{p}}\|\bP_{\perp}\tilde{\bw}\|\bg^T\bu-\frac{1}{\sqrt{p}}\|\bu\|\bh^T\bP_{\perp}\tilde{\bw}+\frac{1}{\sqrt{p}\|\bbeta_\star\|}\tilde{\bw}^T\bbeta_\star\bz^T\bu+\frac{\sigma}{\sqrt{p}}\bn^T\bu-\frac{1}{\sqrt{p}} \epsilon \bone^T|\bu| +\frac{1}{2}\| \bbeta_\star-\tilde{\bw}\|^2
\label{AO_SM}
\end{equation}
In a similar way as in the H-SVR, we decompose $\tilde{\bf w}$ as:
$$
\tilde{\bf w}=\gamma_1\frac{{\bf P}_{\perp}\tilde{\bf w}}{\|{\bf P}_\perp\tilde{\bf w}\|}+ \gamma_2\frac{\boldsymbol{\beta}_\star}{\|\boldsymbol{\beta}_\star\|}
$$
Hence, 
\begin{equation}
\phi_S^{(n)}=\min_{\substack{\gamma_1,\gamma_2\\ \gamma_1^2+\gamma_2^2\leq C_w^2}} \max_{\substack{\bu\\ |\bu|\leq \frac{C}{\sqrt{p}}}} \  A_n(\gamma_1,\gamma_2,{\bf u})
\end{equation}
where
$$
A_n(\gamma_1,\gamma_2,{\bf u})=\frac{\gamma_1}{\sqrt{p}}\bg^T\bu-\frac{\gamma_1}{\sqrt{p}}\|\bu\|\|\bP_{\perp}\bh\|+\frac{\gamma_2}{\sqrt{p}}\bz^T\bu+\frac{\sigma}{\sqrt{p}}\bn^T\bu-\frac{ \epsilon}{\sqrt{p}} \|\bu\|_1+ \frac{1}{2}(\gamma_1^2+\gamma_2^2)+\frac{1}{2}\|\bbeta_\star\|^2-\gamma_1\|\bbeta_\star\|
$$
It can be easily seen that 
$$
\inf_{\substack{\gamma_1,\gamma_2\\ \gamma_1^2+\gamma_2^2\leq C_w^2}} \sup_{|\bu|\leq \frac{C}{\sqrt{p}}} \left|A_n(\gamma_1,\gamma_2,{\bf u})-\tilde{A}_n(\gamma_1,\gamma_2,{\bf u})\right|\asto 0.
$$ 
where
$$
\tilde{A}_n(\gamma_1,\gamma_2,{\bf u}):=\frac{\gamma_1}{\sqrt{p}}\bg^T\bu-\gamma_1\|\bu\|+\frac{\gamma_2}{\sqrt{p}}\bz^T\bu+\frac{\sigma}{\sqrt{p}}\bn^T\bu-\frac{ \epsilon}{\sqrt{p}} \|\bu\|_1+ \frac{1}{2}(\gamma_1^2+\gamma_2^2)+\frac{1}{2}\beta^2-\gamma_1\beta$$
As a result $\phi_S^{(n)}-\tilde{\phi}_S^{(n)}\asto 0$, where
$$
\tilde{\phi}_S^{(n)}=\min_{\substack{\gamma_1,\gamma_2\\ \gamma_1^2+\gamma_2^2\leq C_w^2}} \max_{\substack{\bu\\ |\bu|\leq \frac{C}{\sqrt{p}}}} \  \tilde A_n(\gamma_1,\gamma_2,{\bf u})
$$
 Based on this, we will work from now on with $\tilde{\phi}_S^{(n)}$. 
For that, we rely on Lemma~\ref{lemma_opt_u} to perform the optimization with respect to ${\bf u}$ and simplify $\tilde{\phi}_S^{(n)}$ as:
\begin{equation}
\tilde{\phi}_S^{(n)}= \min_{\substack{\gamma_1,\gamma_2\\ \gamma_1^2+\gamma_2^2\leq C_w^2}} \sup_{\chi\geq 0} \frac{1}{2}(\gamma_1^2+\gamma_2^2)+\frac{1}{2}\beta^2-\gamma_1\beta +\hat{R}_n(\gamma_1,\gamma_2,\chi) \label{eq:stee}
\end{equation}
where
$$
\hat{R}_n(\gamma_1,\gamma_2,\chi):=\left\{
\begin{array}{ll}
\frac{1}{p}\displaystyle\sum_{i=1}^n C(b_i-\frac{C\gamma_1}{2\chi}) {\bf 1}_{\{b_i\chi>\gamma_1C\}} +\frac{1}{p}\displaystyle\sum_{i=1}^n \frac{b_i^2\chi}{2\gamma_1}{\bf 1}_{\{b_i\chi\leq \gamma_1C\}} -\frac{\gamma_1\chi}{2}, &\text{if} \ \ \gamma_1\neq 0 \\
\frac{C}{p}\sum_{i=1}^n b_i &  \text{if} \  \gamma_1=0.
\end{array}\right.
$$
with ${\bf b}=(|\gamma_1\bg+\gamma_2\bz+\sigma\bn|-\epsilon)_{+}$. 

\noindent{\bf Asymptotic behavior of the AOs costs.}
Denote the objective function of \eqref{eq:stee} by   $\hat{D}_n(\gamma_1,\gamma_2,\chi)$. Fix $\gamma_1,\gamma_2,\chi$, then:
\begin{align}
\hat{D}_n(\gamma_1,\gamma_2,\chi)\asto \overline{D}(\gamma_1,\gamma_2,\chi):=\frac{1}{2}\gamma_2^2+\frac{1}{2}(\gamma_1-\beta)^2+\overline{R}(\gamma_1,\gamma_2,\chi) 
\end{align}
where $\overline{R}(\gamma_1,\gamma_2,\chi)=\lim_{n\to\infty} \hat{R}_n(\gamma_1,\gamma_2,\chi)$ and is given by:
$$
\overline{R}(\gamma_1,\gamma_2,\chi):=\left\{\begin{array}{ll}
\delta \mathbb{E}\left\{C\left[\left(\left|\sqrt{\gamma_1^2+\gamma_2^2}{G}+\sigma N\right|-\epsilon\right)_{+}-\frac{C\gamma_1}{2\chi}\right]\right.{\bf 1}_{\{(|\sqrt{\gamma_1^2+\gamma_2^2}G+\sigma N|-\epsilon)_{+}\chi\geq\gamma_1C\}}&\\
+ \left.\frac{\chi}{2\gamma_1}\left(\left|\sqrt{\gamma_1^2+\gamma_2^2} {G} +\sigma N\right|-\epsilon\right)_{+}^2 {\bf 1}_{\{\left(\left|\sqrt{\gamma_1^2+\gamma_2^2}G+\sigma N\right|-\epsilon\right)_{+}\chi \leq \gamma_1 C\}}\right\} -\frac{\gamma_1\chi}{2},& \text{if} \ \ \gamma_1\neq 0\\
\tilde{R}(\gamma_1,\gamma_2):=C\delta \mathbb{E}\left[(|\gamma_2{G}+\sigma N|-\epsilon)_{+}\right], \ & \text{if} \ \ \gamma_1=0.
\end{array}\right.
$$
Fix $\gamma_1\neq 0$ and $\gamma_2\in\mathbb{R}$. Function $(\gamma_1,\gamma_2)\mapsto \sup_{\chi\geq 0}\hat{D}_n(\gamma_1,\gamma_2,\chi)$ is convex in its arguments. Moreover, one can check that for $\gamma_1\neq 0$:
$$
\lim_{\chi\to\infty} \overline{D}(\gamma_1,\gamma_2,\chi)=-\infty
$$
Hence,  we may use \cite[Lemma 10]{thrampoulidis-IT} to obtain: 
\begin{align}
\sup_{\chi\geq 0} \hat{D}_n(\gamma_1,\gamma_2,\chi)\asto\sup_{\chi\geq 0}\overline{D}(\gamma_1,\gamma_2,\chi) \label{eq:gamma1_neq0},\ \  \gamma_1\neq 0, \gamma_2\in\mathbb{R}
\end{align}
When $\gamma_1=0$, clearly, 
$$
\sup_{\chi\geq 0} \hat{D}_n(0,\gamma_2,\chi)\asto \sup_{\chi\geq 0} \overline{D}(0,\gamma_2,\chi)
$$
Function $(\gamma_1,\gamma_2)\mapsto \sup_{\chi\geq 0}\hat{D}_n(\gamma_1,\gamma_2,\chi)$ is convex in its arguments and converges pointwise to $(\gamma_1,\gamma_2)\mapsto \sup_{\chi\geq 0}\overline{D}(\gamma_1,\gamma_2,\chi)$. It thus converges uniformly over compacts in $\mathbb{R}^2$. Hence,
$$
\min_{\substack{\gamma_1,\gamma_2\\ \gamma_1^2+\gamma_2^2\leq C_w^2}} \sup_{\chi\geq 0}\hat{D}_n(\gamma_1,\gamma_2,\chi) \asto \min_{\substack{\gamma_1,\gamma_2\\ \gamma_1^2+\gamma_2^2\leq C_w^2}} \sup_{\chi\geq 0}\overline{D}(\gamma_1,\gamma_2,\chi)
$$
Now, one can easily check that
$$
\lim_{\|\substack{\scriptscriptstyle{\gamma_1}\\\scriptscriptstyle{\gamma_2}}\|\to\infty} \sup_{\chi\geq 0 }\overline{D}(\gamma_1,\gamma_2,\chi)\geq\lim_{\|\substack{\scriptscriptstyle{\gamma_1}\\\scriptscriptstyle{\gamma_2}}\|\to\infty}  \frac{1}{2}(\gamma_1^2+\gamma_2^2)-\gamma_1\beta =\infty
$$
Hence, 
$$
 \min_{\substack{\gamma_1,\gamma_2\\ \gamma_1^2+\gamma_2^2\leq C_w^2}} \sup_{\chi\geq 0}\overline{D}(\gamma_1,\gamma_2,\chi)= \min_{\substack{\gamma_1,\gamma_2}} \sup_{\chi\geq 0}\overline{D}(\gamma_1,\gamma_2,\chi)
$$
and as such:
$$
\min_{\substack{\gamma_1,\gamma_2\\ \gamma_1^2+\gamma_2^2\leq C_w^2}} \sup_{\chi\geq 0}\hat{D}_n(\gamma_1,\gamma_2,\chi)\asto \min_{\gamma_1,\gamma_2} \sup_{\chi\geq 0}\overline{D}(\gamma_1,\gamma_2,\chi)
$$
\noindent{\bf Concluding. } Applying the CGMT framework, we conclude that:
$$
\Phi_{S}^{(n)}\asto  \min_{\gamma_1,\gamma_2} \sup_{\chi\geq 0}\overline{D}(\gamma_1,\gamma_2,\chi)
$$
Since $(\gamma_1,\gamma_2)\mapsto  \sup_{\chi\geq 0}\overline{D}(\gamma_1,\gamma_2,\chi)$ is jointly strictly convex in its arguments and is coercive, it admits a unique minimizer $(\gamma_1^\star,\gamma_2^\star)$. Considering a suitable perturbation of the AO problem and applying the same arguments as in H-SVR, we may conclude \eqref{eq:s} and \eqref{eq:st}.
Similarly to the H-SVR, to retrieve the results of Theorem \ref{soft_thm}, we use  the change of variable $\tilde{\gamma}_1=\frac{\gamma_1}{\sigma}$ and $\tilde{\gamma}_2=\frac{\gamma_2}{\sigma}$.    

%% file: technical_lemmata.tex
\section{Technical Lemmas}

\begin{lemma}
\label{lem:max}
Let $m$ be a strictly positive scalar and ${\bf a}$ be a vector in $\mathbb{R}^{n}$. Then:
\begin{equation}
\max_{\substack{{\bf u}\in\mathbb{R}^n \\ \|{\bf u}\|_2=m}} \ \ {\bf u}^{T}{\bf a}-\epsilon \|{\bf u}\|_1 = m\sqrt{\sum_{i=1}^n (|a_i|-\epsilon)_{+})^2} 
\label{eq:RR}
\end{equation}
\end{lemma}
\begin{proof}
To begin with, we note that if for some $i$, the optimal solution is non-zero then it should have the same sign as $a_i$. 
Hence, \eqref{eq:RR} is equivalent to the following optimization problem:
\begin{equation}
\max_{\substack{{\bf u}\in\mathbb{R}^n \\ \|{\bf u}\|_2=m}}\sum_{i=1}^{n}|u_i|(|a_i|-\epsilon)
\label{eq:RRR}
\end{equation}
If for some $i$, $|a_i|\leq \epsilon$ then we need to set $u_i=0$. Hence, Problem \eqref{eq:RRR} amounts to solving:
$$
\max_{\substack{{\bf u}\in\mathbb{R}^n \\ \|{\bf u}\|_2=m}}\sum_{i=1}^{n}|u_i|(|a_i|-\epsilon)_{+}
$$ 
Using Cauchy-Schwartz inequality, the objective of the above optimization problem can be upper-bounded as
$$
\sum_{i=1}^{n}|u_i|(|a_i|-\epsilon)_{+}\leq \| {\bf u}\|_2\sqrt{\sum_{i=1}^{n}(|a_i|-\epsilon)_{+}^2}
$$
where  equality holds when ${\bf u}=m \frac{(|{\bf a}|-\epsilon)_{+}}{\sqrt{\|(|{\bf a}|-\epsilon)_{+}\|}}$.  The optimal cost of \eqref{eq:RR} is thus given by  $m\sqrt{\sum_{i=1}^n (|a_i|-\epsilon)_{+})^2}$. 
\end{proof}
%
%
%
%
\begin{lemma}
Let $f:\mathbb{R}\mapsto \mathbb{R}$ be a convex and even function. Then for all $x\in\mathbb{R}$,
$$
f(0)\leq f(x)
$$ 
or in other words $f$ is minimized at zero. 
\label{lem:even}
\end{lemma}
\begin{proof}
Let $x\in\mathbb{R}$. Then, by convexity of $f$,
$$
f(0)=f(\frac{1}{2}x-\frac{1}{2}x)\leq \frac{1}{2}f(x)+\frac{1}{2}f(-x)
$$
As $f(x)=f(-x)$, we thus have:
$$
f(0)\leq f(x). 
$$
\end{proof}

\begin{lemma}\cite{boyd}
Let $X$ and $Y$ be two convex sets. Let $f : X \times Y \rightarrow \mathbb{R}$ be a jointly convex function in $X\times Y$. Assume
that $\forall y \in Y$, $\inf_{x\in X} f (x, y) > -\infty$. Then, $g : y\rightarrow \inf_{x\in X} f (x, y)$ is convex in $Y$ .
\label{lem_conc}
\end{lemma}
\begin{lemma}
\label{lem:equality}
Let ${d}\in\mathbb{N}^\star$. Let ${S}_x$ be a compact non-empty set in $\mathbb{R}^{d}$. Let $f$ and $c$ be two continuous functions over $S_x$ such that the set $\left\{c(x)\leq 0\right\}$ is non-empty. Then:
$$
\min_{\substack{{\bf x}\in{S}_x \\ c({\bf x})\leq 0}}  f({\bf x})\\
=\sup_{\delta \geq 0} \min_{\substack{{\bf x}\in{S}_x \\ c({\bf x})\leq \delta}}  f({\bf x})= \inf_{\delta>0} \min_{\substack{{\bf x}\in{S}_x \\ c({\bf x})\leq -\delta}}  f({\bf x}) 
$$ 
\end{lemma}
\begin{lemma}
Let $\ba \in \mathbb{R}^{n\times1}$. Let $\beta$, $\epsilon$, and $\tau$ be positive scalars. Then, 
\begin{itemize}
\item if $\beta=0$,\\
\begin{equation}
\max_{ |\bu| \leq \tau} (\ba-\epsilon\sign(\bu))^T\bu-\beta \|\bu\|=\max_{|\bu|\leq \tau} (\ba-\epsilon\sign(\bu))^T\bu=\sum_{i=1}^n \tau \max(|a_i|-\epsilon,0),
\label{opt_lemma1}
\end{equation}
\item if $\beta\neq 0$,\\
\begin{equation}
\max_{ |\bu| \leq \tau} (\ba-\epsilon\sign(\bu))^T\bu-\beta \|\bu\|
=\sup_{\chi>0}\sum_{i=1}^{n}\frac{b_i^2\chi}{2\beta}\mathbbm{1}_{\{\frac{b_i\chi}{\beta} \leq \tau\}}+\sum_{i=1}^{n}\left(b_i \tau- \frac{\beta}{2\chi}\tau^2\right)\mathbbm{1}_{\{\frac{b_i\chi}{\beta} >\tau\}}-\beta \frac{\chi}{2}
\label{opt_lemma}
\end{equation}
where $b_i= (|a_i|-\epsilon)_{+}$.
Moreover, function $\chi\mapsto \sum_{i=1}^{n}\frac{b_i^2\chi}{2\beta}\mathbbm{1}_{\{\frac{b_i\chi}{\beta} \leq \tau\}}+\sum_{i=1}^{n}\left(b_i \tau- \frac{\beta}{2\chi}\tau^2\right)\mathbbm{1}_{\{\frac{b_i\chi}{\beta} >tau\}}-\beta \frac{\chi}{2}$ is concave on $(0,+\infty)$.
\end{itemize}
\label{lemma_opt_u}
\end{lemma}
\begin{proof}
If $\beta=0$, one can note that if for some $i$   the optimal solution is non zero then it  should have the same sign as $a_i$. Hence, the optimization problem becomes:
$$
\min_{\substack{0\leq |u_i|\leq \tau\\ i=1,\cdots,n}} \sum_{i=1}^n|u_i|(|a_i|-\epsilon)=\sum_{i=1}^n\tau(|a_i|-\epsilon)_{+}
$$
To treat the case  $\beta\neq 0$, we start by rewriting  
 $\|\bu\|$ as
$$
\|\bu\|=\inf_{\chi>0}\frac{\chi}{2}+\frac{\|\bu|\|^2}{2\chi},
$$
thus yielding:
\begin{align*}
\max_{ |\bu| \leq \tau} (\ba-\epsilon\sign(\bu))^T\bu-\beta \|\bu\|&=\max_{ |\bu| \leq \tau}\sup_{\chi>0}(\ba-\epsilon\sign(\bu))^T\bu-\beta \left[\frac{\chi}{2}+\frac{\|\bu\|^2}{2\chi}\right]
\end{align*}
In a similar way as in the case of $\beta=0$, we note that if for some $i$ the optimum $u_i$ is non zero then it necessarily have the same sign as $a_i$. Moreover, if for some $i$, $|a_i|\leq \epsilon$ then necessarily $u_i=0$. With these,  the above problem writes thus as:
$$
\sup_{\chi>0}\max_{ |\bu| \leq \tau}\sum_{i=1}^n(|a_i|-\epsilon )_{+}|u_i|-\frac{\beta}{2\chi}u_i^2-\beta \frac{\chi}{2}.
$$
Let ${\bf v}=|{\bf u}|$, and define for $i=1,\cdots,n$ $b_i=(|a_i|-\epsilon)_{+}$. 
Now, noting that $b_i\geq 0$ and that the function $x\rightarrow b_ix-\frac{\beta}{2\chi}x^2$ is increasing on $(-\infty,\frac{b_i\chi}{\beta})$ and decreasing on $(\frac{b_i\chi}{\beta},\infty)$ with its maximum achieved at $x^\star=\frac{b_i\chi}{\beta}$, one can easily see that
$$
\max_{0\leq x\leq \tau} b_i x-\frac{\beta}{2\chi}x^2=\begin{cases} b_i \tau- \frac{\beta}{2\chi}\tau^2  \ \ \ {\rm if} \ \ \frac{b_i\chi}{\beta}> \tau\\
\frac{b_i^2\chi}{2\beta}  \ \ \ \ \ \ \ \ \ \ \ \ {\rm if} \ \ \frac{b_i\chi}{\beta} \leq \tau.
 \end{cases}
$$
Hence,
\begin{align*}
&\sup_{\chi>0}\max_{ |\bu| \leq \tau}\sum_{i=1}^n(a_i-\epsilon \sign(u_i))u_i-\frac{\beta}{2\chi}u_i^2-\beta \frac{\chi}{2}\\&=\sup_{\chi>0}\sum_{i=1}^{n}\frac{b_i^2\chi}{2\beta}\mathbbm{1}_{\{\frac{b_i\chi}{\beta} \leq\tau\}}+\sum_{i=1}^{n}\left(b_i \tau- \frac{\beta}{2\chi}\tau^2\right)\mathbbm{1}_{\{\frac{b_i\chi}{\beta} >\tau\}}-\beta \frac{\chi}{2}
\end{align*}
Now, we proceed with the proof of the concavity of the function $\psi:\chi \rightarrow \sum_{i=1}^{n}\frac{b_i^2\chi}{2\beta}\mathbbm{1}_{\{\frac{b_i\chi}{\beta} \leq \tau\}}+\sum_{i=1}^{n}\left(b_i \tau- \frac{\beta}{2\chi}\tau^2\right)\mathbbm{1}_{\{\frac{b_i\chi}{\beta}>\tau\}}-\beta \frac{\chi}{2}$. For that, it suffices to note that 
$$
\psi(\chi)=\max_{|u_i|\leq \tau}\sum_{i=1}^n a_iu_i-\epsilon |u_i|-\frac{\beta}{2\chi}u_i^2 -\frac{\beta\chi}{2}
$$
Function $(u_i,\chi)\mapsto \frac{\beta u_i^2}{2\chi}$ is jointly convex on $[-\tau,\tau]\times {\mathbb{R}}_{+} $, then $(u_i,\chi)\mapsto  a_iu_i-\epsilon |u_i|-\frac{\beta}{2\chi}u_i^2 -\frac{\beta \chi}{2}$ is jointly concave on $[-\tau,\tau]\times {\mathbb{R}}_{+} $. Applying Lemma \ref{lem_conc}, we show that function $\psi$ is concave.
\end{proof}
